\documentclass{article}

\usepackage[utf8]{inputenc} 
\usepackage[T1]{fontenc}    
\usepackage{hyperref}       
\usepackage{url}            
\usepackage{booktabs}       
\usepackage{amsfonts}       
\usepackage{nicefrac}       
\usepackage{microtype}      
\usepackage{bbold}
\usepackage{caption}
\usepackage{amsmath}
\usepackage{enumitem}

\usepackage{graphicx}
\usepackage{float}
\usepackage{booktabs}
\usepackage{multirow}
\usepackage{subfigure}
\usepackage{amsthm}
\usepackage{amsmath}
\usepackage{amssymb}

\usepackage{wrapfig}

\newcommand{\B}{\mathbb{B}}
\newcommand{\bhw}{\boldsymbol{\hat w^*}}
\newcommand{\bu}{\boldsymbol{u}}
\newcommand{\bl}{\boldsymbol{l}}
\newcommand{\bli}{\boldsymbol{l}_i}
\newcommand{\bui}{\boldsymbol{u}^i}

\newcommand{\Ft}{\mathcal{F}_{t}}
\newcommand{\UR}{a_+}
\newcommand{\UL}{a_-}
\newcommand{\LR}{b_+}
\newcommand{\Ll}{b_-}
\newcommand{\Tmp}{T_{\min}}
\newcommand{\Tap}{T_{\max}}
\newcommand{\bdi}{\boldsymbol{\bar \delta}_i}
\newcommand{\bd}{\boldsymbol{\bar \delta}}
\newcommand{\delt}{\boldsymbol{ \delta}}
\newcommand{\E}{\mathbb{E}}
\newcommand{\Pm}{p_{\min }}
\newcommand{\wswa}{\bar w}
\newcommand{\Pa}{p_{\max }}
\newcommand{\LL}{\mathsf{L}}
\newcommand{\EE}{\mathfrak{E}}
\newcommand{\one}{\mathbb{1}}
\newcommand{\F}{\mathcal{F}}
\newcommand{\R}{\mathbb{R}}

\DeclareMathOperator*{\argmin}{arg\,min}

\allowdisplaybreaks

\usepackage{tikz}
\newcommand*\circled[1]{\tikz[baseline=(char.base)]{\node[shape=circle,draw,inner sep=0.3 pt] (char) {\scriptsize #1};}}

\newtheorem{thm}{Theorem}
\newtheorem{lem}[thm]{Lemma}
\newtheorem{definition}{Definition}
\newtheorem{assump}{Assumption}

\usepackage[accepted]{icml2019}

\icmltitlerunning{
Asymmetric Valleys: Beyond Sharp and Flat Local Minima
}

\usepackage{etoolbox}
\makeatletter
\patchcmd\@combinedblfloats{\box\@outputbox}{\unvbox\@outputbox}{}{\errmessage{\noexpand patch failed}}
\makeatother

\begin{document}
	
	\twocolumn[
	\icmltitle{
		Asymmetric Valleys: Beyond Sharp and Flat Local Minima
	}
	
	
	
	\icmlsetsymbol{equal}{*}
	
	\begin{icmlauthorlist}
		\icmlauthor{Haowei He}{buaa}
		\icmlauthor{Gao Huang}{tsinghua}
		\icmlauthor{Yang Yuan}{mit}
	\end{icmlauthorlist}
	
	\icmlaffiliation{buaa}{Beihang University,  China}
	\icmlaffiliation{tsinghua}{Tsinghua University, China}
	\icmlaffiliation{mit}{MIT, United States}
	
	\icmlcorrespondingauthor{Haowei He}{hehaowei@buaa.edu.cn}
	
	\icmlkeywords{SGD, Landscape, Non-convex optimization}
	
	\vskip 0.3in
	]
	
	
	
	\printAffiliationsAndNotice{}  
	
	\begin{abstract}

		Despite the non-convex nature of their loss functions, deep neural networks are known to generalize well when optimized with stochastic gradient descent (SGD). Recent work conjectures that SGD with proper configuration is able to find wide and flat local minima, which have been proposed to be associated with good generalization performance. In this paper, we observe that local minima of modern deep networks are more than being flat or sharp.
 Specifically, at a local minimum there exist many asymmetric directions such that the loss increases abruptly along one side, and  slowly along the opposite side  -- 
		we formally define such minima as \emph{asymmetric valleys}. Under mild assumptions, we prove that for asymmetric valleys, a solution biased towards the flat side generalizes better than the exact minimizer. Further, we show that simply averaging the weights along the SGD trajectory gives rise to such biased solutions implicitly. This provides a theoretical explanation for the intriguing phenomenon observed by \citet{swa}. In addition, we empirically find that batch normalization (BN) appears to be a major cause for asymmetric valleys.

	\end{abstract}
	
	\vspace{-3ex}
	\section{Introduction}
	\vspace{-1ex}
	
	The loss landscape of 
	neural networks 
	has attracted great research interests in the deep learning community \cite{open_prob,thelosslandscapeofoverparameterizednn,largebatchtraining, essentiallynobarriersin, landscapedesign,empiricalanalysisofhessianofoverparameter}. It provides the basis of designing better optimization algorithms, and helps to answer the question of when and how a deep network can achieve good generalization performance. One hypothesis that draws attention recently is that the local minima of neural networks can be characterized by their flatness, and it is conjectured that sharp minima tend to generalize worse than the flat ones \cite{largebatchtraining}. A plausible explanation is that a flat minimizer of the training loss can achieve lower generalization error if the test loss is shifted from the training loss due to random perturbations. \figurename~\ref{fig:Shirish} gives an illustration for this argument.
	
	\begin{figure*}
		\centering
		\subfigure[\label{fig:Shirish}]{
			\centering
			\begin{minipage}{0.41\textwidth}
				\centering
				\includegraphics[height=1.2in]{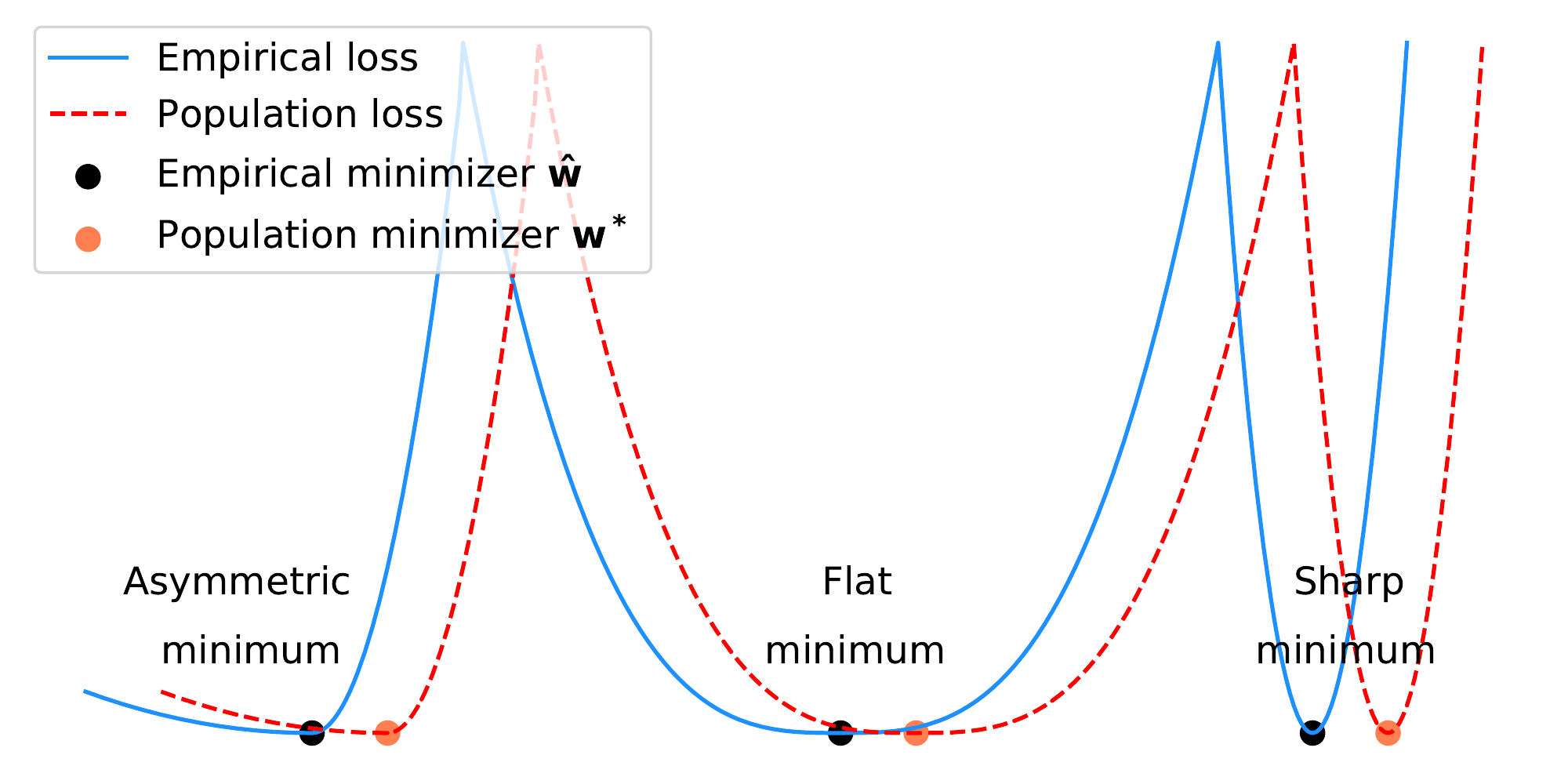}
			\end{minipage}
		}
		\subfigure[\label{fig:asy_shift}]{
			\centering
			\begin{minipage}{0.26\textwidth}
				\centering
				\includegraphics[height=1.2in]{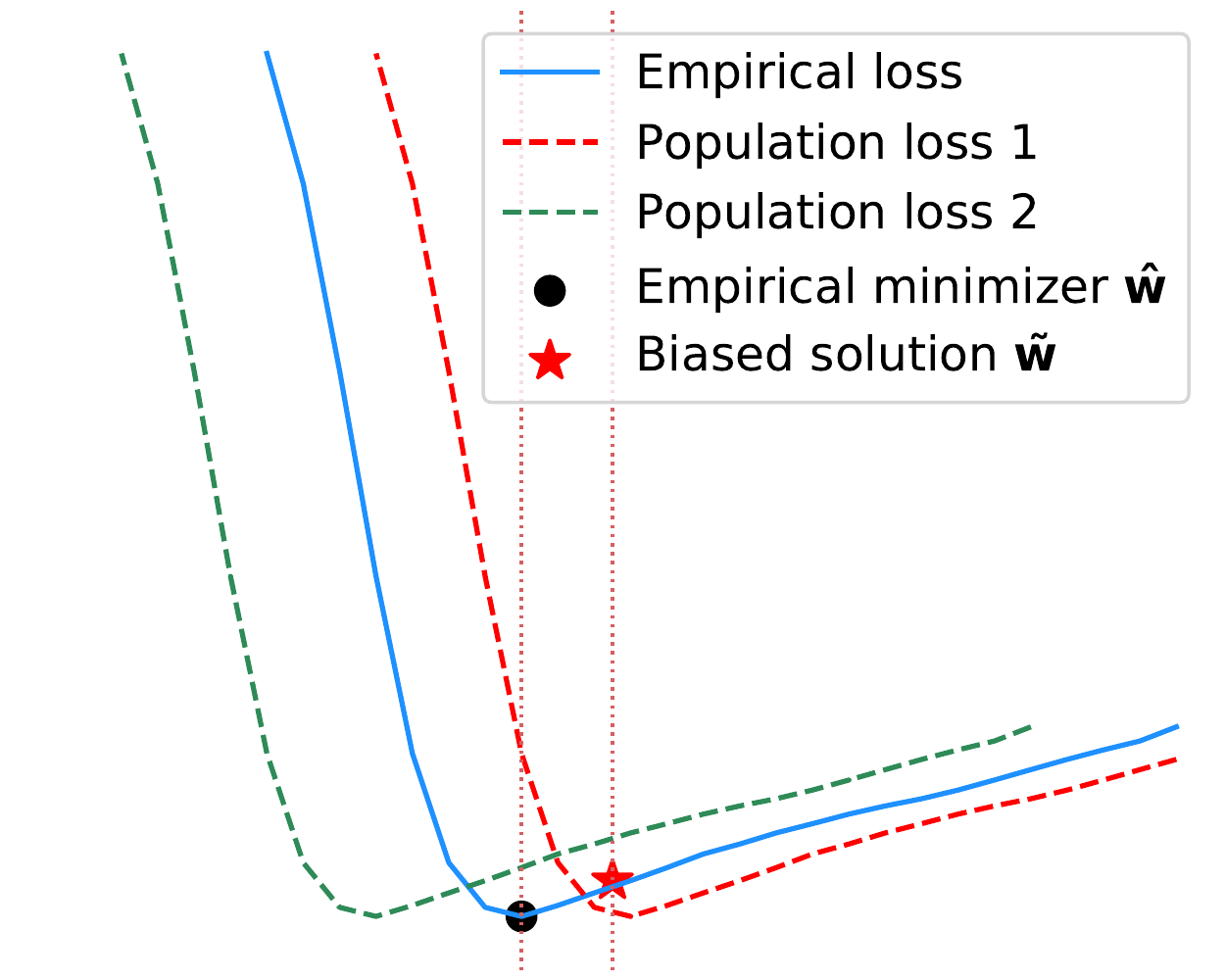}
			\end{minipage}
		}
		\subfigure[\label{fig:asym}]{
			\centering
			\begin{minipage}{0.26\textwidth}
				\centering
				\includegraphics[height=1.2in]{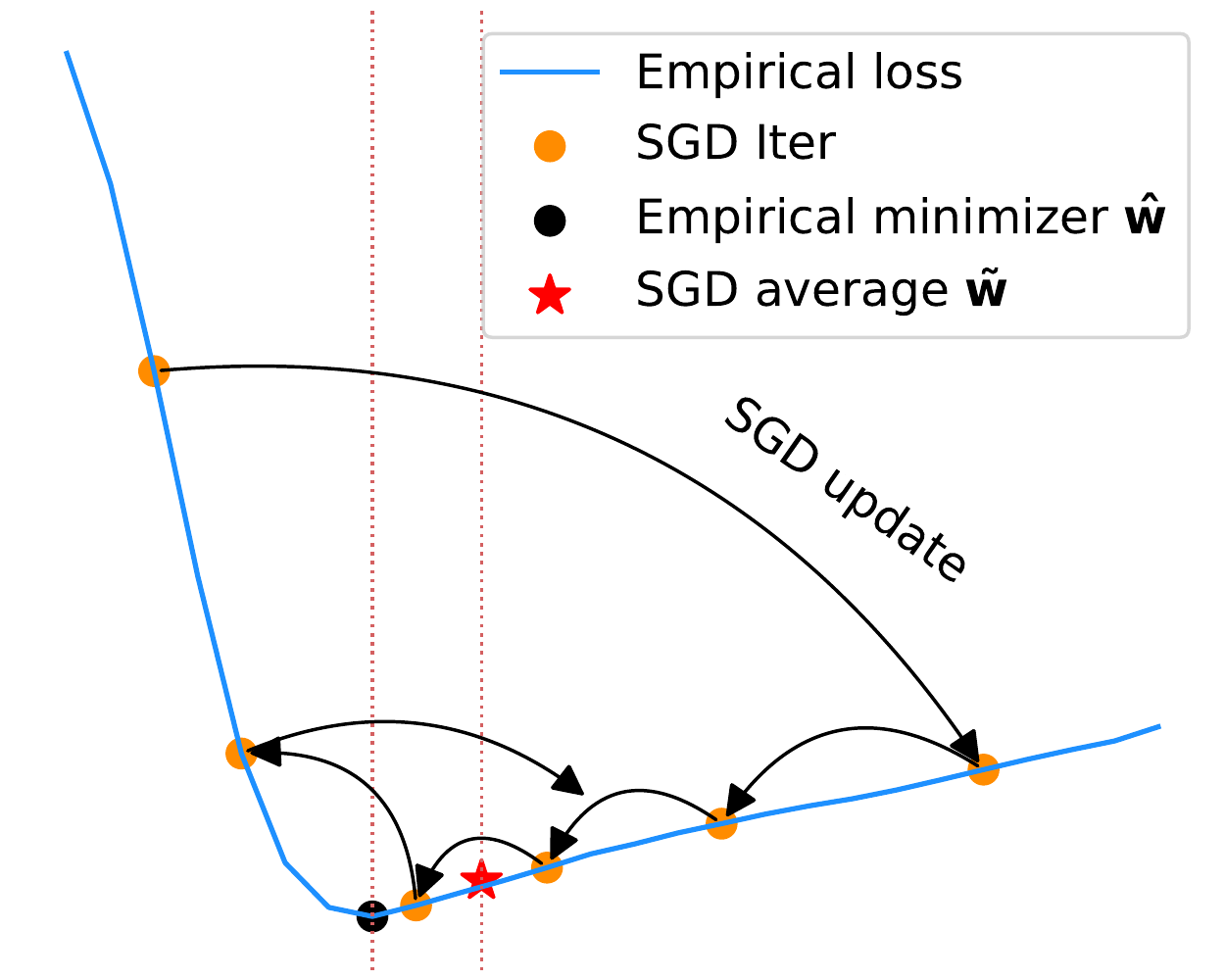}
			\end{minipage}
		}
		\caption{\textbf{(a)} Three kinds of local minima: asymmetric, flat and sharp. 
			If there exists a shift from empirical loss to population loss, 
			flat minimum is more robust than sharp minimum. 
			\textbf{(b)}
			For asymmetric valleys, if there exists a random shift, 
			the solution $\boldsymbol{ \tilde w}$ biased towards the flat side is more robust than the minimizer $\bhw$.
			\textbf{(c)} SGD tends to stay longer on the flat side of asymmetric valleys, therefore SGD averaging automatically produces the desired bias.
		}
	\end{figure*}
	
	%

	Although being supported by plenty of empirical observations \cite{largebatchtraining,swa,visualizingthelosslandscape}, the definition of flatness was recently challenged by \cite{sharpminimacangenearlizefordeepnets}, who showed that one can construct arbitrarily sharp minima through weight re-parameterization without changing the generalization performance. In addition, recent evidence suggests that the minima of modern deep networks are connected with simple paths with low generalization error \cite{essentiallynobarriersin,losssurfacesmodeconnectivity}. Similarly, the minima found by large batch training and small batch training are shown to be connected without any ``bumps'' \cite{empiricalanalysisofhessianofoverparameter}. This raises several questions: (1) If all the minima are well connected, 
	why do some algorithms keep finding sharp minima and others keep finding flat ones \cite{largebatchtraining}?
	(2) Does flatness really affect generalization? 
	
	In this paper, we address these questions by introducing the concept of \emph{asymmetric valleys}.
	We observe that the local geometry of the loss function of neural networks is usually asymmetric.
	In other words, there exist many directions such that the loss increases abruptly along one side, and grows rather slowly along the opposite side (see \figurename~\ref{fig:asy_shift} as an illustration). We formally define this kind of local minima as asymmetric valleys. 
	As we will show in Section~\ref{subsec:samebasin}, asymmetric valleys brings interesting illusions in high dimensional space. 
	For example,  located in the same valley, $\boldsymbol{\tilde w}$ may appear to be a wider and flatter minimum than $\boldsymbol{\hat w}$ as the former is farther away from the sharp side. 
	
	For the second question, we argue that flatness does affect generalization. However, we do not simply follow the argument in \cite{largebatchtraining}, which states that flat minima tend to generalize better because they are more stable. Instead,
	we prove that in asymmetric valleys, the solution biased towards the flat side of the valley 
	gives better generalization under mild assumptions. This result has at least two interesting implications: (1) converging to \emph{which} local minimum (if there are many) may not be critical for modern deep networks. However, it matters a lot \emph{where} the solution locates; and (2) the solution with lowest \emph{a priori} generalization error is not necessarily  the minimizer of the training loss.

	Given that a biased solution is preferred for asymmetric valleys, an immediate question is how  we can find such solutions in practice. It turns out that simply averaging the weights along the SGD trajectory, naturally leads to the desired solutions with bias. We give a theoretical analysis to support this argument, see \figurename~\ref{fig:asym} for an illustration. Note that our result is in line with the empirical observations recently made by \citet{swa}. 
	
%

	In addition, we provide empirical analysis to verify our theoretical results and support our claims. For example, we show that asymmetric valleys are indeed prevalent in modern deep networks, and solutions with lower generalization error has bias towards the flat side of the valley. We also find that batch normalization seems to be a major cause for shaping asymmetric loss surfaces.


	\vspace{-1ex}
	\section{Related Work}
	\vspace{-1ex}
	
	\textbf{Neural network landscape}.
	Analyzing the landscape of deep neural networks is an active and exciting area \cite{qualitativelycharacterizing,visualizingthelosslandscape,landscapedesign,geometryofneuralnetworklosssurfaces,towardsunderstandinggeneralizationofdeeplearning,thelosslandscapeofoverparameterizednn,empiricalanalysisofhessianofoverparameter}. 
	For example, \cite{essentiallynobarriersin,losssurfacesmodeconnectivity} observed that essentially all local minima are connected together with simple paths.  \cite{snapshotensembles} used cyclic learning rate and took the ensemble of intermediate models to get improved accuracy. 
	There are also appealing visualizations for the neural network landscape \cite{visualizingthelosslandscape}.
	
	\textbf{Sharp and flat minima}.
	The discussion of sharp and flat local minima dates back to \cite{flat_minima_2}, and recently regains its popularity. For example, \citet{largebatchtraining} proposed that large batch SGD finds sharp minima, which leads to poor generalization.
	In \cite{entropy-sgd}, an entropy regularized SGD was introduced to explicitly searching for flat minima. It was later pointed out that large batch SGD can yield comparable performance when the learning rate or the number of training iterations are properly set \cite{trainlonger,trainimagenetin1hour,dontdecaythelearningrate,revisitingsmallbatchtraining,abayesianperspectiveongeneralization,threefactorsinfluencingminima}.
	Moreover, \cite{sharpminimacangenearlizefordeepnets} showed that from a given flat minimum,
	one could construct another minimum with arbitrarily sharp directions but equally good performance. 
	In this paper, we argue that 
	the description of
	sharp or flat minima is an oversimplification.  
	There may simultaneously exist steep directions,  flat directions, and asymmetric directions for the same minimum.

	\textbf{SGD optimization and generalization}.
	As the de facto optimization tool for deep networks, SGD and its variants are extensively studied in the literature.  For example, it is shown that they could escape saddle points 
	or sharp local minima under reasonable assumptions~\cite{Ge2015,escapesaddlepointsefficiently,onthelocalminimaoftheempriicalrisk,acceleratedgradientdescentescapesaddlepoints,firstorderstochasticalgorithmsforescapingfrom,howtomakethegradientssmall,natasha2,neon2,analternativeview}. 
	For convex functions \cite{polyakaveraging} or strongly convex but non-smooth functions~\cite{makinggradientdescentoptimal},  SGD averaging is shown to give better convergence rate. 
	In addition, it can also achieve higher generalization performance for Lipschitz functions in theory~\cite{shalev2009stochastic,onlinetobatch}, or for deep networks in practice~\cite{snapshotensembles,swa,therearemanyconsistentexplanations}.  Discussions on the generalization bound of neural networks can be found in~\cite{spectrally-normalizedmarginbounds,towardsunderstandingtherole,exploringgeneralizationindeeplearning,generalizationindeeplearning,apacbayesianapproach,strongergeneralizationboundsfordeepnets,nonvacuousgeneralizationbound}.
	
	We show that SGD averaging has implicit bias on the flat sides of the minima. Previously, it was shown that SGD has other kinds of implicit bias as well 
	\cite{theimplicitbiasofgd,riskandparameterconvergence,characterizingimplicitbias}.


	\vspace{-1ex}
	\section{Asymmetric Valleys}
	\label{sec:asym}
	
	In this section, we give a formal definition of asymmetric valley, and show that it is prevalent in the loss landscape of modern deep neural networks.
	
	\vspace{-1ex}
	\paragraph{Preliminaries.} In supervised learning, 
	we seek to optimize $
	\boldsymbol{w^*}\triangleq \argmin_{\boldsymbol{w}\in \R^d} \LL(\boldsymbol{w}), ~~ \textrm{where}~~ \LL(\boldsymbol{w})\triangleq\E_{\boldsymbol{x}\sim \mathcal{D}} [f(\boldsymbol{x};\boldsymbol{w})]
	\in \R^d \rightarrow \R
	$
	is the population loss, $\boldsymbol{x}\in \R^{m}$ is the input from distribution $\mathcal{D}$, $\boldsymbol{w}\in \R^d$ denotes the model parameter, and $f\in \R^m \times \R^d\rightarrow \R $ is the loss function. 
	
	Since the data distribution $\mathcal{D}$ is usually unknown, instead of optimizing $\LL$ directly,  we often use SGD to find the empirical risk minimizer $\bhw$ for a set of random samples $\{\boldsymbol{x}_i\}_{i=1}^n$ from $\mathcal{D}$ (a.k.a. training set):
	$\bhw\triangleq \argmin_{\boldsymbol{w}\in \R^d} \hat\LL(\boldsymbol{w}), ~~ \mathrm{where}~~ \hat\LL(\boldsymbol{w})\triangleq \frac{1}{n}\sum_{i=1}^n f(\boldsymbol{x}_i;\boldsymbol{w})
	$.

	We use a unit vector $\boldsymbol{u} \in \R^d$ to represent a direction such that 
	the points on this direction passing $\boldsymbol{w}\in \R^d$ can be written as $\boldsymbol{w}+ l \boldsymbol{u}$ for $l\in (-\infty, \infty)$.
	
	\vspace{-1ex}
	\subsection{Definition of asymmetric valley}
	Before formally introducing asymmetric valleys, we first define asymmetric directions. 
	
	\begin{definition}[Asymmetric direction]
		Given constants $p>0, r>\zeta>0,c>1$, 
		a direction $\boldsymbol{u}$ is $(r,p,c, \zeta)$-asymmetric with respect to point $\boldsymbol{w}\in \R^d$ and loss function $\hat \LL$, if 
		$\nabla_{l} \hat \LL(\boldsymbol{w}+l\boldsymbol{u})<p$, 
		and $\nabla_{l} \hat \LL(\boldsymbol{w}-l\boldsymbol{u})< -cp$
		for $l\in (\zeta, r)$. 
	\end{definition}
	
	To put it simply, asymmetric direction is a direction $\boldsymbol{u}$ along which the loss function grows at different rates at the positive/negative direction. The constant $\zeta$ handles the small neighborhood around $\boldsymbol{w}$ with very small gradients. With this definition, we now formally define the \emph{asymmetric valley}.
	
	\begin{definition}[Asymmetric valley]
		\label{def:asy_valley}
		Given constants $p, r>\zeta>0,c>1$, 
		a local minimum $\boldsymbol{\hat w}^*$ of $\hat\LL\in \R^d\rightarrow \R$ is a $(r,p,c, \zeta)$-asymmetric valley, if 
		there exists at least one direction $\boldsymbol{u}$ such that $\boldsymbol{u}$ is $(r,p,c, \zeta)$-asymmetric with respect to $\bhw$ and $\hat\LL$. 
	\end{definition}
	Notice that here we abuse the name ``valley'', since $\bhw$ is essentially a point at the center of a  valley.

	\vspace{-1ex}
	\subsection{Find asymmetric directions empirically}
	\label{subsec:find_asym}
	
	\begin{figure}[t]
		\centering
		\includegraphics[width=.7\columnwidth]{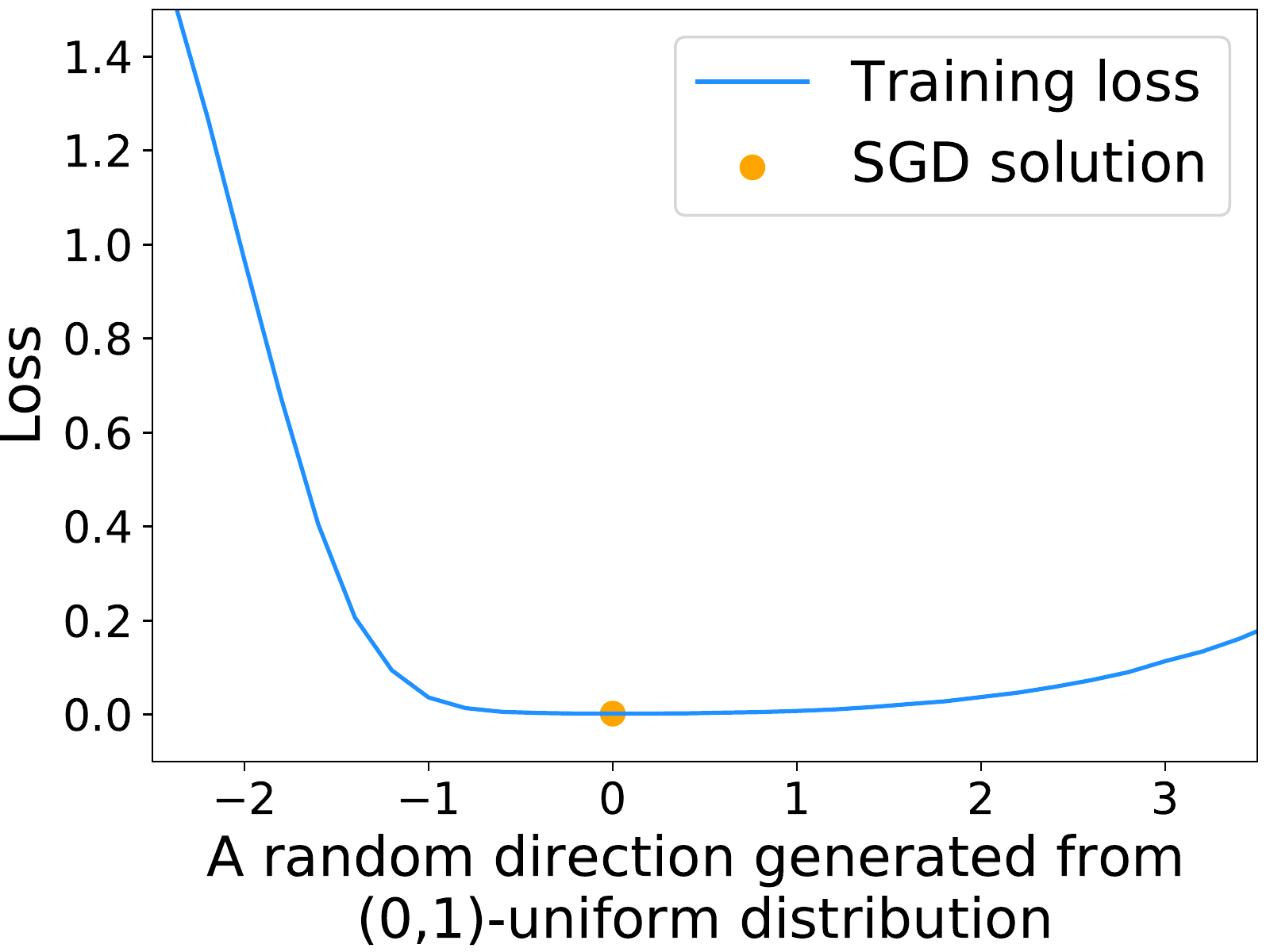}
		\caption{An asymmetric direction of a local minimum on the loss landscape of ResNet-110 trained on CIFAR-10.}
		\label{fig:sgd_asym} 
	\end{figure}
	
	Empirically, by 
	taking random directions with value 
	$(0,1)$ in each dimension, we could find   an asymmetric  direction
	for a given local minimum  
	with decent probability\footnote{By contrast, a random direction with value in $(-1,1)$ is usually not asymmetric. }. 
	We perform experiments with three widely used deep networks, i.e., ResNet-110, ResNet-164~\cite{resnet}, DenseNet-100~\cite{densenet}, on the CIFAR-10 and CIFAR-100 image classification datasets.  
	For each model on each dataset, we conduct 5 independent runs. The results show that we could \emph{always} find asymmetric directions with certain specification $(r,p,c, \zeta)$ with $c> 2$, which means all the local minima\footnote{Notice that empirically we could not verify whether the SGD solution $\bhw$ is a local minimum. See the discussion in Section \ref{subsec:samebasin}.} being found are located in asymmetric valleys.
	Figure \ref{fig:sgd_asym} shows an asymmetric direction for a local minimum in ResNet-110 trained on the CIFAR-10 dataset. 
	We verified that it is a $(2.5, 0.2, 7.5, 1.2)$-asymmetric direction.  Asymmetric valleys widely exist in other models as well, see Appendix~\ref{appendix_asy_valleys}.

	
	\vspace{-1ex}
	\section{Bias and Generalization}
	\vspace{-1ex}
	\label{sec:generalization}
	
	As we show in the previous section, most local minima in practice are \emph{asymmetric}, i.e., they might be sharp on one direction, but flat on the opposite direction. Therefore, it is important to investigate the generalization ability of a solution $\boldsymbol{w}$ in this scenario. In this section, we prove that  a \emph{biased} solution  on the flat side of an asymmetric valley yields lower generalization error than the empirical minimizer $\hat{\boldsymbol{w}}^*$ in that valley. 
	
	\vspace{-1ex}
	\subsection{Theoretical analysis}
	
	Before presenting our theorem, we first introduce two mild assumptions. We will show that they empirically hold on modern deep networks in Section \ref{sec:verify_assump}.

	The first assumption (Assumption \ref{assump:random_shift_assumption}) states that there exists a shift between the empirical loss and true population loss. This is a common assumption in the previous works, e.g., \cite{largebatchtraining}, but was usually presented in an informal way. Here we define the ``shift'' in a formal way.
	Without loss of generality, we will compare the empirical loss $\hat \LL$ with $\LL'\triangleq \LL \!-\! \min_{\boldsymbol{w}} \LL(\boldsymbol{w})  
	\!+\! \min_{\boldsymbol{w}} \hat \LL(\boldsymbol{w}) $ to remove the ``vertical difference'' between $\hat \LL$ and $\LL$.
	Notice that $\min_{\boldsymbol{w}} \LL(\boldsymbol{w})$  and 
	$\min_{\boldsymbol{w}} \hat \LL(\boldsymbol{w})$ are constants 
	and do not affect our generalization guarantee.
	
	\begin{definition}[$(\boldsymbol{\delta},R)$-shift gap]
		\label{def:shiftgap}
		For $\xi\geq 0$, $\boldsymbol{\delta}\in \R^d$, and fixed functions $\LL$ and $\hat \LL$, 
		we define the $(\boldsymbol{\delta}, R)$-shift gap between $\LL$ and $\hat \LL$ with respect to a point $\boldsymbol{w}$ as	
		\[
		\xi_{\boldsymbol{\delta}}(\boldsymbol{w})=\max_{\boldsymbol{v}\in \B(R)}|
		\LL'(\boldsymbol{w}+\boldsymbol{v} +\boldsymbol{\delta})-\hat \LL(\boldsymbol{w}+\boldsymbol{v} )|
		\]
		where $\LL'(\boldsymbol{w})\triangleq \LL(\boldsymbol{w})- \min_{\boldsymbol{w}} \LL(\boldsymbol{w})  
		+ \min_{\boldsymbol{w}} \hat \LL(\boldsymbol{w}) $, and $\B(R)$ is the $d$-dimensional ball with radius $R$ 
		centered at $\boldsymbol{0}$. 
	\end{definition}
	
	From the above definition, we know that the two functions match well after the shift $\boldsymbol{\delta}$ if $\xi_{\boldsymbol{\delta}}(\boldsymbol{w})$ is very small. For example, $\xi_{\boldsymbol{\delta}}(\boldsymbol{w})=0$ means $\LL$ is locally identical to $\hat \LL$ after the shift $\boldsymbol{\delta}$. Since $\hat \LL$ is computed on a set of random samples from $\mathcal{D}$, the actual shift $\boldsymbol{\delta}$ between $\hat \LL$ and $\LL$ is a random variable, ideally with zero expectation.

	\begin{assump}[Random shift assumption]
		\label{assump:random_shift_assumption}
		For a given population loss $\LL$ and a random empirical loss $\hat \LL$, 
		constants $R>0, r\geq \zeta>0, \xi\geq 0$, 
		a vector $\bd\in\R^d$ with
		$r \geq \bdi
		\geq \zeta$ for all $i\in [d]$, a minimizer $\boldsymbol{\hat w}^*$, 
		we assume that there exists 
		a random variable $\boldsymbol{\delta}\in \R^d$ correlated with $\hat \LL$ such that $\Pr(\delt_i = \bdi)=\Pr(\delt_i= -\bdi)=\frac12 $  for all $i\in[d]$, and the $(\delt ,R)$-shift gap between $\LL$ and $\hat \LL$ with respect to $\boldsymbol{\hat w}^*$  is bounded by $\xi$. 	
	\end{assump}

	Roughly, the above assumption says that the local landscape of the empirical loss and population loss match well after applying a shift vector $\boldsymbol{\delta}$, which has equal probability of being positive or negative in each dimension. 
	Therefore, $\boldsymbol{\delta}$ has $2^d$ possible values for a given shift vector $\boldsymbol{\bar \delta}$, each with probability $2^{-d}$. 
	The second assumption stated below can be seen as an extension of Definition \ref{def:asy_valley}.
	
	\begin{assump}[Locally asymmetric]
		\label{assump:Locally_identical_assumption}
		For a given population loss $\hat \LL$, 
		and a minimizer $\boldsymbol{\hat w}^*$, 
		there exist orthogonal directions $\boldsymbol{u}^1, \cdots, \boldsymbol{u}^k\in \R^d $ s.t. $\bui$ is $(r,p_i, c_i, \zeta)$-asymmetric with respect to $\boldsymbol{\hat w}^* + \boldsymbol{v}- \langle \boldsymbol{v}, \bu^i\rangle \bu^i $ for all $\boldsymbol{v}\in \B(R')$ and $i\in [k]$. 
	\end{assump}
	
	Assumption \ref{assump:Locally_identical_assumption} states that if $\bui$ is an asymmetric direction at $\boldsymbol{\hat w}^*$, then
	the point $\boldsymbol{\hat w}^* \!+\! \boldsymbol{v} \!-\! \langle \boldsymbol{v}, \bu^i\rangle \bu^i $ that deviates from $\boldsymbol{\hat w}^*$ along the perpendicular direction of $\bui$, is also asymmetric along the direction of $\bui$. In other words, the \emph{neighborhood} around $\bhw$ is an asymmetric valley.  
	
	Under the above assumptions, we are ready to state our theorem, which says the empirical minimizer is not necessarily the optimal solution, while a biased solution leads to better generalization. 
	We defer the proof to 
	Appendix \ref{appendix:proof_dimd}.

	\begin{thm}[Bias leads to better generalization]
		\label{thm:dimd}	
		For any $\bl\in \R^k$, if 
		Assumption \ref{assump:random_shift_assumption} holds for $R=\|\bl\|_2$, 
		Assumption \ref{assump:Locally_identical_assumption} holds for $R'= \|\bd\|_2+\|\bl\|_2$, and
		$ \frac{4\xi }{(c_i-1)p_i}< \bli\leq  \max\{r-\bdi, \bdi -\zeta \}$, 
		then we have \vspace{-0.1in}
		\begin{align*}
		&\E_{\boldsymbol{\delta}}\LL(\boldsymbol{\hat w}^*)
		-\E_{\boldsymbol{\delta}}\LL\left (\boldsymbol{\hat w}^* +\sum_{i=1}^k\bli  \bui\right )
		\\\geq & \sum_{i=1}^k (c_i-1)\bli p_i/2
		-2k\xi>0
		\end{align*}
		
	\end{thm}

	\paragraph{Remark on Theorem \ref{thm:dimd}.} It is widely known that the empirical minimizer is usually different from the true optimum. 
	However, in practice it is difficult to know how the training loss shifts from the population loss. 
	Therefore, the best we could do is minimizing the empirical loss function (with some regularizers). On the contrary,  Theorem~\ref{thm:dimd} states that under the asymmetric case, we should pick a biased solution to minimize the expected population loss even the shift is unknown. Moreover, it is possible to distill our insight into practical algorithms, as we will discuss in Section \ref{sec:avg_is_good}.

	\vspace{-1ex}
	\subsection{Verification of assumptions}
	\label{sec:verify_assump}
	
	\paragraph{Verification of Assumption \ref{assump:random_shift_assumption}.}
	We show that a shift between $\LL$ and $\hat \LL$ is quite common in practice, by taking a ResNet-110 trained on CIFAR-10 as an example. Since we could not visualize a shift in a high dimensional space, we randomly sample an asymmetric direction $\boldsymbol{u}$ (more results are shown Appendix \ref{appendix_shift}) at the SGD solution $\boldsymbol{\hat w}^*$.
	The blue and red curves shown in \figurename~\ref{fig:random_shift} are obtained by calculating 
	$\hat \LL(\boldsymbol{\hat w}^* + l \boldsymbol{u})$ and  $ \LL'(\boldsymbol{\hat w}^*+ l \boldsymbol{u})$ for  $l \in [-3,3] $, which correspond to the training and test loss, respectively. 
	
	We then try different shift values of $\boldsymbol{\delta}$ to ``match'' the two curves. As shown in \figurename~\ref{fig:random_shift}, after applying a horizontal shift $\boldsymbol{\delta}\!=\!0.4$ to the test loss, the two curves overlap almost perfectly. Quantitatively, we can use the \emph{shift gap} defined in Definition \ref{def:shiftgap} to evaluate how well the two curves match each other after shifting. It turns out that  $\xi_{\boldsymbol{\delta}=0.4}\!=\!0.0335$, which is much lower than $\xi_{\boldsymbol{\delta}=0} \!=\! 0.223$ before shifting ($\delt$ has only one dimension  here). In \figurename~\ref{fig:ratio}, we plot $\xi_{\boldsymbol{\delta}}/\xi_{\boldsymbol{0}}$ as a function of $\boldsymbol{\delta}$. Clearly, there exists a $\boldsymbol{\delta}$ that minimizes this ratio, indicating a good match.
	
	We conducted the same experiments for different directions, models and datasets, and similar observations were made. Please refer to Appendix \ref{appendix_shift} for more results.

	\begin{figure}[htbp]
		\centering
		\subfigure[Shfit between $\LL'$ and $\hat\LL$\label{fig:random_shift} ]{
			\centering
			\begin{minipage}{0.22\textwidth}
				\centering
				\includegraphics[width=4cm]{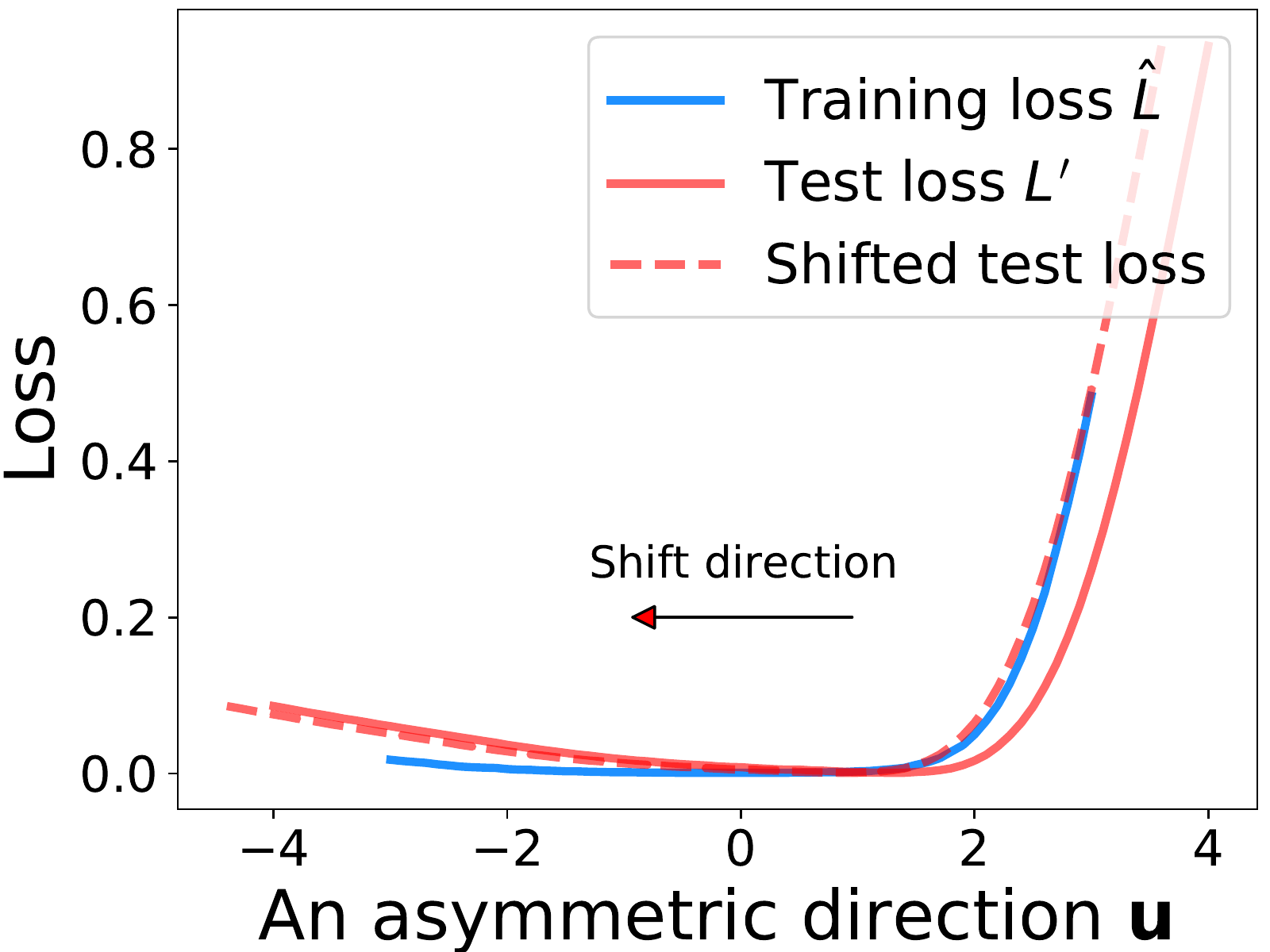}
			\end{minipage}
		}
		\subfigure[$ \xi_{\boldsymbol{\delta}} /\xi_{\boldsymbol 0}$ vs different shift $\boldsymbol{\delta}$\label{fig:ratio}]{
			\centering
			\begin{minipage}{0.22\textwidth}
				\centering
				\includegraphics[width=4cm]{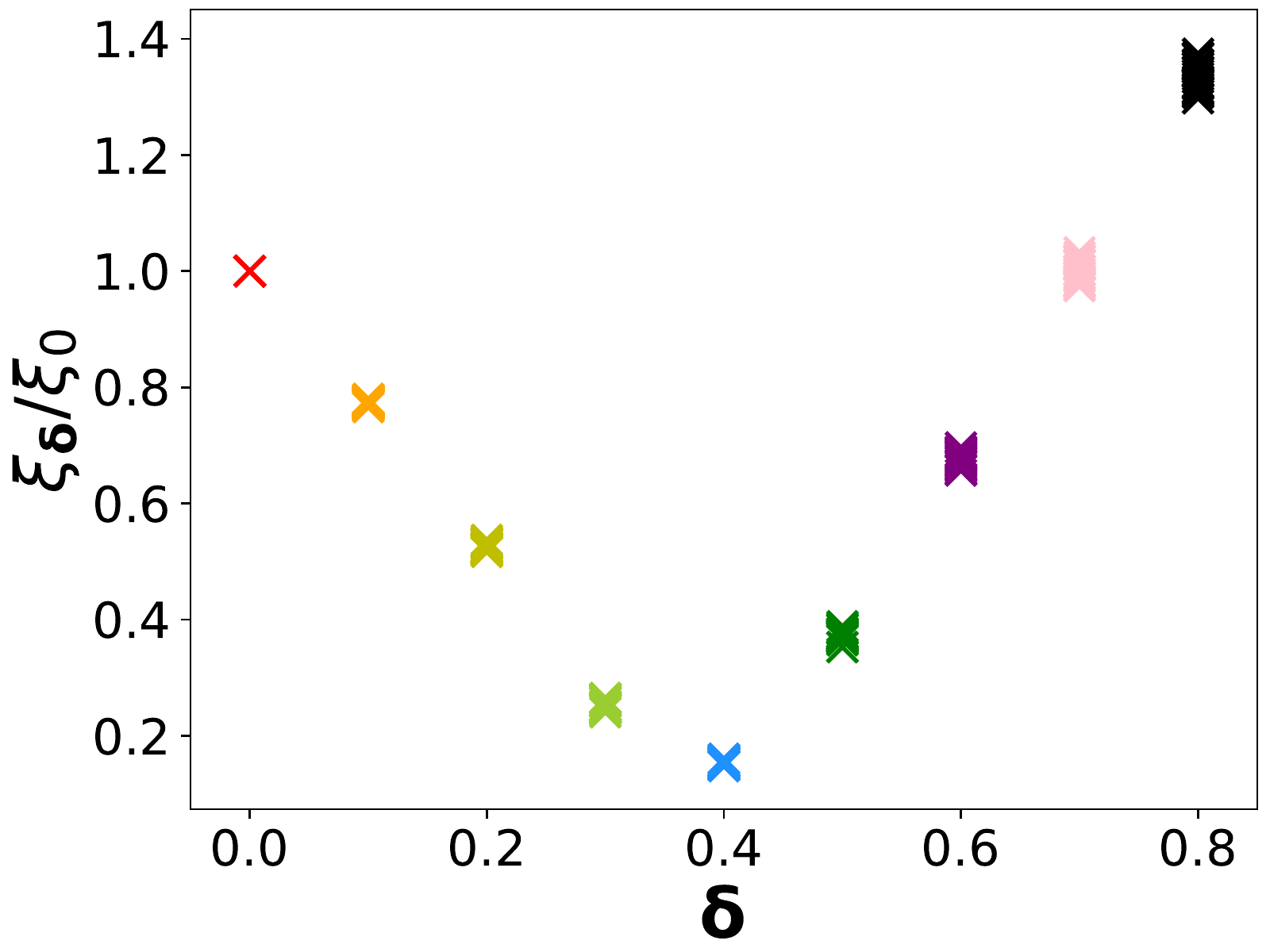}
			\end{minipage}
		}
		\caption{Shift exists between empirical loss and population loss for ResNet-110 on CIFAR-10.}
	\end{figure}
	

	\paragraph{Verification of Assumption \ref{assump:Locally_identical_assumption}.}

	\begin{figure}[t]
		\centering
		\includegraphics[width=.75\columnwidth]{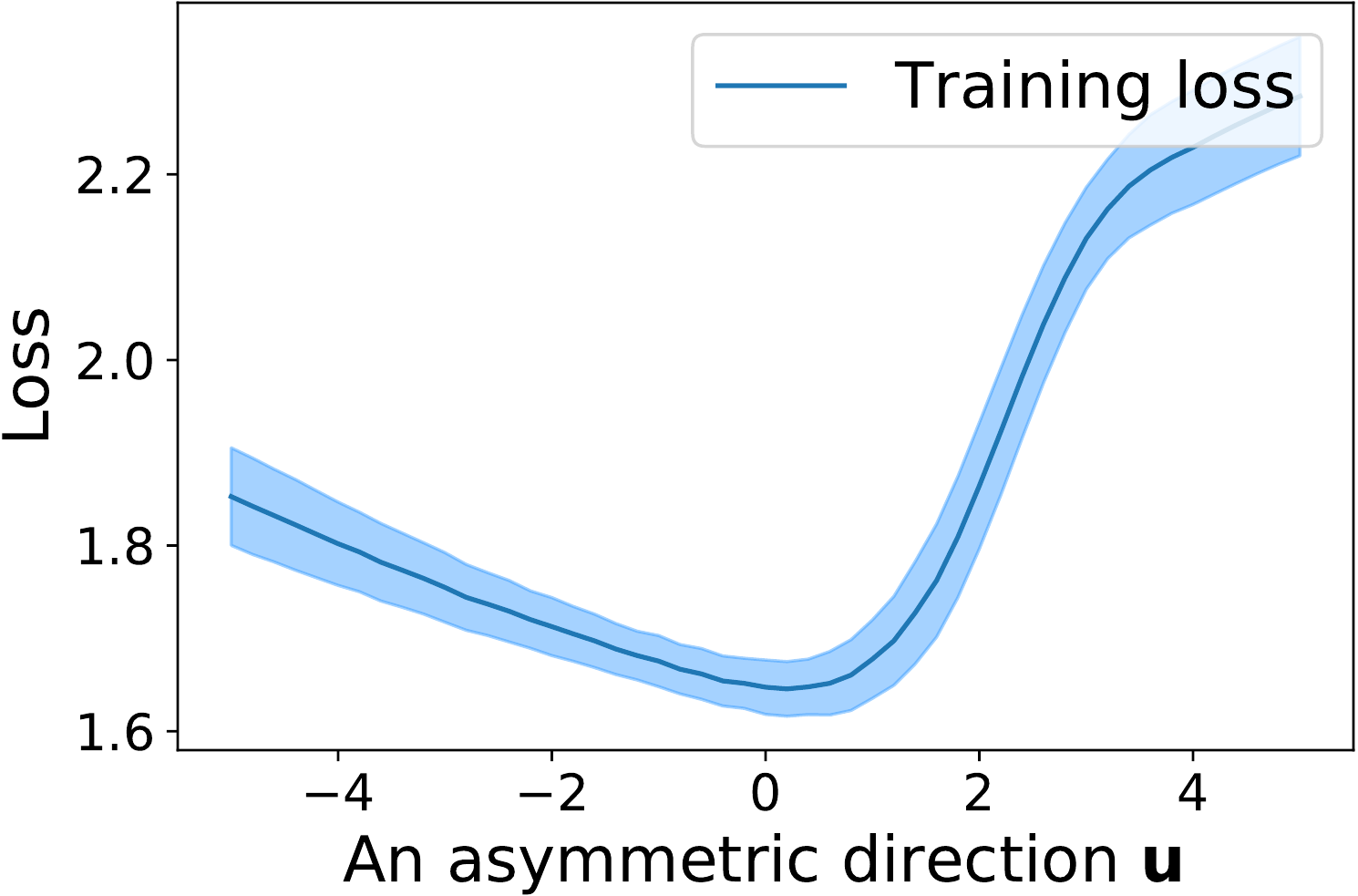}
		\caption{Training loss mean and standard variance for the neighborhood of $\bhw$ at the direction of $\bu$.}
		\label{fig:local_shift} 
	\end{figure}
	This is a mild assumption
	that can be
	verified empirically. For example, 
	we take  a SGD solution of ResNet-110 on CIFAR-10 as $\boldsymbol{\hat w}^*$, and specify an asymmetric direction $\boldsymbol{u}$ for $\bhw$. We then randomly sample $100$ different local adjustments for $\boldsymbol{v}\in \B(25)$. Based on these adjustments, 
	we present the mean loss curves and standard variance zone on the
	asymmetric direction $\boldsymbol{u}$
	for all the points $\bhw+\boldsymbol{v} - \langle \boldsymbol{v}, \bu \rangle \bu $ in Figure \ref{fig:local_shift}. 
	As we can see, 
	the variance of these curves are very small, which means
	all of them are similar to each other. Moreover, we verified that $\bu$ is  $(4,0.1,5.22,2)$-asymmetric with respect to all neighboring points.

	
	\section{Averaging Generates Good Bias}
	\label{sec:avg_is_good}
	
	In the previous section, we show that when the loss landscape of a local minimum is asymmetric, a solution with bias towards the flat side of the valley has better generalization performance. One immediate question is that how can we obtain such a solution via practical algorithms? Below we show that it can be achieved by simply taking the average of SGD iterates during the course of training. We first analyze the one dimensional case in Section~\ref{sec:onedim}, and then extend the analysis to the high dimensional case in Section~\ref{sec:highdim}.

	\subsection{One dimensional case}
	\label{sec:onedim}
	
	For asymmetric functions, as long as the learning rate is not too small, SGD will oscillate between the flat side and the sharp side. Below we focus on one round of oscillation, and show that the average of the iterates in each round has a bias on the flat side. Consequently, by aggregating all rounds of oscillation, averaging SGD iterates leads to a bias as well. 
	
	For each individual round $i$, we assume that it starts from the iteration when SGD goes from sharp side to flat side (denoted as $w^i_0$), and ends at the iteration exactly before the iteration that SGD goes from sharp side to flat side again (denoted as $w^i_{T_i}$). Here $T_i$ denotes the number of iterations in the $i$-th rounds. 
	The average iterate in the $i$-th round can be written as $\wswa\triangleq \frac1{T_i} \sum_{j=0}^{T_i} w^i_j$. 
	For notational simplicity, we will omit the super script $i$ on $w^i_j$.
	
	The following theorem shows that the expectation of the average has bias on the flat side. 
	To get a formal lower bound on $\wswa$, we consider the asymmetric case where $\zeta=0$, and also assume lower bounds for the gradients on the function. Notice that 
	we made little effort to optimize the constants or bounds on the parameters, and we defer the proof to Appendix~\ref{appendix_main_theorem}.
	
	\begin{thm}[SGD averaging generates a bias]
		\label{thm:asym_avg}
		Assume that a local minimizer $w^*=0$ is a $(r,\UR,c, 0)$-asymmetric valley, 
		where 
		$\Ll \leq \nabla \LL(w) \leq \UL <0$ for $w< 0$, and $0<\LR \leq \nabla \LL(w) \leq \UR$ for $w\geq 0$. Assume $-\UL= c \UR$ for a large constant $c$, and
		$\frac{-(\Ll-\nu)}{\LR}=c'<\frac{e^{c/3}}{6}$. 
		The SGD updating rule is $w_{t+1} = w_t - \eta (\nabla L(w)+\omega_t)$
		where $\omega_t$ is the noise and $|\omega_t|<\nu$, and assume $\nu \leq \UR$. 
		Then we have 
		\begin{align}
		\vspace{-1ex}
		\E[\wswa]>c_0>0, \nonumber
		\vspace{-1ex}
		\end{align}
		where $c_0$ is a constant that only depends  on 
		$\eta, \UR, \UL, \LR, \Ll$ and $\nu $.
	\end{thm}

	Theorem \ref{thm:asym_avg} can be intuitively explained by Figure \ref{fig:Asy_onedim_occillate}.
	If we run SGD on this one dimensional function, it will stay at the flat side for more iterations as the magnitude of the gradient on this side is much smaller. Therefore, the average of the locations is biased towards the flat side. 
	
	
	\begin{figure}[hbpt]
		\centering
		\includegraphics[width=0.7\columnwidth]{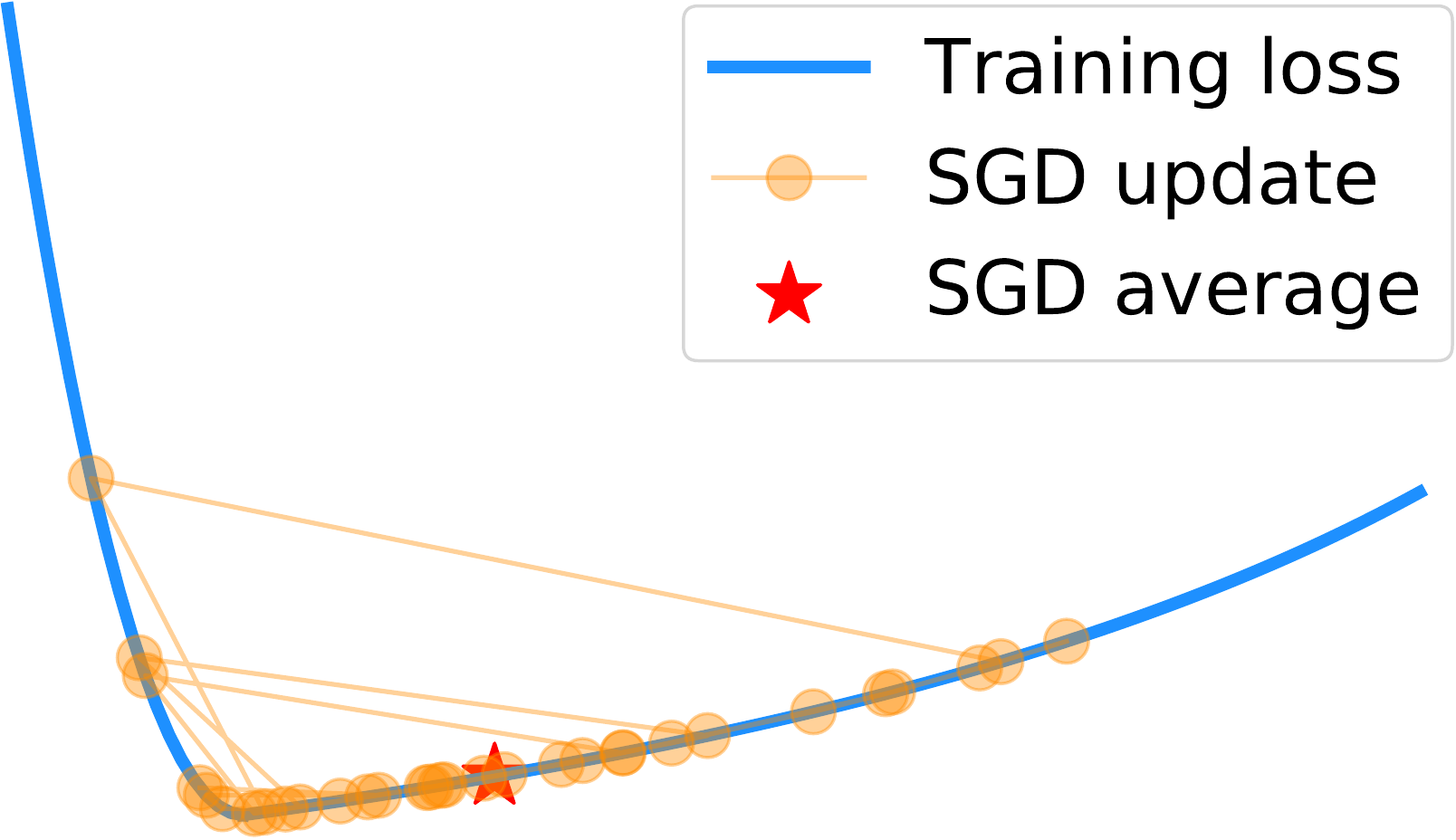}
		\caption{SGD iterates and their average on an asymmetric function: the oscillation case.}
		\label{fig:Asy_onedim_occillate} 
	\end{figure}

	Of course, if the learning rate is sufficiently small, there will be no oscillations on the SGD trajectory, as shown in Figure \ref{fig:Asy_onedim}. In this case, the bias on the sharp side tends to be closer to the center compared to the bias on the flat side, as the gradient on the sharp side is much larger than the gradient on the flat side, so SGD converges much faster. In other words, even if there is no oscillation and Theorem \ref{thm:asym_avg} does not apply, SGD averaging creates more bias on flat sides than sharp sides in expectation. Thus in all the scenarios, taking average of SGD iterates would be beneficial for asymmetric loss function. 
	
	In addition, for symmetric loss functions, averaging SGD iterates may also be helpful in terms of denoising (see Appendix~\ref{appendix:other_sgd_pattern} for concrete examples). Therefore, taking the average of the SGD trajectory may always improve generalization, regardless of whether the loss function is symmetric or not.

	\begin{figure}[htbp]
		\centering
		\includegraphics[width=0.45\columnwidth]{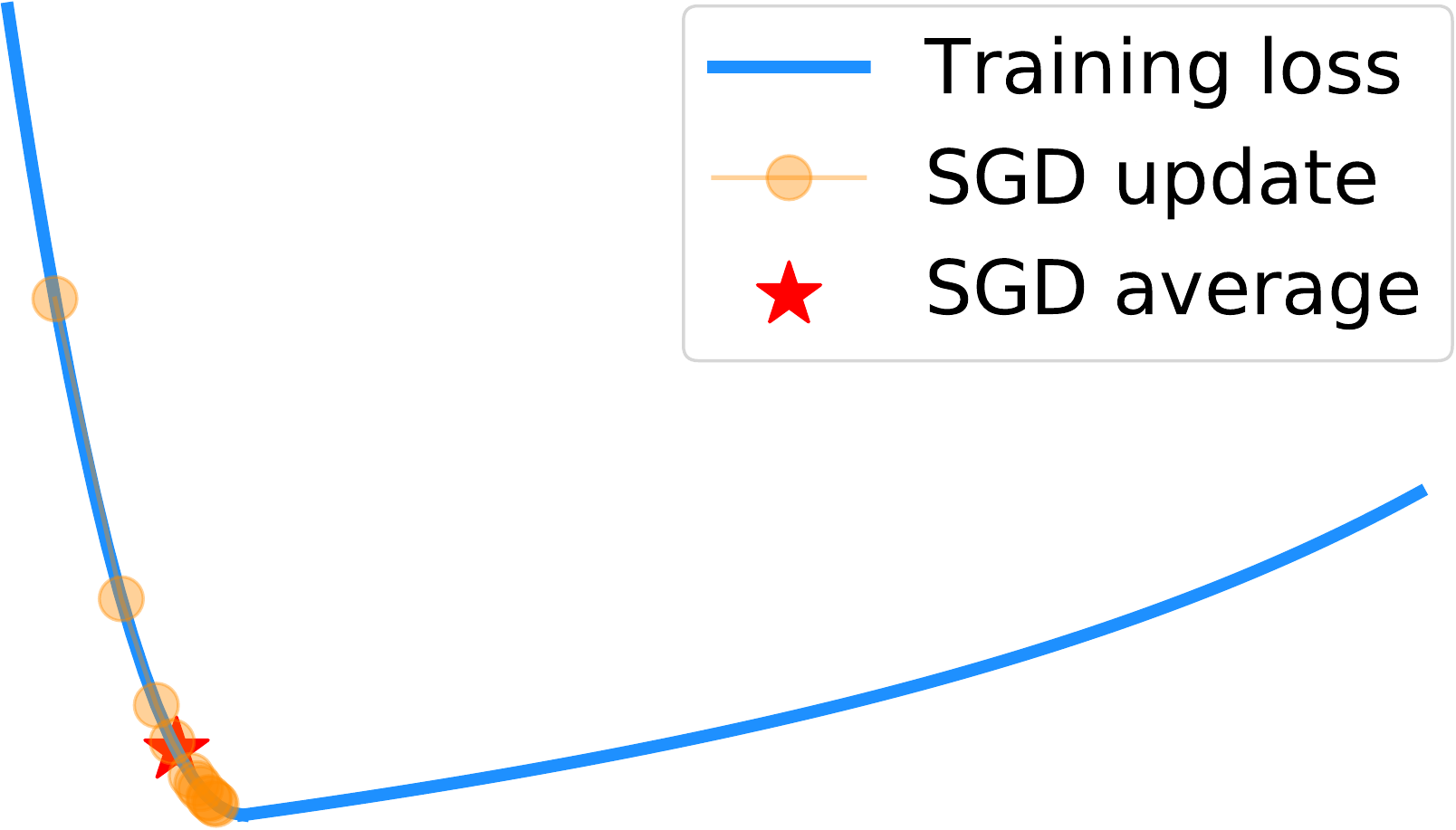}
		\includegraphics[width=0.45\columnwidth]{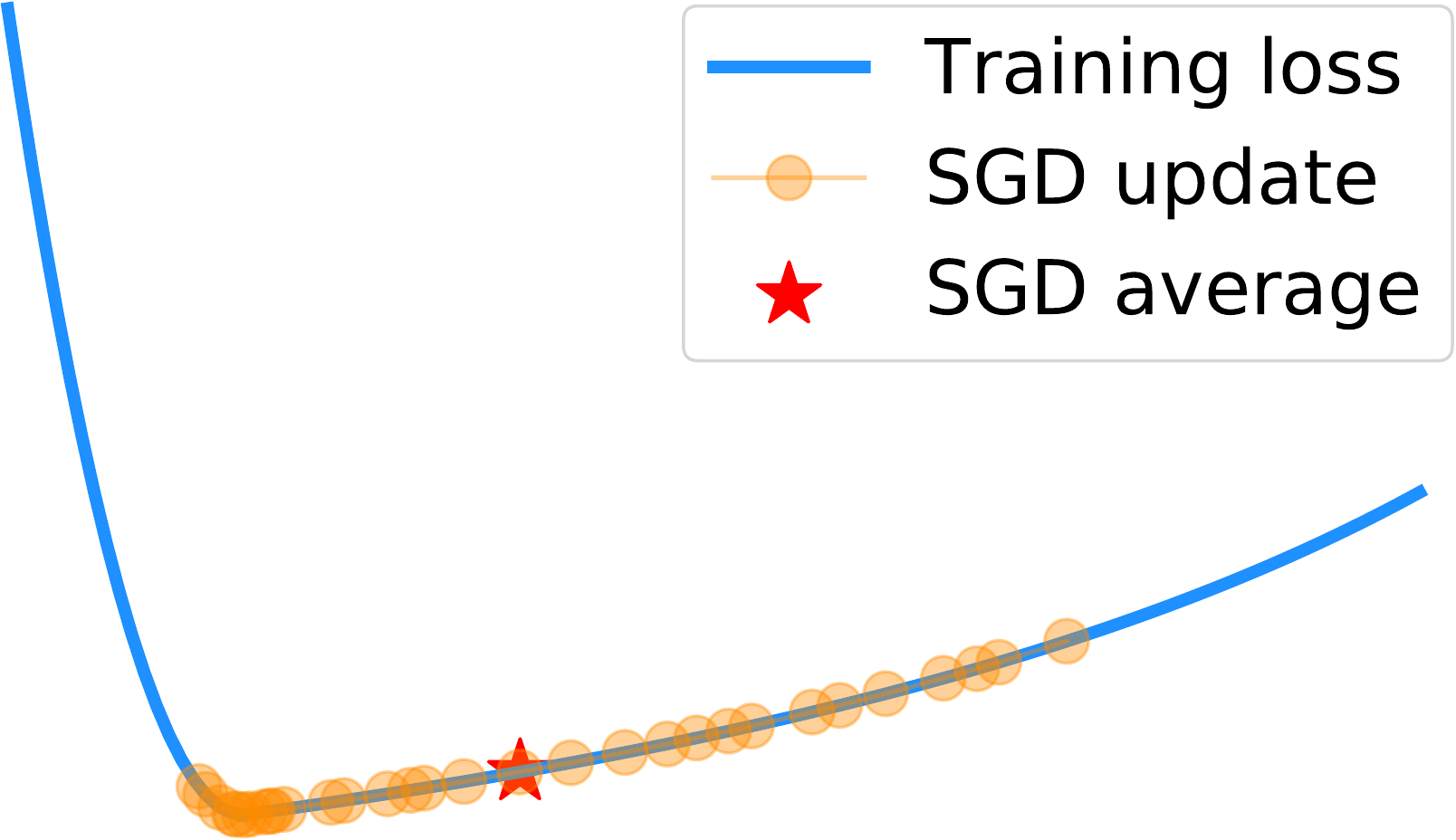}
		\caption{SGD iterates and their average on an asymmetric function with small learning rate: starting from the sharp side (\emph{Left}); and
			starting from the flat side (\emph{Right})	
		}
		\label{fig:Asy_onedim} 
	\end{figure}

	\subsection{High dimensional case}
	\label{sec:highdim}
	
	For high dimensional functions, the analysis on averaging SGD iterates would be more complicated compared to that given in the previous subsection.  However, 
	if we only care about the bias on a specific direction $\bu$, we could directly apply Theorem \ref{thm:asym_avg} with one additional assumption. 
	Specifically, 
	if the projections of the loss function onto $\bu$
	along the SGD trajectory satisfy the assumptions in Theorem \ref{thm:asym_avg}, i.e., being asymmetric and the gradient on both sides have upper and lower bounds, then the claim of Theorem \ref{thm:asym_avg} directly applies. This is because only the gradient along the direction $\bu$ will affect the SGD trajectory projected onto $\bu$, and we could safely omit all other directions.
	
	Empirically, we find that this assumption generally holds.
	For a given SGD solution,  we fix a random asymmetric direction $\bu\in\R^d$, and sample the loss surface on direction $\bu$ that passes the $t$-th epoch of SGD trajectory (denoted as $\boldsymbol{w_t}$), i.e., evaluate 
	$\hat \LL(\boldsymbol{w_t} + l\boldsymbol{\bu })$, for $0\leq t\leq 200$ and $l \in [-15,15]$. 
	
	As shown in the \figurename~\ref{fig:sgd_high_dimension_slice}, 
	after the first $40$ epochs, the projected loss surfaces becomes relatively stable. Therefore, we could directly apply Theorem \ref{thm:asym_avg} to the direction $\bu$.  
	
	As we will see in Section \ref{subsec:illusion_swa}, compared with SGD solutions, SGD averaging indeed creates bias along different asymmetric directions, as predicted by our theory.
	
	\begin{figure}[t]
		\centering
		\includegraphics[width=.45\textwidth]{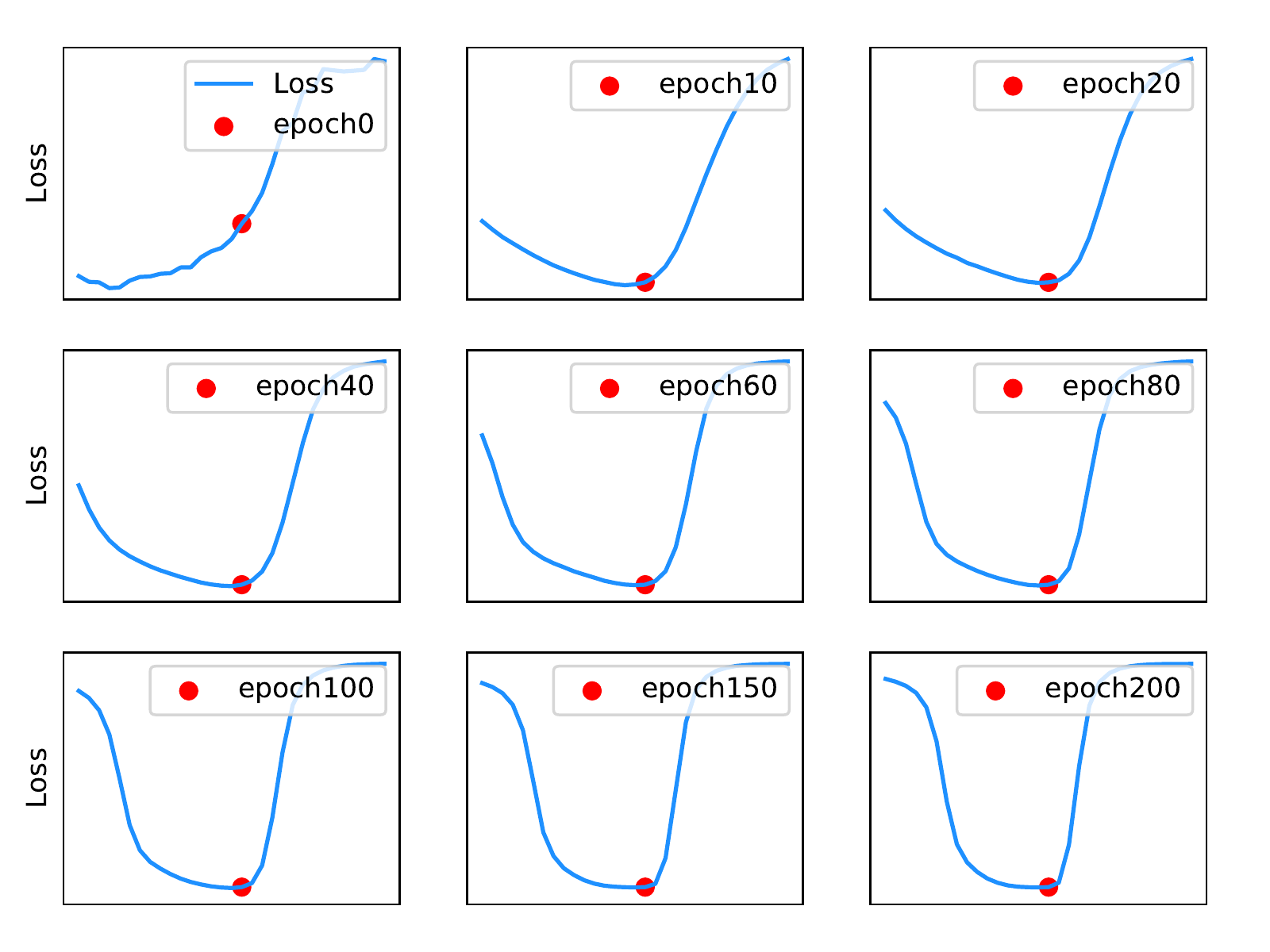}
		\caption{Projection of the training loss surface onto an asymmetric direction $\bu$} 
		\label{fig:sgd_high_dimension_slice} 
	\end{figure}


	\section{Sharp and Flat Minima Illusion}
	\label{subsec:samebasin}
	
	In this section, we show that \emph{where} the solution locates at a local minimum basin is very important, which is a refinement of judging the generalization performance by the sharpness/flatness of a local minimum. All of our observations support our theoretical analysis in the previous sections.

	First we remark that rigorously  testing whether a point is a local minimum, or even close to a local minimum, is extremely hard for deep models, see e.g.  \cite{spuriouslocalminimaarecommon}. In fact, the Hessian of most empirical solutions still have plenty of small negative eigenvalues \cite{entropy-sgd}, so technically they are saddle points. But 
	we choose to ignore these technicalities, and treat all these points as ``local minima''.

	\subsection{Illusion case 1: SWA algorithm}	
	\label{subsec:illusion_swa}

	Recently, \citet{swa} proposed the stochastic weight averaging (SWA) algorithm, which explicitly takes the average of SGD iterates to achieve better generalization. 
	Inspired by their observation that ``SWA leads to solutions corresponding
	to wider optima than SGD'', we provide a more refined explanation in this subsection. That is, averaging weights leads to ``biased'' solutions in an asymmetric valley, which correspond to better generalization.

	Specifically, we run the SWA algorithm (with deceasing learning rate) with three popular deep networks, i.e., ResNet-110, ResNet-164 and DenseNet-100, on the CIFAR-10 and CIFAR-100 datasets, following the configurations in \cite{swa} (denoted as SWA). Then we
	run SGD with small learning rate \emph{from the SWA solutions} to
	find a solution located in the same basin (denoted as SGD).
	
	In Figure \ref{fig:inter_swasgd_pic},
	We draw an interpolation between the solutions obtained by SWA and SGD\footnote{\citet{swa} have done a similar experiment.}. 
	One can observe that there is no ``bump'' between these two solutions, meaning they are located in the same basin. 
	 Clearly, the SWA solution is biased towards the flat side, which verifies our theoretical analysis in Section~\ref{sec:avg_is_good}. Further, we notice that although the biased SWA solution has higher training loss than the empirical minimizer, it indeed yields lower test loss. This verifies our analysis in Section~\ref{sec:generalization}. Similar observations are made on other networks and other datasets, which we present in Appendix \ref{appendix:swa_asym_figures}.
	
	\begin{figure}[htbp]
		\centering
		\includegraphics[width=.45\textwidth]{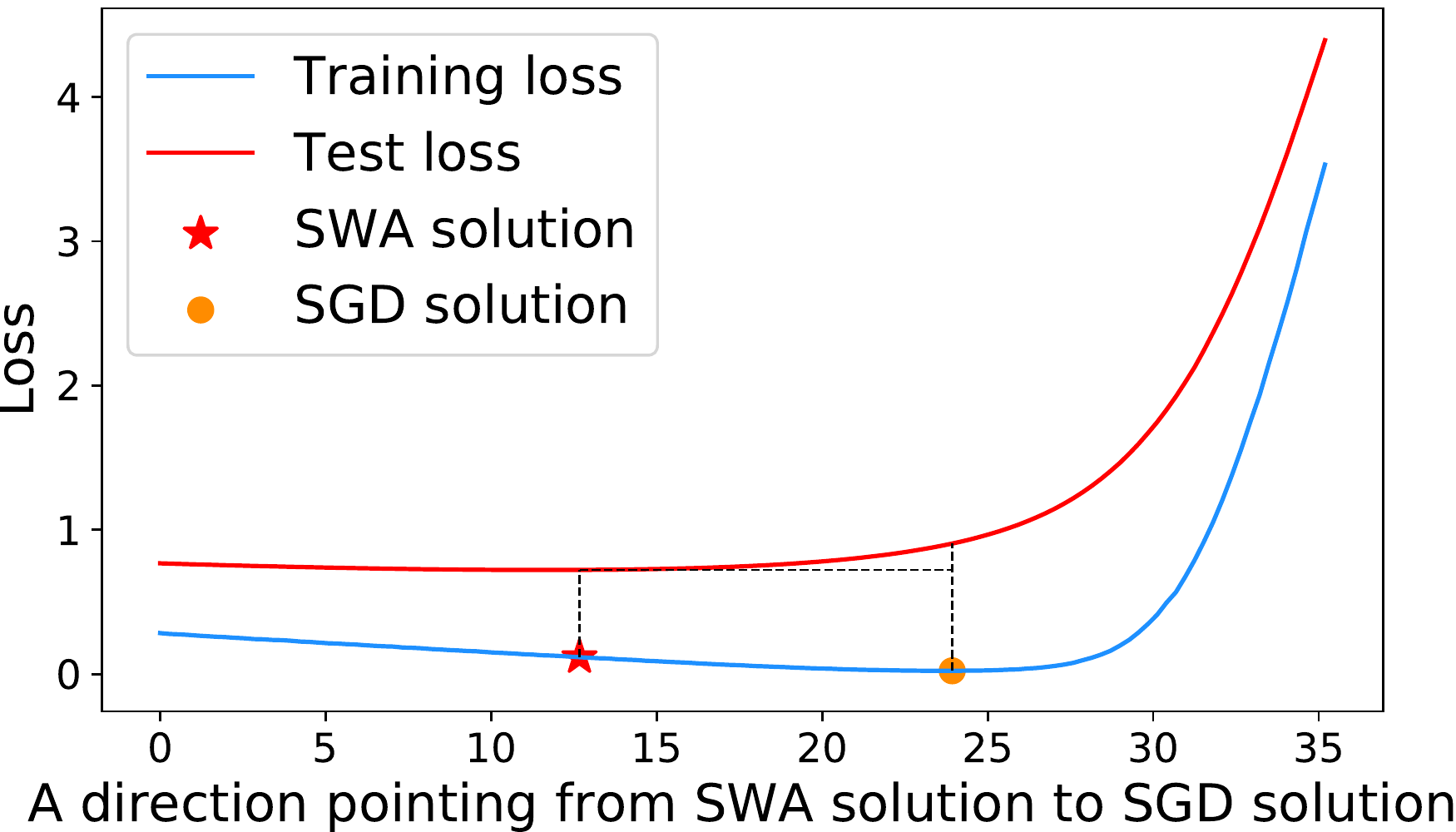}
		\caption{SWA solution and SGD solution interpolation (ResNet-164 on CIFAR100)}
		\label{fig:inter_swasgd_pic} 
	\end{figure}
	
	To further support our claim, we list our result in Table \ref{table 2:sgdafterswa}, from which we can observe that SGD solutions always have higher training accuracy,
	but worse test accuracy, compared to SWA solutions. This supports our claim in Theorem \ref{thm:dimd}, which states that a bias towards the flat sides of asymmetric valleys could help improve generalization, although it yields higher training error.

	\begin{table}[!htbp]
		\vspace{-0ex}
		\centering
		\small
		\caption{Training accuracy of various networks on CIFAR-100.}
		\label{table 2:sgdafterswa}
		\begin{tabular}{|l|l|l|l|}
			\hline
			\multirow{2}{*}{Network} & \multicolumn{2}{c|}{CIFAR-100}                                \\ \cline{2-3} 
			&   \multicolumn{1}{c|}{train} & \multicolumn{1}{c|}{test} \\ \hline
			ResNet-110-SWA        &               94.98\%                         &               78.94\%               \\ \hline
			ResNet-110-SGD       &          97.52\%                            &              78.29\%                \\ \hline
			ResNet-164-SWA        &         97.48\%                   &           80.69\%                   \\ \hline
			ResNet-164-SGD        &  99.12\%   &             76.56\%                 \\ \hline
			DenseNet-100-SWA          &     99.84\%                    &             72.29\%                 \\ \hline
			DenseNet-100-SGD         &  99.87\%                        &              71.46\%                \\ \hline
		\end{tabular}\\
		\vspace{-0ex}
	\end{table}

	\paragraph{Verifying Theorem \ref{thm:asym_avg}.} We further verify that averaging SGD solutions could create a bias towards the flat side in expectation for many other asymmetric directions, not just for the specific direction we discussed above. 
	
	We take a ResNet-110 trained on CIFAR-100 as an example. 
	Denote $\bu_{inter}$ as the unit vector pointing from the SGD solution to the SWA solution. We pick another unit random direction $\bu_{rand}$. Then, we use the direction $\bu_{inter}+\bu_{rand}$ to verify our claim. 
	
	\begin{figure}[htbp]
		\centering
		\includegraphics[width=.45\textwidth]{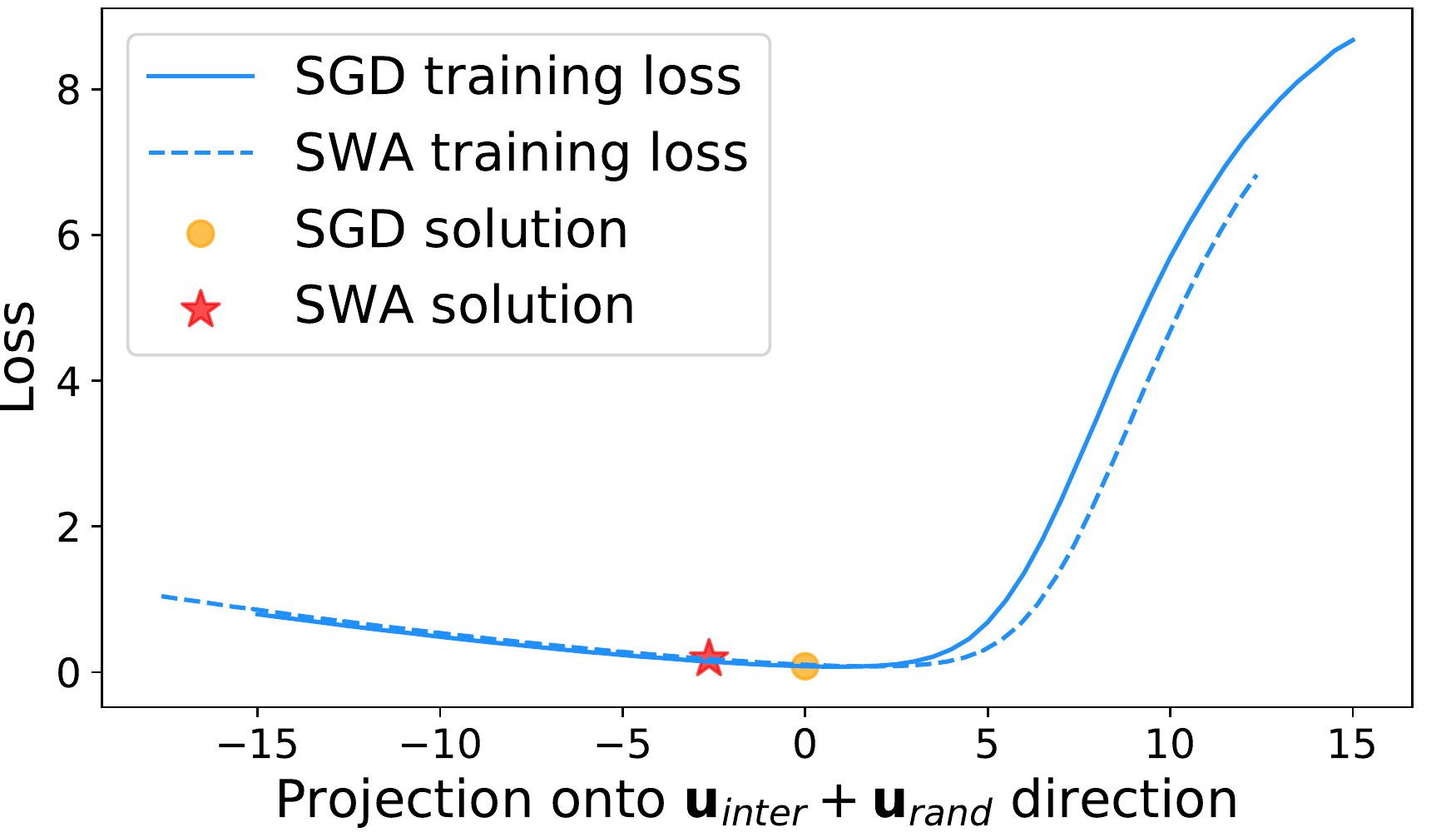}
		\caption{The average of SGD has a bias on flat side (ResNet-110 on CIFAR100)}
		\label{fig:high_dimension_find_110_1_C100} 
	\end{figure}
	
	The results are shown in Figure \ref{fig:high_dimension_find_110_1_C100}, from which we can observe that SWA has a bias on the flat side compared with the SGD solution. We create $10$ different random vectors for each network and each dataset, and similar observations can be made (see more examples in Appendix \ref{appendix:verify_high_dimension}).

	\subsection{Illusion case 2: large batch SGD}
	\label{subsec:largebatchtraining}
	\citet{largebatchtraining} observed that training with small batch size using SGD algorithm generalizes better than training with large batch size. They argue that it is because large batch SGD tends to converge to sharp minima, while small batch SGD generally converges to flat minima. Here we show that it may not be the case in practice. 
	
	We use a PreResNet-164 trained on CIFAR-100 as an example. We first running SGD with a batch size of 128 for 200 epochs to find a solution (denoted as \emph{Large batch solution}), and then contintue the training with batch size 32 for another 80 epoch to find a nearby solution (denoted as \emph{Small batch solution}).

	From the results shown in Figure \ref{fig:inter_large_small_batch_pic}, it is clear that the small batch solution has worse training accuracy but better test accuracy. Meanwhile, there is no 'bump' between these solutions which suggests they are in the same basin. Therefore, small batch SGD generalizes better because it could find a better biased solution in the asymmetric valley, not because it finds a different wider or flatter minimum. 
	
	\begin{figure}[ht]
		\centering
		\includegraphics[width=.4\textwidth]{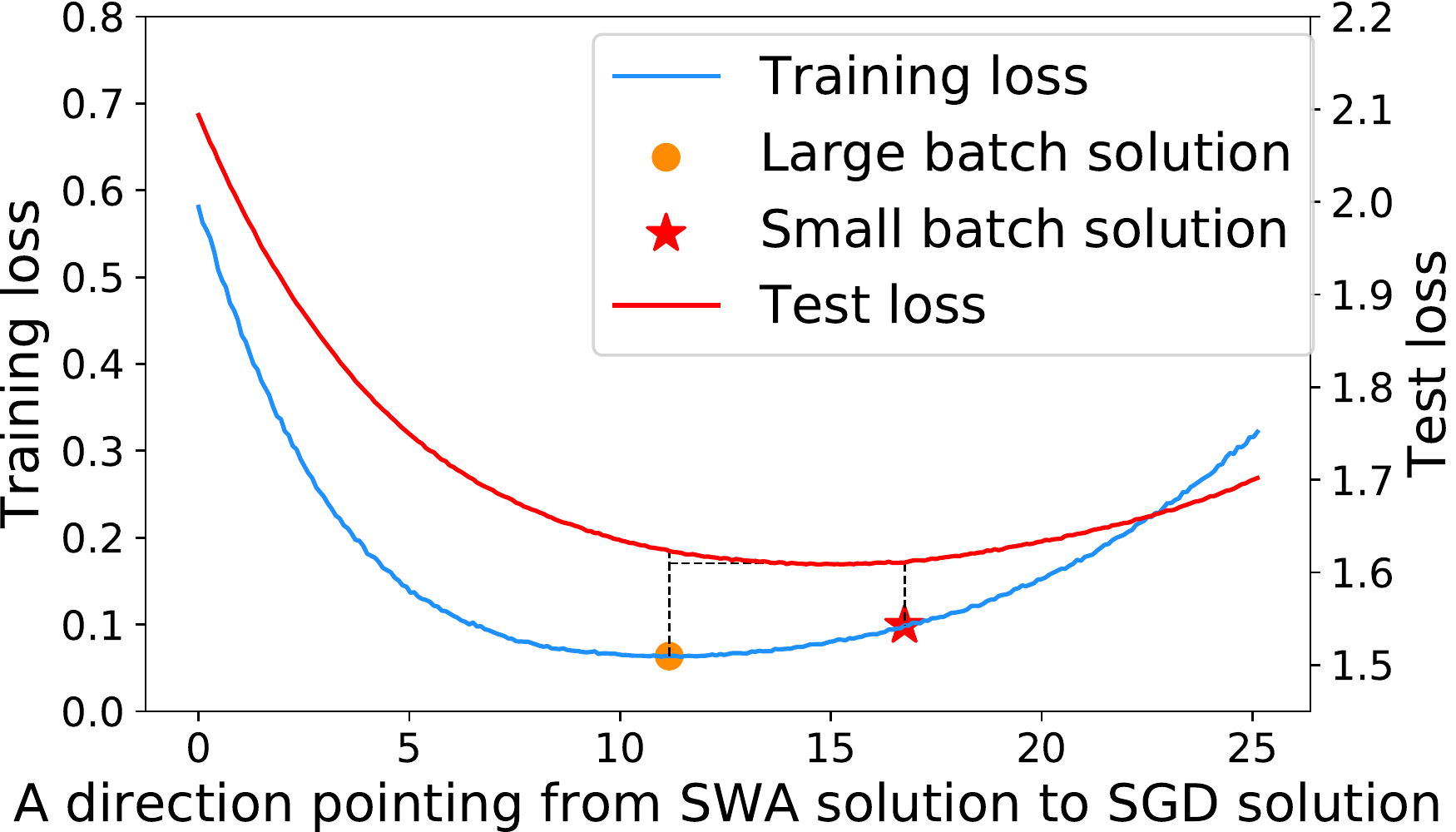}
		\caption{Large and small minibatch interpolation(batch size 128 to 32 of PreResNet-164 on CIFAR-100)} 
		\label{fig:inter_large_small_batch_pic} 
	\end{figure}
	
	\subsection{Illusion on the width of a minimum}
	\label{subsec:random_ray}
	We further point out that visualizing the ``width'' of a local minimum in a low-dimensional space may lead to illusive results. For example, one visualization technique \cite{swa} is showing how the loss changes along many random directions $\boldsymbol{v}_i$'s drawn from the $d$-dimensional Gaussian distribution.

	We take the large batch and small batch solutions from the previous subsection as our example. Figure \ref{fig:random_ray_largeandsmallbatch} visualizes the ``width'' of the two solutions using the method described above. From the figure, one may draw the conclusion that small batch training leads to a wider minimum compared to large batch training. However, as discussed in Subsection~\ref{subsec:largebatchtraining} these two solutions are actually from the same basin. In other words, the loss curvature near the two solutions looks different because they are located at different locations in an asymmetric valley, instead of being located at different local minima. Similar observation holds for SWA and SGD solutions, see Appendix \ref{sec:random_ray}.

	\begin{figure}[htbp]
		\centering
		\includegraphics[width=.45\textwidth]{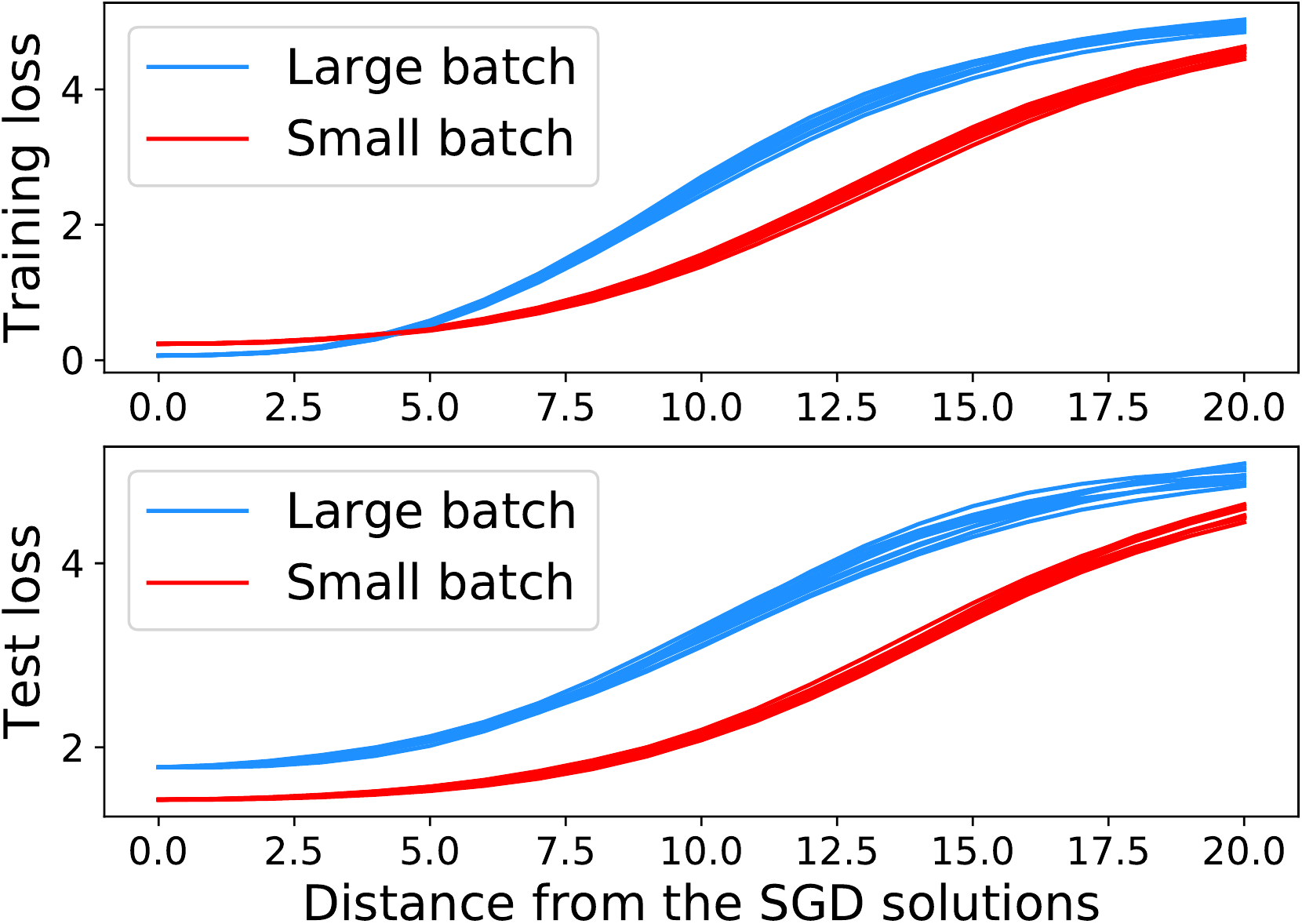}
		\caption{Random ray of large batch and small batch solution(ResNet-164 on CIFAR-100).} 
		\label{fig:random_ray_largeandsmallbatch} 
	\end{figure}

	
	\section{Batch Norm and Asymmetric Valleys} 
	\label{sec:bn}
	Preview sections have focused on defining \emph{what} are asymmetric valleys, and \emph{how} to leverage them for better generalization. In this section, we take a step forward to answer \emph{where} they originate, by showing empirical evidences that the Batch Normalization (BN) \cite{batchnorm} adopted by modern neural networks seems to be a major cause for asymmetric valleys.
	
	\paragraph{Directions on BN parameters are more asymmetric.}
	For a given SGD solution, if we take a random direction 
	where only the BN parameters have non-zero entries, and compare it with a random direction 
	where only the non-BN parameters have non-zero entries, 
	we observe that those BN-related directions are usually more asymmetric. 
	The result with ResNet-110 on CIFAR-10 is shown in Figure~\ref{fig:conv_bn_comp_res110_c10}. As we can see, the Non-BN direction is sharp on both sides, but BN direction is flat on one side, and sharp on the other side.
	We also conducted trials with different networks and datasets, and obtained similar results (see Appendix~\ref{sec:exploration_on_bn_drection}). 

	\begin{figure}[t]
		\centering
		\includegraphics[width=.45\textwidth]{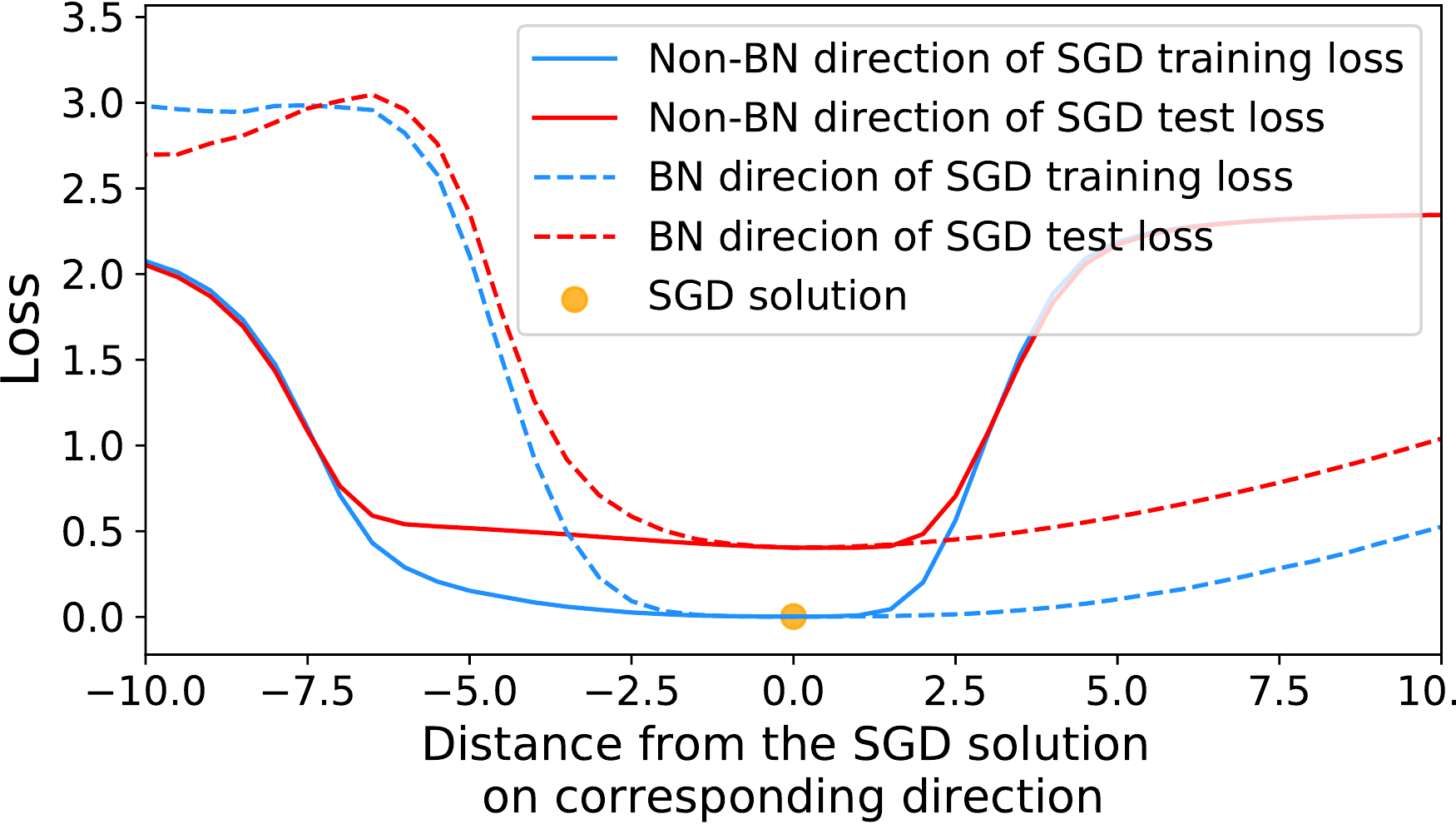}
		\caption{
			BN and Non-BN directions 
			through a local minimum of
			ResNet-110 on CIFAR-10. }
		\label{fig:conv_bn_comp_res110_c10}
	\end{figure}
	
	\paragraph{SGD averaging is more effective on BN parameters.} 
	By Theorem \ref{thm:dimd} and \ref{thm:asym_avg},  we know that 
	SGD averaging could lead to biased solutions on asymmetric directions with better generalization. 
	If BN indeed creates many asymmetric directions, can we improve the model performance by only averaging the weights of BN layers? 
	
	Note that BN parameters only constitute a small fraction of the total model parameters, e.g., 1.41\% in a ResNet-110. 
	In the follow experiment on ResNet-110 for CIFAR-10, we perform SGD averaging only on BN parameters, denoted as SWA-BN; and also averaging randomly selected non-BN parameters of the same amount (1.41\% of the total parameters), denoted as SWA-Non-BN. The results are shown in Figure \ref{fig:swa_bn_compare}. It can be observed that averaging only BN parameters (blue curve) is more effective than averaging non-BN parameters (green curve), although there is still a gap comparing to averaging all the weights (yellow curve).

	\begin{figure}[t]
		\centering
		\includegraphics[width=.45\textwidth]{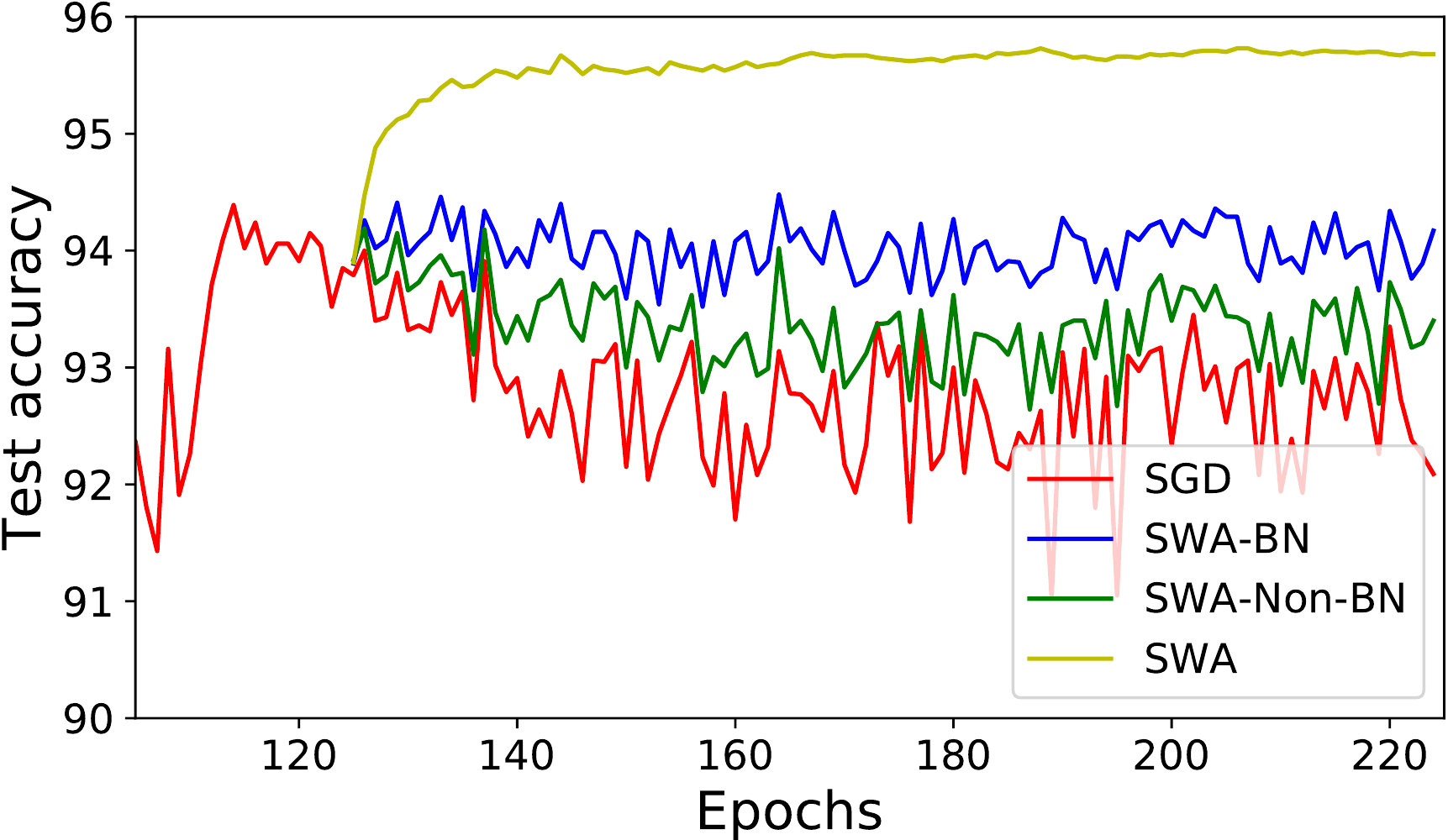}
		\caption{
			SGD averaging on BN parameters could give better test accuracy compared with SGD averaging on non-BN parameters.
		} 
		\label{fig:swa_bn_compare}
	\end{figure}

	Moreover, we also conduct experiments with two 8-layer ResNets on CIFAR-10, one with BN layers and one without. We choose shallow networks here as deeper models without BN can not be effectively trained.
	
	\begin{figure}[htbp]
		\centering
		\includegraphics[width=.45\textwidth]{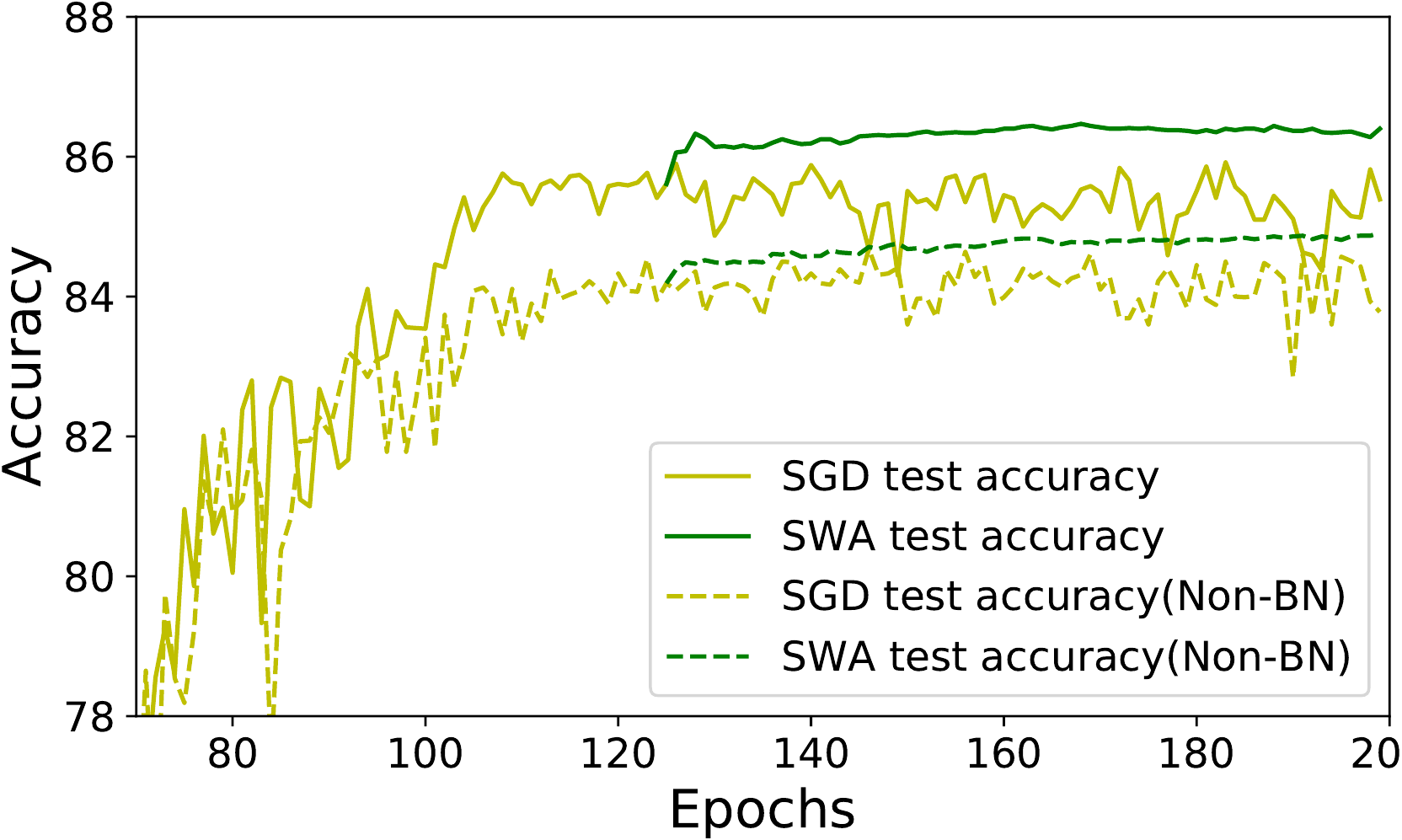}
		\caption{Test accuracy of ResNet-8 with and without BN layers, after running weight averaging (SWA).} 
		\label{fig:BN_compare}
	\end{figure}
	
	As shown in figure \ref{fig:BN_compare}, we start weight averaging at the $126$-th epoch. Although in both networks, we observe an improvement in test accuracy after averaging, it is clear that 
	the network with BN layers have larger improvement compared with the network without BN layers. 
	This again indicates that SGD averaging is more effective on BN parameters. 
	
	The results presented above are still quite preliminary. Understanding how the asymmetric valleys are formed in deep networks might be a valuable future research direction.

	
	\section{Conclusion}
	The width of solutions has been used to explain generalization. In this paper, we elaborate on these arguments, and show that width along \emph{Asymmetric Valleys}, where the loss may increase at different rates along two opposition directions, is especially important for explaining generalization.
	Based on a formal definition of asymmetric valley, we showed that a biased solution lying on the flat side of the valley generalizes better than the empirical minimizer. Further, it is proved that by averaging the points along the SGD trajectory naturally leads to such biased solution. 
	We have conducted extensive experiments with state-of-the-art deep models to verify our theorems. We hope this paper will strengthen our understanding on the loss landscape of deep neural networks, and inspire new theories and algorithms that further improve generalization.
	
	\bibliography{ref}
	\bibliographystyle{icml2019}
	
	\clearpage
	
	\appendix
	\onecolumn

	\section{Additional Figures for Section \ref{subsec:find_asym}: Asymmetric Directions}
	\label{appendix_asy_valleys}
	See Figure \ref{missing:SGD_asym_res164_c10},
	Figure \ref{missing:SGD_asym_dense_c10}, 
	Figure \ref{missing:SGD_res110_C100}, 
	Figure \ref{missing:SGD_res164_C100}, 
	and Figure \ref{missing:SGD_densenet_C100}.
	
	\begin{figure}[htbp]
		\centering
		\includegraphics[width=.45\textwidth]{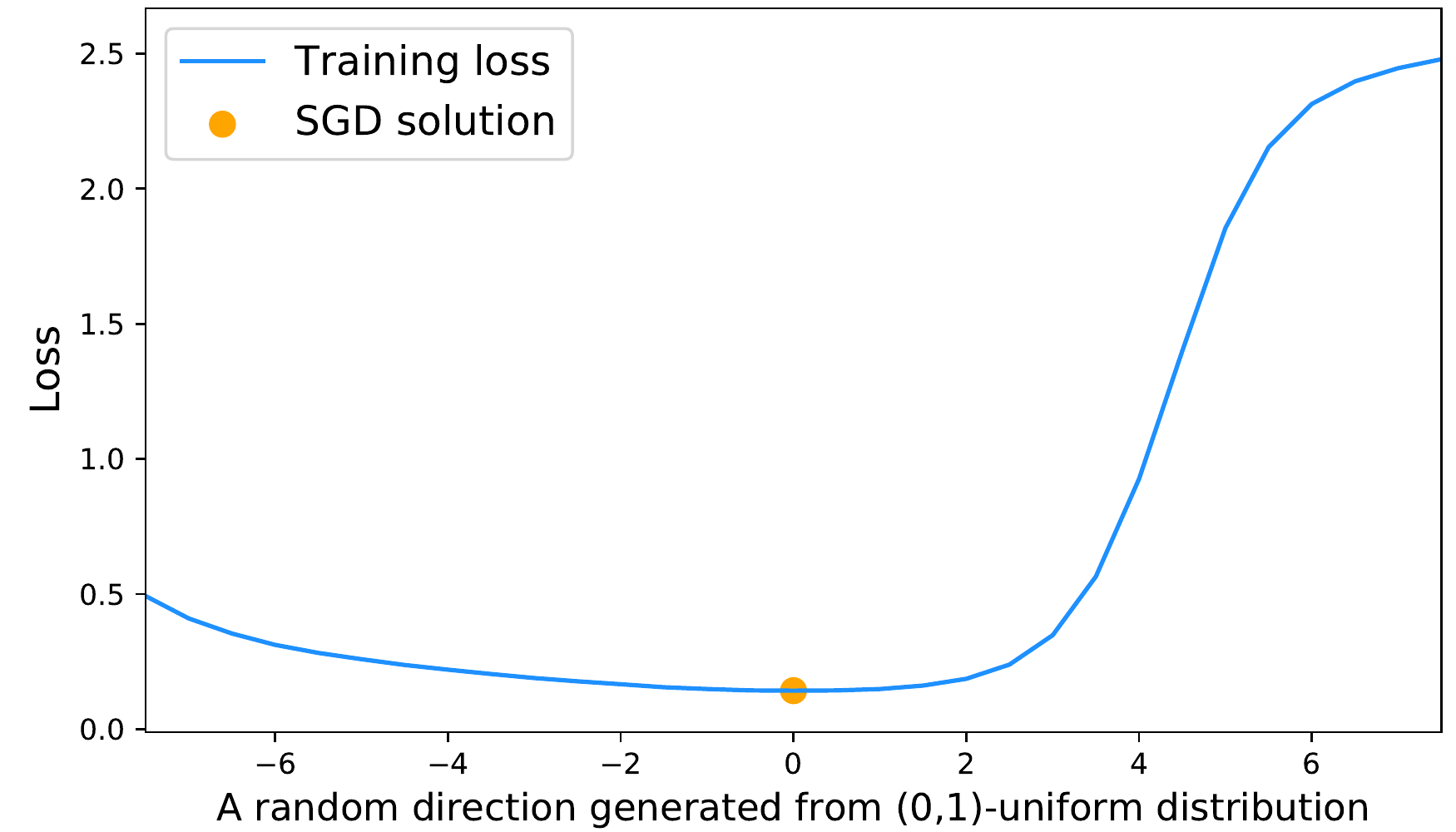}
		\caption{Asymmetric direction for a solution of ResNet-164 on CIFAR-10. $(r,p,c, \zeta)=(4.0, 0.0270, 12.1, 2.0)$. }
		\label{missing:SGD_asym_res164_c10} 
	\end{figure}
	
	\begin{figure}[htbp]
		\centering
		\includegraphics[width=.45\textwidth]{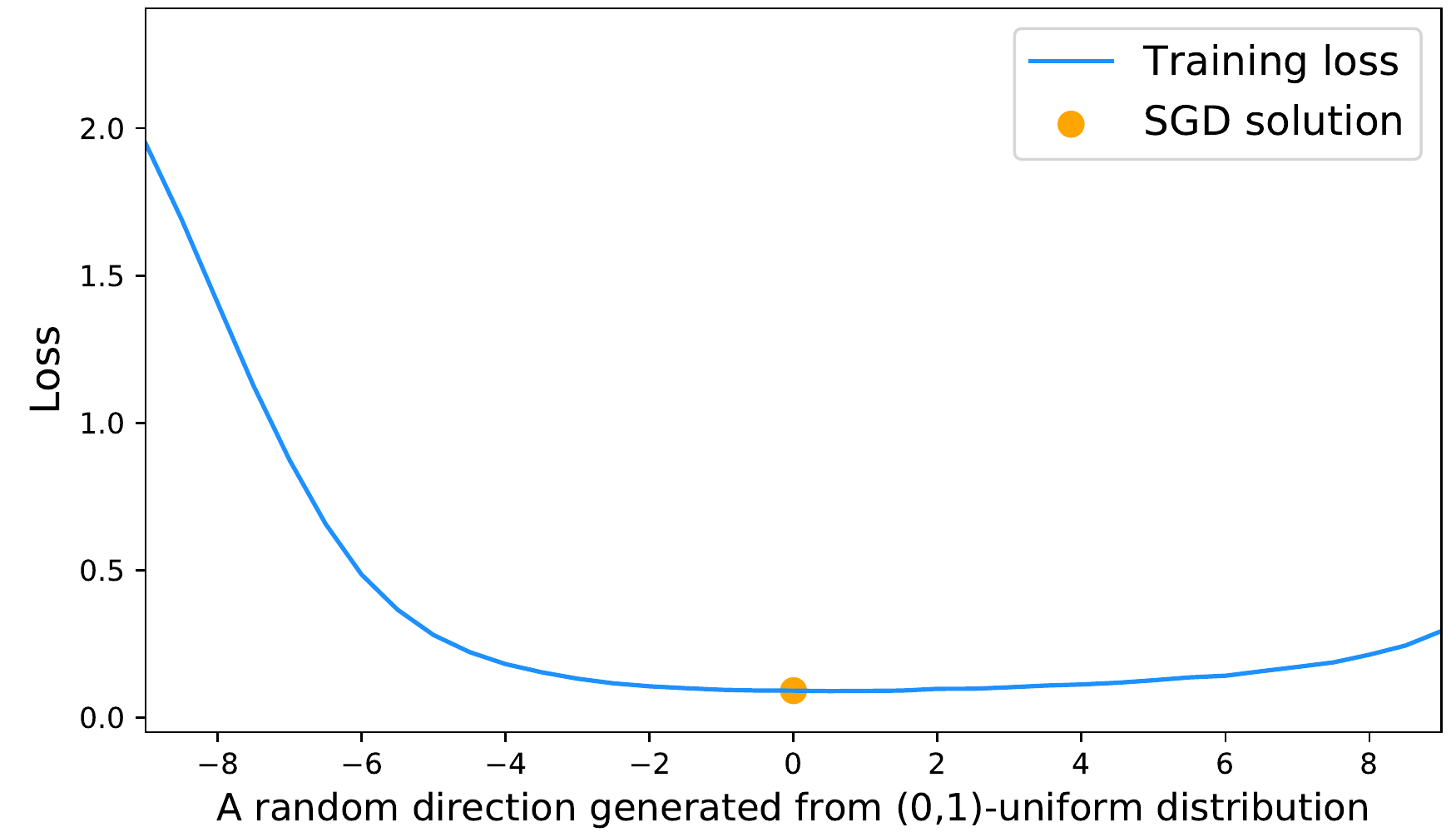}
		\caption{Asymmetric direction for a solution of DenseNet-100 on CIFAR-10. $(r,p,c, \zeta)=(5.0,0.00022,452.5,1.5)$. }
		\label{missing:SGD_asym_dense_c10} 
	\end{figure}

	\begin{figure}[htbp]
		\centering
		\includegraphics[width=.45\textwidth]{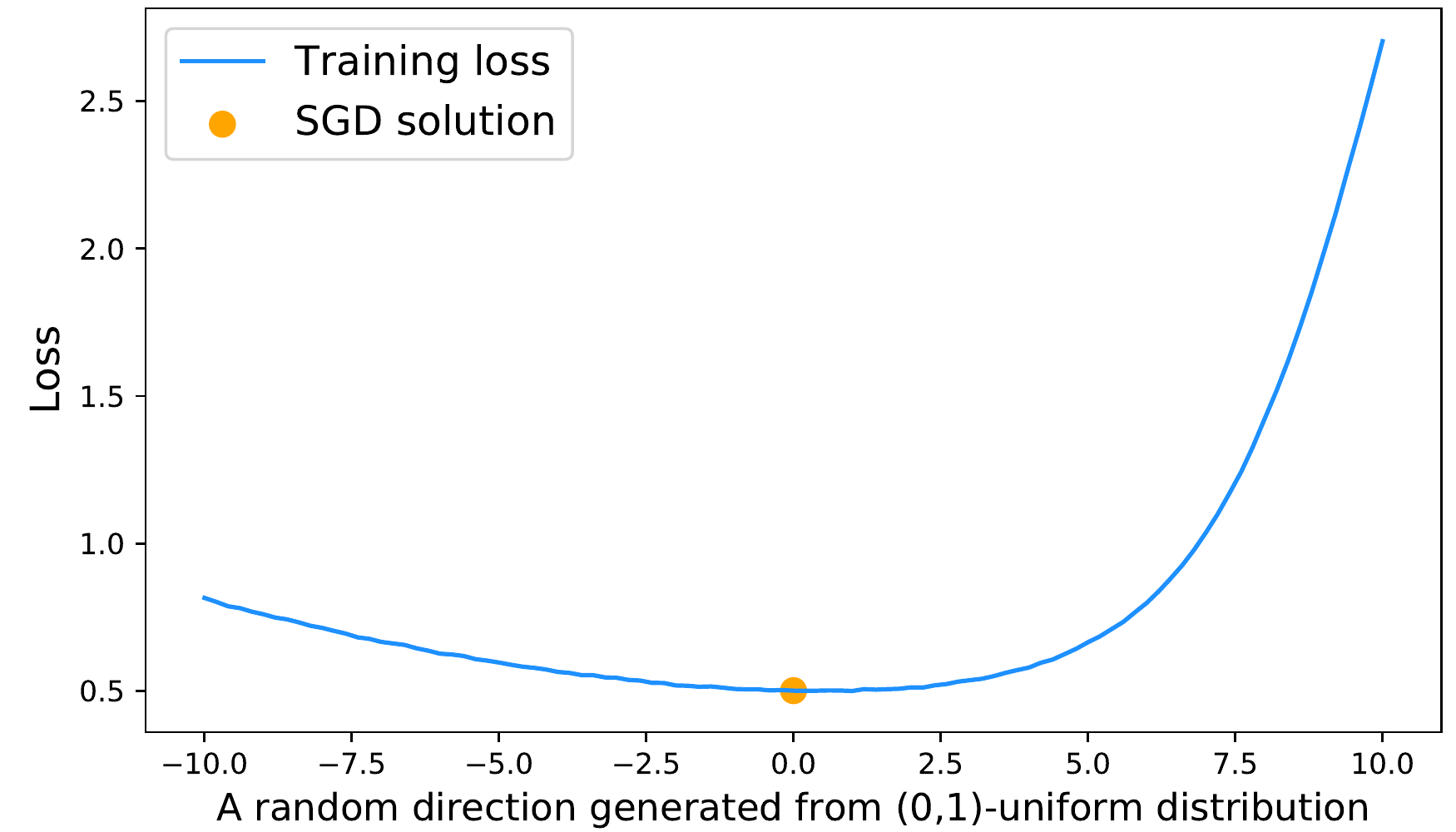}
		\caption{Asymmetric direction for a solution of ResNet-110 on CIFAR-100. $(r,p,c, \zeta)=(8.0,0.032,5.52,4.0)$.}
		\label{missing:SGD_res110_C100} 
	\end{figure}

	\begin{figure}[htbp]
		\centering
		\includegraphics[width=.45\textwidth]{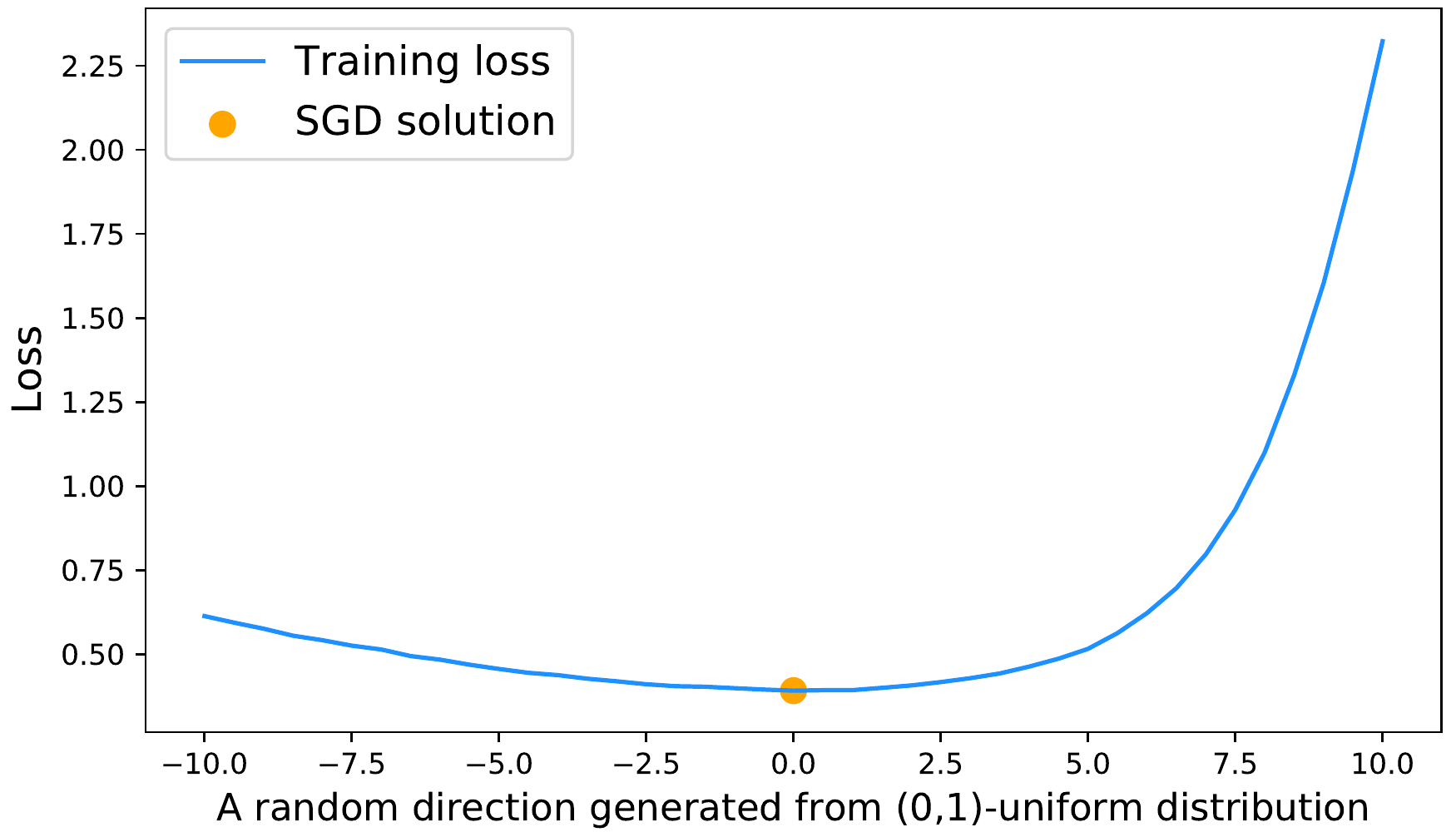}
		\caption{Asymmetric direction for a solution of ResNet-164 on CIFAR-100. $(r,p,c, \zeta)=(7.0,0.0175,7.66,3.0)$.}
		\label{missing:SGD_res164_C100} 
	\end{figure}
	
	\begin{figure}[ht]
		\centering
		\includegraphics[width=.45\textwidth]{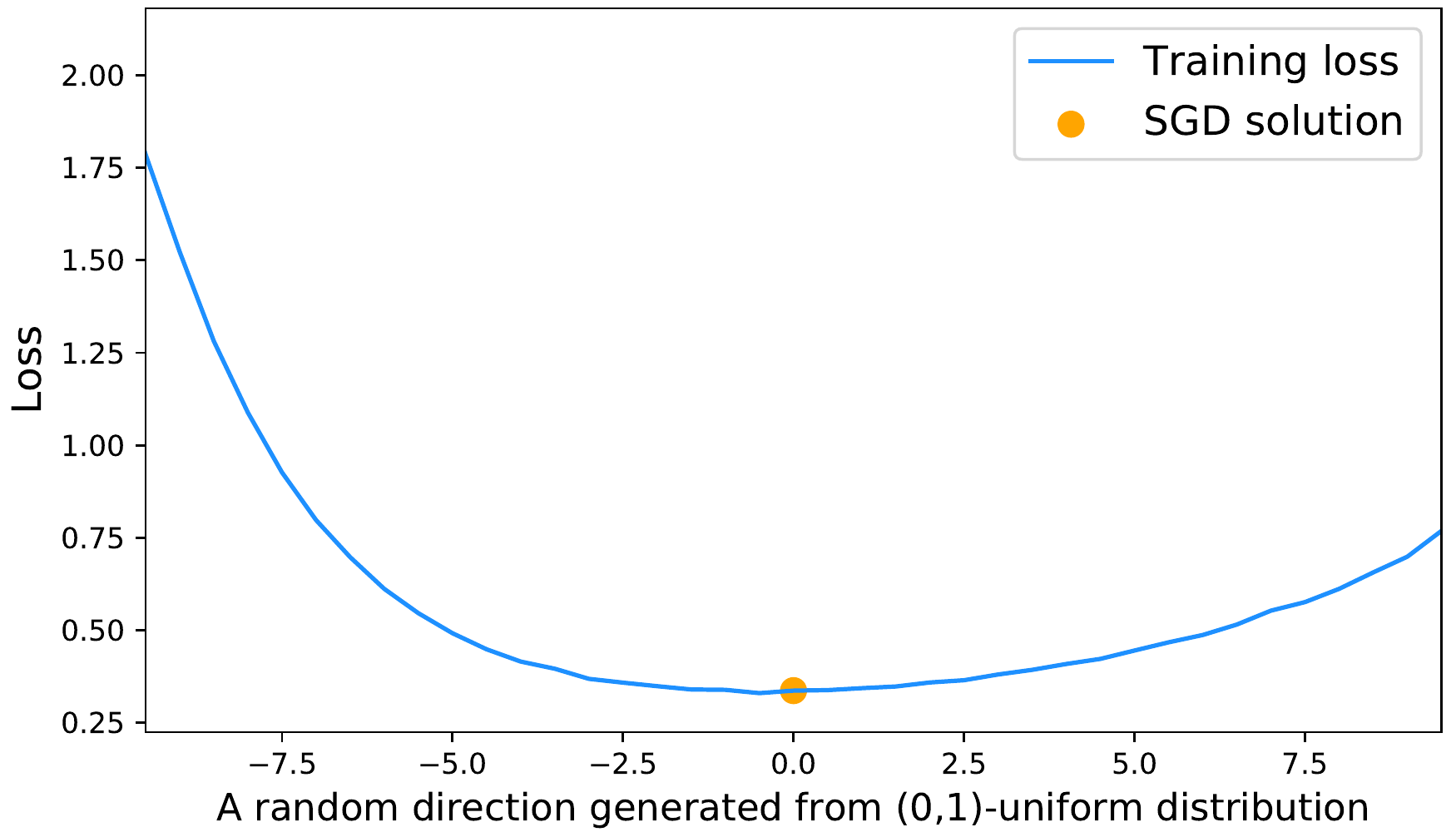}
		\caption{Asymmetric direction for a solution of DenseNet-100 on CIFAR-100. $(r,p,c, \zeta)=(5,0.0198,2.259,2.5)$.}
		\label{missing:SGD_densenet_C100} 
	\end{figure}

	\section{Missing Proof for Theorem \ref{thm:dimd}}
	\label{appendix:proof_dimd}
	\begin{proof}
		Since $\boldsymbol{\delta}$ has $2^d$ possible value for a given $\bar {\boldsymbol{\delta}}$, we can use an integer $j\in \{0, \cdots, 2^d-1\}$ to represent each value. When writing $j$ in binary, its $i$-th digit represents whether ${\delt}_i=\bdi $ (equal to $1$) or ${\delt}_i= -\bdi$ (equal to $0$). We use $j\wedge 2^i$ to represent the bitwise AND operator between $j$ and $2^i$, which equals $0$ if the $i$-th digit of $j$ is $0$. 
		
		To prove our theorem, it suffices to show that
		for any $i\in [k]$, 
		\begin{align}
		\E_{\boldsymbol{\delta}}\LL\left (\boldsymbol{\hat w}^*
		+\sum_{i_0=1}^{i-1}\bl_{i_0}  \bu_{i_0}
		\right )
		-\E_{\boldsymbol{\delta}}\LL\left (\boldsymbol{\hat w}^* +\sum_{i_0=1}^i \bl_{i_0}  \bu_{i_0}\right )
		\geq  (c_i-1)\bli p_i/2
		-2\xi>0
		\label{inequ:step_by_step}
		\end{align}
		If (\ref{inequ:step_by_step}) is true, it suffices to take summation over $i$ on both sides, and we will get our conclusion. Therefore, below we will prove (\ref{inequ:step_by_step}).
		
		\begin{align}
		&\E_{\boldsymbol{\delta}}\LL\left (\boldsymbol{\hat w}^*
		+\sum_{i_0=1}^{i-1}\bl_{i_0}  \bu_{i_0}
		\right )-\min_{\boldsymbol{w}} \LL(\boldsymbol{w})  
		+ \min_{\boldsymbol{w}} \hat \LL(\boldsymbol{w}) \nonumber 
		\\
		=&\E_{\boldsymbol{\delta}}\LL'\left (\boldsymbol{\hat w}^*
		+\sum_{i_0=1}^{i-1}\bl_{i_0}  \bu_{i_0}
		\right ) \nonumber\overset{\circled{1}}\geq  \frac{1}{2^d}\sum_{j=0}^{2^d-1}
		\hat \LL\left (\boldsymbol{\hat w}^* 
		+\sum_{i_0=1}^{i-1}\bl_{i_0}  \bu_{i_0}+ \boldsymbol{\delta}^j\right )-\xi  \nonumber \\
		= & \frac{1}{2^d}\sum_{
			\substack{j=0\\
				j\wedge 2^i=0 } 
		}^{2^d-1}
		\left [ \hat\LL\left (\boldsymbol{\hat w}^* +
		\sum_{i_0=1}^{i-1}\bl_{i_0}  \bu_{i_0}+
		\boldsymbol{\delta}^{j }\right )
		+
		\hat\LL\left (\boldsymbol{\hat w}^* + 
		\sum_{i_0=1}^{i-1}\bl_{i_0}  \bu_{i_0}+
		\boldsymbol{\delta}^{j + 2^i}\right ) \right ] -\xi\label{eqn:split_into_pairs}
		\end{align}
		Where $\circled{1}$ holds by Assumption \ref{assump:random_shift_assumption}, and 
		the fact that $	\|\sum_{i_0=1}^{i-1}\bl_{i_0}  \bu_{i_0}\|_2\leq \|\bl\|_2=R$.
		For every $j$ s.t. $j\wedge2^i=0$, 
		\begin{align*}
		&\bhw +\sum_{i_0=1}^{i}\bl_{i_0}  \bu_{i_0}+ \boldsymbol{\delta}^j\\
		=&
		\bhw + \sum_{i_0=1}^{i}\bl_{i_0}  \bu_{i_0}+\boldsymbol{\delta}^j
		+
		\langle 
		\boldsymbol{\delta}^j, \bu^i \rangle \bu^i- \langle 
		\boldsymbol{\delta}^j, \bu^i \rangle \bu^i\\ 
		=&
		\bhw + \sum_{i_0=1}^{i-1}\bl_{i_0}  \bu_{i_0}+\boldsymbol{\delta}^j
		-
		\bdi  \bu^i  - \langle 
		\boldsymbol{\delta}^j, \bu^i \rangle \bu^i+\bli \bu^i \\ 
		= &
		\bhw + \sum_{i_0=1}^{i-1}\bl_{i_0}  \bu_{i_0}+\boldsymbol{\delta}^j
		- \langle 
		\boldsymbol{\delta}^j, \bu^i \rangle \bu^i+
		(\bli-\bdi )\bu^i 
		\end{align*}
		Since $\|\sum_{i_0=1}^{i-1}\bl_{i_0}  \bu_{i_0}\|_2\leq \|\bl\|_2$, $\|\delta^j\|_2=\|\bd\|_2$,
		we know that $\forall j, 
		\sum_{i_0=1}^{i-1}\bl_{i_0}  \bu_{i_0}+
		\boldsymbol{\delta}^j\in \B(R')$.
		By Assumption \ref{assump:Locally_identical_assumption},  for every $i\in [k]$,  $\bu^i$ is asymmetric with respect to $\bhw + 
		\sum_{i_0=1}^{i-1}\bl_{i_0}  \bu_{i_0}+
		\boldsymbol{\delta}^j- \langle 
		\boldsymbol{\delta}^j, \bu^i \rangle \bu^i 
		$. 
		Since $\bli\leq \bdi - \zeta$, we have
		$ \bli-\bdi  < - \zeta$. By the definition of asymmetric direction, we know 
		\begin{align}
		\hat \LL\left (
		\bhw +
		\sum_{i_0=1}^{i-1}\bl_{i_0}  \bu_{i_0}
		+ \boldsymbol{\delta}^j\right )
		\geq
		\hat \LL\left (
		\bhw 
		+\sum_{i_0=1}^{i}\bl_{i_0}  \bu_{i_0}	
		+ \boldsymbol{\delta}^j
		\right 	)
		+ c_i \bli p_i
		\label{inequ:left}  
		\end{align}
		
		Similarly, 
		\begin{align*}
		&\bhw
		+\sum_{i_0=1}^{i}\bl_{i_0}  \bu_{i_0}	
		+ \boldsymbol{\delta}^{j + 2^i}\\
		=&
		\bhw 
		+\sum_{i_0=1}^{i-1}\bl_{i_0}  \bu_{i_0}	
		+ \boldsymbol{\delta}^{j + 2^i}
		+
		\langle 
		\boldsymbol{\delta}^{j + 2^i}, \bu^i \rangle \bu^i- \langle 
		\boldsymbol{\delta}^{j + 2^i}, \bu^i \rangle \bu^i+\bli \bu^i \\ 
		=&
		\bhw
		+\sum_{i_0=1}^{i-1}\bl_{i_0}  \bu_{i_0}	
		+ \boldsymbol{\delta}^{j + 2^i}
		- \langle 
		\boldsymbol{\delta}^{j + 2^i}, \bu^i \rangle \bu^i+(\bdi+ \bli) \bui  
		\end{align*}
		Since $\bli\leq r-\bdi  $, we have 
		$\bdi+\bli\leq r$. Therefore, 
		\begin{align}
		\hat \LL\left (
		\bhw 	+\sum_{i_0=1}^{i-1}\bl_{i_0}  \bu_{i_0}	+ \boldsymbol{\delta}^{j+2^i}\right )
		\geq
		\hat \LL\left (
		\bhw 	+\sum_{i_0=1}^{i}\bl_{i_0}  \bu_{i_0}	+ \boldsymbol{\delta}^{j+2^i}
		\right 	)
		-\bli p_i  
		\label{inequ:right}
		\end{align}
		
		Combining (\ref{inequ:left}) and (\ref{inequ:right}), we have,
		\begin{align*}
		(\ref{eqn:split_into_pairs})\geq  & \frac{1}{2^d}\sum_{
			\substack{j=0\\
				j\wedge 2^i=0 } 
		}^{2^d-1}
		\left [ \hat\LL\left (\boldsymbol{\hat w}^* 
		+\sum_{i_0=1}^{i}\bl_{i_0}  \bu_{i_0}	
		+ \boldsymbol{\delta}^{j}\right )
		+c_i\bli p_i+
		\hat\LL\left (\boldsymbol{\hat w}^* 
		+\sum_{i_0=1}^{i}\bl_{i_0}  \bu_{i_0}	+
		\boldsymbol{\delta}^{j + 2^i}\right ) -\bli p_i\right ] 
		-\xi \\
		= & \frac{1}{2^d}\sum_{
			j=0}^{2^d-1}
		\left [ \hat\LL\left (\boldsymbol{\hat w}^* 
		+\sum_{i_0=1}^{i}\bl_{i}  \bu_{i_0}	+ \boldsymbol{\delta}^{j}\right )
		\right ] +(c_i-1)\bli p_i/2
		-\xi \\
		\overset{\circled{2}}\geq & 
		\E_{\boldsymbol{\delta}}\LL'\left (\boldsymbol{\hat w}^* 
		+\sum_{i_0=1}^i \bl_{i_0}  \bu_{i_0}	\right )
		+ (c_i-1)\bli p_i/2
		-2\xi\\
		=&  \E_{\boldsymbol{\delta}}\LL\left (\boldsymbol{\hat w}^* +\sum_{i_0=1}^i \bl_{i_0}  \bu_{i_0}\right )
		-  \min_{\boldsymbol{w}} \LL(\boldsymbol{w})  
		+ \min_{\boldsymbol{w}} \hat \LL(\boldsymbol{w}) 
		+ (c_i-1)\bli p_i/2
		-2\xi
		\end{align*}
		
		Where $\circled{2}$ holds by Assumption \ref{assump:random_shift_assumption} and the fact that $\|\sum_{i_0=1}^i \bl_{i_0}  \bu_{i_0}\|_2\leq \|\bl\|_2=R$. 
		That means, 
		\begin{align*}
		\E_{\boldsymbol{\delta}}\LL\left (\boldsymbol{\hat w}^*
		+\sum_{i_0=1}^{i-1}\bl_{i_0}  \bu_{i_0}	
		\right )
		\geq  \E_{\boldsymbol{\delta}}\LL\left (\boldsymbol{\hat w}^* 	+\sum_{i_0=1}^{i}\bl_{i_0}  \bu_{i_0}	\right )
		+ (c_i-1)\bli p_i/2
		-2\xi
		>0
		\end{align*}
		Where the last inequality holds as
		$\bli > \frac{4\xi }{(c_i-1)p_i}$.
		
	\end{proof}
	\section{Additional Figures for Section \ref{sec:verify_assump}: Shift Exists Empirically}
	\label{appendix_shift}
	See Figure \ref{missing:shift_on_asym_dense100c100},  Figure \ref{missing:shift_on_asym_res164c10}, and 
	Figure  \ref{missing:shift_on_sym_res110c100}.
	\begin{figure}[htbp]
		\centering
		\includegraphics[width=.4\textwidth]{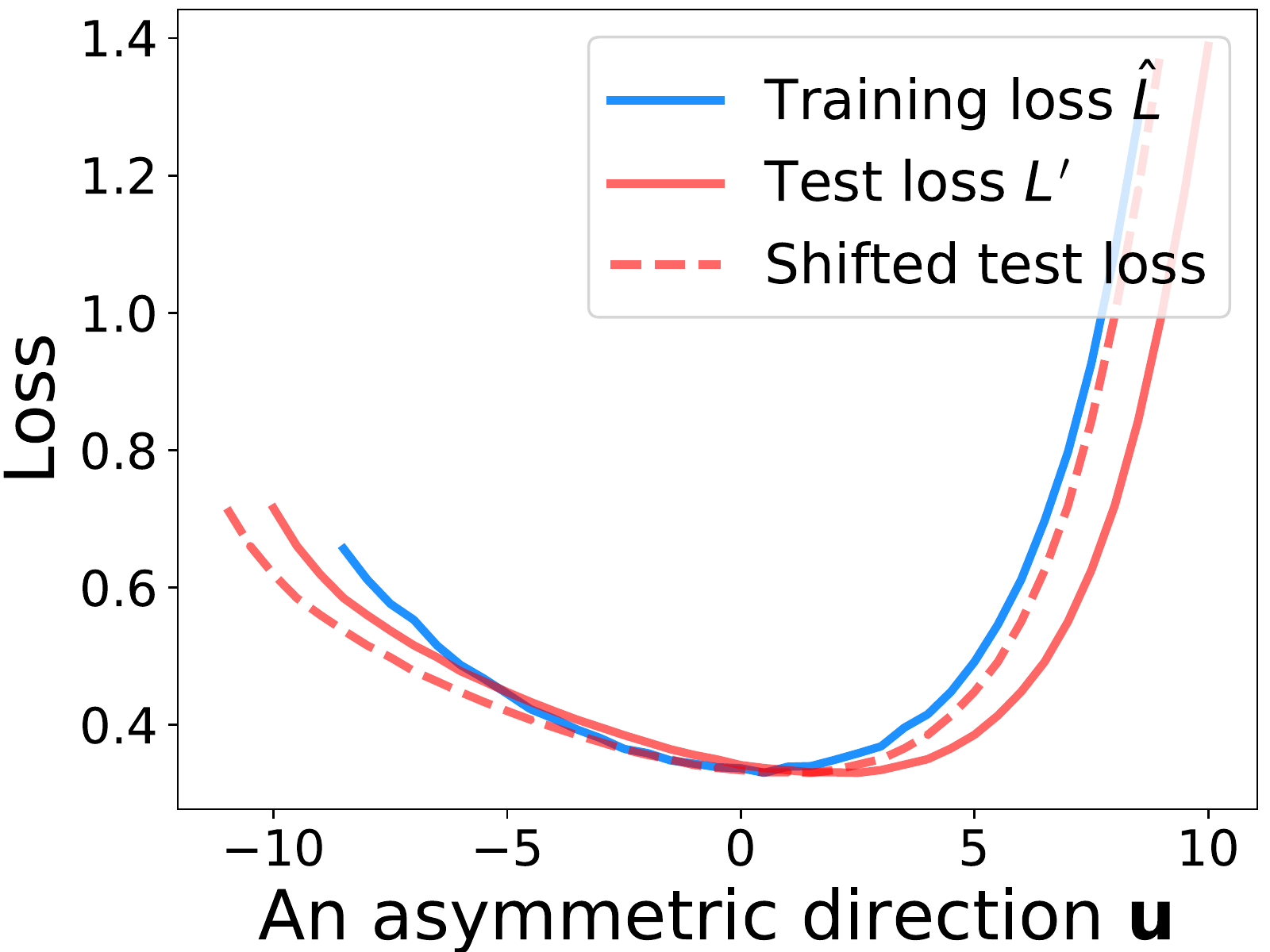}
		\caption{Shift on asymmetric direction (DenseNet-100 on CIFAR-100), $\xi_{\boldsymbol{\delta}=1}\!=\!0.119$, $\xi_{\boldsymbol{\delta}=0}\!=\!0.439$}
		\label{missing:shift_on_asym_dense100c100} 
	\end{figure}
	
	\begin{figure}[htbp]
		\centering
		\includegraphics[width=.4\textwidth]{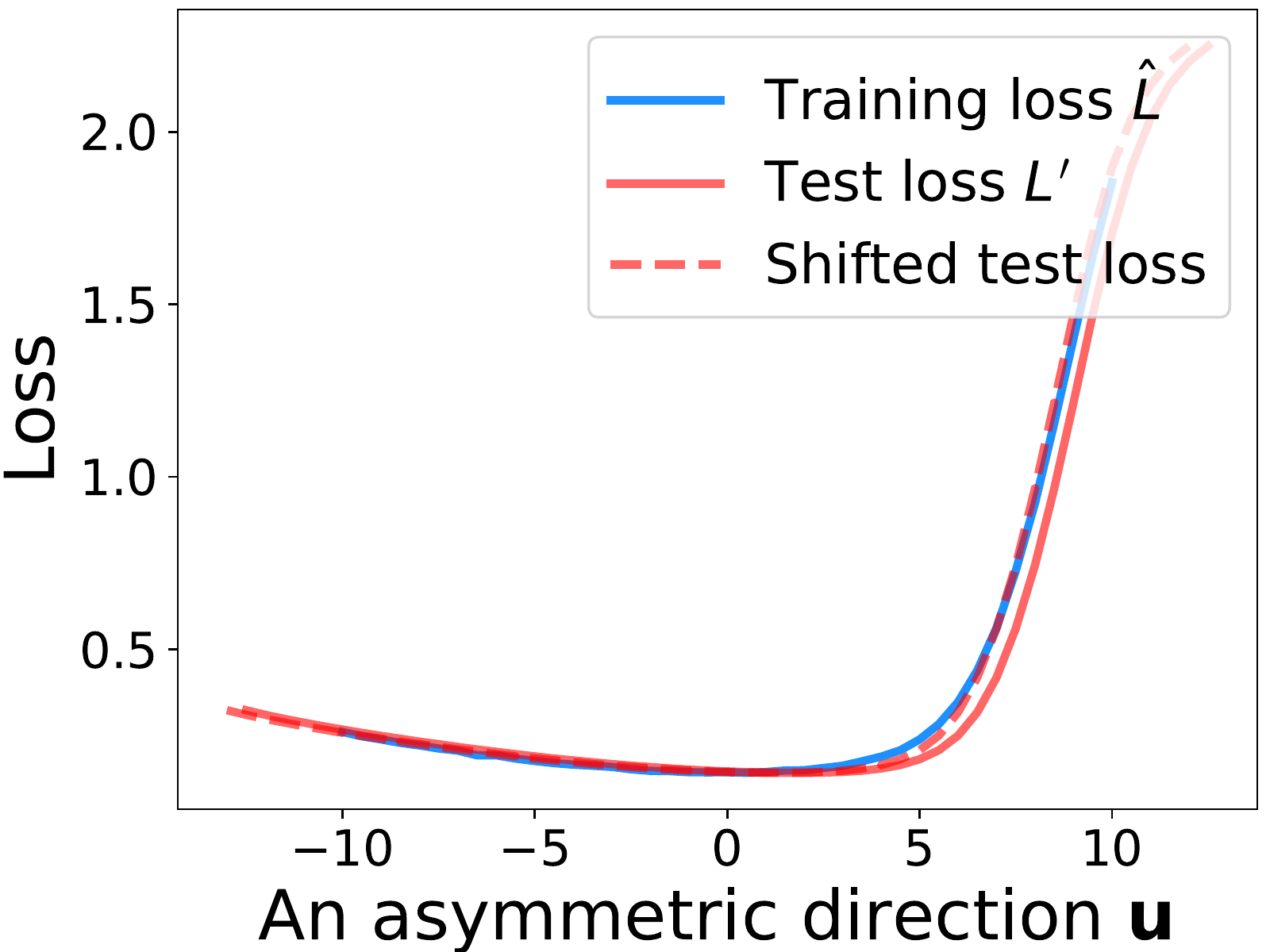}
		\caption{Shift on asymmetric direction (ResNet-164 on CIFAR-10), $\xi_{\boldsymbol{\delta}=0.5}\!=\!0.0699$, $\xi_{\boldsymbol{\delta}=0}\!=\!0.189$}
		\label{missing:shift_on_asym_res164c10} 
	\end{figure}

	\begin{figure}[htbp]
		\centering
		\includegraphics[width=.4\textwidth]{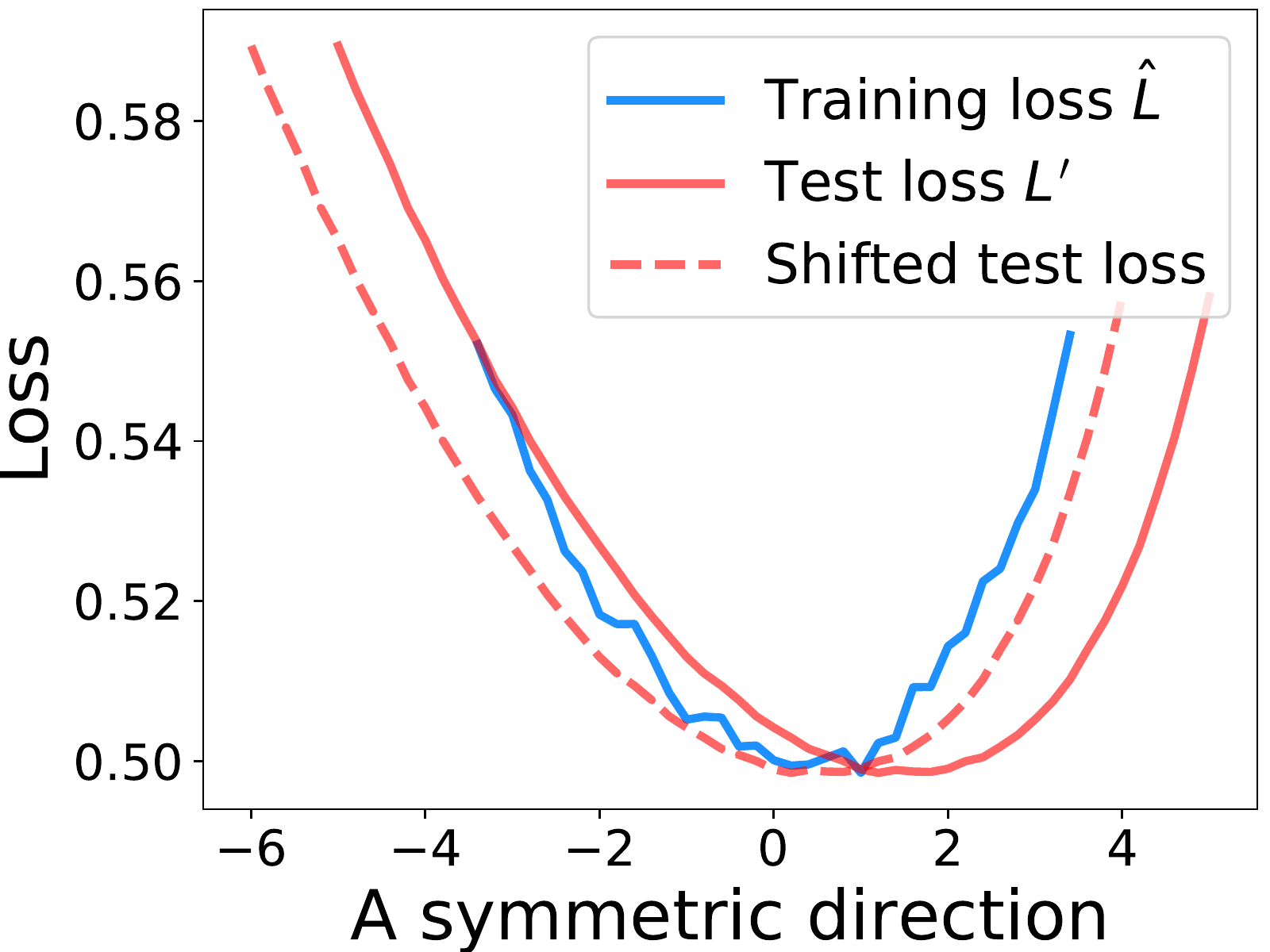}
		\caption{Shift on symmetric direction (ResNet-110 on CIFAR-100), $\xi_{\boldsymbol{\delta}=1}\!=\!0.0197$, $\xi_{\boldsymbol{\delta}=0}\!=\!0.0431$}
		\label{missing:shift_on_sym_res110c100} 
	\end{figure}

	\section{Additional Figures in Section \ref{sec:avg_is_good}: Averaging Works For Symmetric Case}
	\label{appendix:other_sgd_pattern}
	If the function is symmetric, there are two possible cases, as we show in Figure \ref{missing:sym_soft} and Figure \ref{missing:Symmetric perturbation}. 
	On one hand, if the function is flat, SGD is likely to stay on one side of the function along the trajectory, and the average will have bias on that side. 
	On the other hand, if the function is sharp, SGD is likely to oscillate between the two sides, and therefore the average of the iterates will concentrate around the center.  In both cases, SGD averaging could help to create bias on flat sides or to denoise.

	\begin{figure}[htbp]
		\centering
		\includegraphics[width=.40\textwidth]{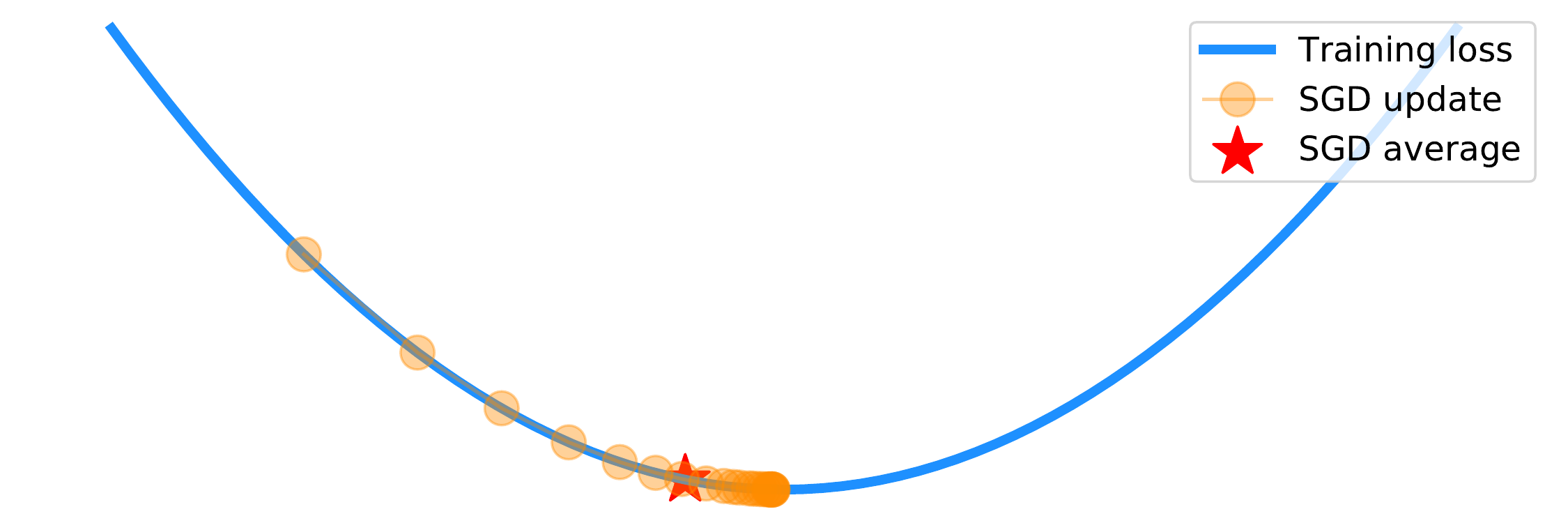}
		\caption{Symmetric function with flat sides} 
		\label{missing:sym_soft} 
	\end{figure}
	
	\begin{figure}[htbp]
		\centering
		\includegraphics[width=.24\textwidth]{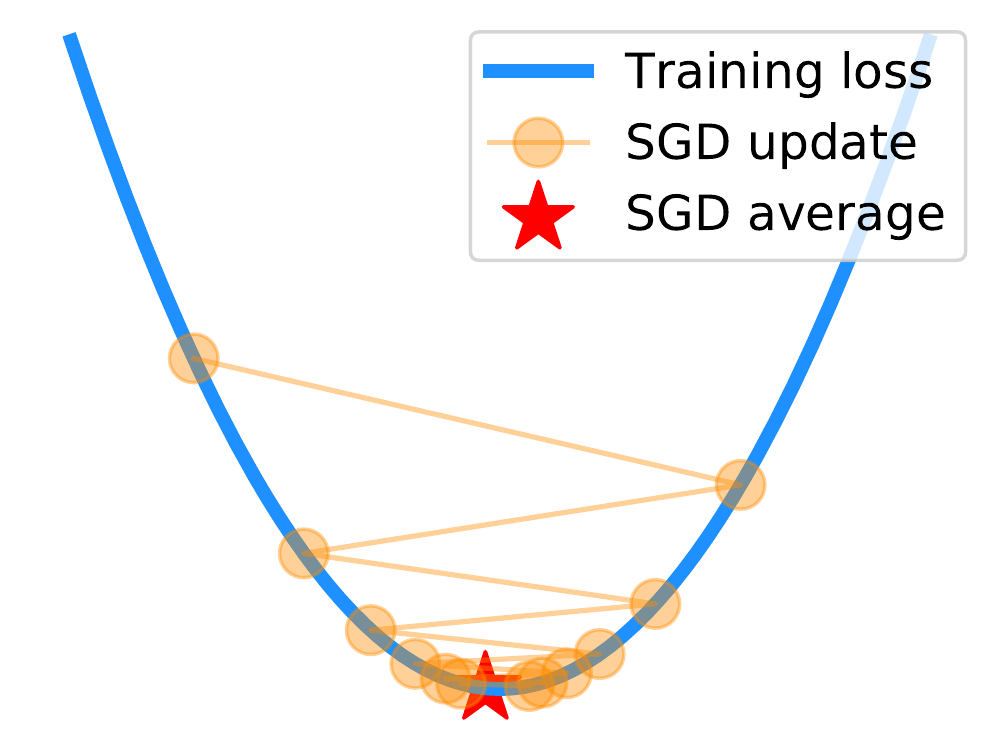}
		\caption{Symmetric function with sharp sides}
		\label{missing:Symmetric perturbation} 
	\end{figure}

	\section{Missing Proof for Theorem \ref{thm:asym_avg}}
	\label{appendix_main_theorem}
	
	To prove Theorem \ref{thm:asym_avg}, we will need the following concentration bound.
	\begin{lem}[Azuma's inequality]
		Let $X_1, X_2, X_3,...X_n$ be independent random variables satisfying $|X_i - \E[X_i]| \leq c_i$, for $1 \leq i \leq n$. We have the following bound for $X = \sum_{i=1}^{n}X_i$:
		$$\Pr(|X-\E(X)| \geq \lambda) \leq 2 e^{-\frac{\lambda^2}{2\sum_{i=1}^{n}c_i^2}}$$
	\end{lem}
	
	Let $\Pm\triangleq -\eta (\UL +\UR +2\nu )$, 
	$\Pa \triangleq -\eta (\Ll -\nu )$.
	Since $-\UL=c \UR$, we know 
	$\Pm> (c-1)\eta \UR - 2\eta \nu$. 
	First, we have the following bounds on the first step $w_0$. 
	
	\begin{lem}
		\label{lem:w0bounds}
		For every $i\in [h]$, $w_0\in [\Pm, \Pa]$. 
	\end{lem}
	\begin{proof}
		Since $w_0$ is the first step that SGD jumps from the flat side to the sharp side,
		denote the previous location as $w_{-1}<0$.
		Since $w_{-1}$ is at the sharp side, we know that the gradient is 
		$\nabla \LL(w_{-1})\leq \UL$. 
		Therefore, we have
		\[
		w_0 = w_{-1}-\eta (\nabla \LL(w_{-1})+\omega_{-1})
		\]
		Where $\omega_{-1}$ is the noise bounded by $\nu $. 
		
		At the time when SGD jump from the flat side to sharp side, denote the target position as $w'_{-1}$. We know that $w'_{-1}\in [-\eta(\UR+\nu ), 0]$. Since the gradient on the sharp side is at most $\UL$, we know the next step is lower bounded by 
		$-\eta(\UR+2\nu +\UL)=\Pm>0 $. In other words, SGD stays at the sharp side for only $1$ iterations (this matches with our empirical observation, see e.g. Figure \ref{fig:Asy_onedim_occillate}).

		That means, the bound on $w'_{-1}$ can be applied to $w_{-1}$ as well, because they are the same iterate. By applying the upper and lower bound on $\nabla \LL(w_{-1})$, we get:
		
		\[
		w_0 \geq -\eta(\UR+\nu )- \eta (\UL+\nu )=\Pm
		\]
		and also 
		\[
		w_0 \leq 0- \eta (\Ll-\nu )=\Pa
		\]
	\end{proof}

	Below we first define 
	$\Tmp \triangleq 
	\left (\frac{-\sqrt{2}\nu \log^{1/2}(2\tau)
		+\sqrt{2\nu ^2\log(2\tau)-4\UR (\UL +\UR +2\nu )}
	}{2\UR}\right )^2
	$, where $\tau$ is a constant with value to be set later. 
	$\Tmp$ satisfies the following inequality.

	\begin{lem}
		\label{lem:tmp}
		$	\forall t \leq \Tmp,  \Pm  - t \eta \UR
		- \sqrt{2t} \eta \nu \log^{1/2}(2\tau)\geq 0$.
	\end{lem}
	\begin{proof}
		By the definition of $\Pm$, we have
		\begin{align*}
		& -\eta (\UL +\UR +2\nu) - t \eta \UR
		- \sqrt{2t} \eta \nu \log^{1/2}(2\tau)\geq 0\\
		\Leftarrow&
		(\UL +\UR +2\nu) + t \UR
		+ \sqrt{2t} \nu\log^{1/2}(2\tau)\leq 0\\
		\Leftarrow&
		(\UL +\UR +2\nu) + \Delta^2 \UR
		+ \sqrt{2}\Delta r\log^{1/2}(2\tau)\leq 0 ~~~~~(\mathrm{\Delta\triangleq \sqrt{t}})\\
		\Leftarrow&
		\Delta \in 
		\left [
		0,
		\frac{-\sqrt{2} \nu\log^{1/2}(2\tau)
			+\sqrt{2\nu^2\log(2\tau)-4\UR (\UL +\UR +2\nu)}
		}{2\UR}
		\right ]\\
		\Leftarrow&
		t \leq 
		\left (\frac{-\sqrt{2} \nu\log^{1/2}(2\tau)
			+\sqrt{2\nu^2\log(2\tau)-4\UR (\UL +\UR +2\nu)}
		}{2\UR}\right )^2 \qedhere
		\end{align*}
	\end{proof}
	
	Now, we have the following theorem that says  with decent probability, the minimum number of iterates on the flat side in $i$-th round is at least $\Tmp$.

	\begin{thm} 
		\label{thm:tmp}
		If we start at $w_0\geq \Pm $, for every fixed $\tau>\Tmp $, with probability at least $1-\frac{\Tmp}{\tau}$, we have $\forall t\leq \Tmp, w_t>  w_0 - t \eta \UR
		- \sqrt{2t} \eta \nu\log^{1/2}(2\tau)\geq 0$.
	\end{thm}
	
	\begin{proof}
		Define filtration $\Ft = \sigma\{\omega_0,\cdots, \omega_{t-1}\}$, where $\sigma\{\cdot\}$ denotes the sigma field.
		Define the event $\EE_T=\{
		\forall t\leq T, w_t>w_0 - t\eta \UR
		- \sqrt{2t} \eta \nu\log^{1/2}(2\tau) \}$ and define $G_t= w_0 - w_t - t \eta \UR + M $, where 
		$M\triangleq (\Tmp+1)(w_0+\nu + 2\eta \UR )$.
		Since we only consider the case $t \leq \Tmp$, we have
		\begin{equation*}
		G_t= w_0 - w_t - t \eta \UR + (\Tmp+1)(w_0+\nu + 2\eta \UR )> 
		w_0 - w_t - t \eta \UR + 
		w_t + t\eta \UR >0
		\end{equation*} 
		Therefore, $G_t$ is always positive. By SGD updating rule, we have
		\begin{align}
		& 
		\E[G_{t+1}\one_{\EE_t} | \mathcal{F}_{t} ] 
		=  \E[(w_0 - w_{t+1}-(t+1) \eta \UR + M)\one_{\EE_t}|\mathcal{F}_{t}] \nonumber \\
		\leq &  \E[(w_0-w_t +\eta \omega_t-t \eta \UR + M)\one_{\EE_t}|\Ft ]
		=   w_0-w_t -t \eta \UR + M 
		=   G_t\one_{\EE_t} 
		\label{eqn1}
		\end{align}
		
		Since $\one_{\EE_t}\leq \one_{\EE_{t-1}}$, and $G_t$ is always positive, we have
		\begin{equation}G_t\one_{\EE_t}
		\leq 
		G_t\one_{\EE_{t-1}}
		\label{eqn:Gt:2}
		\end{equation}
		Combining (\ref{eqn1}) and (\ref{eqn:Gt:2}) together, we know $G_t\one_{\EE_{t-1}}$ is a supermartingale.
		
		We can also bound the absolute value of the difference in every iteration: 
		\begin{align*}
		&|G_{t+1}\one_{\EE_{t}}- \E[G_{t+1}\one_{\EE_{t}}|\F_t]|\\
		= & |(w_0 - w_{t+1}-(t+1) \eta \UR +M)
		-(w_0-w_{t}-
		\nabla \LL(w_t)-
		(t+1) \eta \UR +M)
		|\F_t]\\
		\leq & \eta \nu
		\end{align*}
		
		By Azuma's inequality, we get: 
		\begin{align*}
		\Pr\left (G_{t}\one_{\EE_{t-1}}- G_0 \geq \lambda 
		\right )
		\leq 2e ^{-\frac{\lambda^2}{2 t \eta^2 \nu^2 }}
		\end{align*}
		
		That gives, 
		\begin{align*}
		\Pr\left (G_{t}\one_{\EE_{t-1}}- G_0 \geq \sqrt{2t} \eta \nu\log^{1/2}(2\tau) \right )\leq  1 / \tau
		\end{align*}

		That means, if $\one_{\EE_t-1}$ holds, with probability at least $1-1/ \tau$, 
		
		\[
		w_0-w_t - t\eta \UR+M< \sqrt{2t} \eta \nu\log^{1/2}(2\tau) +G_0
		=
		\sqrt{2t} \eta \nu\log^{1/2}(2\tau)  +M
		\]
		
		Which gives
		\[
		w_t> w_0-  t\eta \UR-\sqrt{2t} \eta \nu\log^{1/2}(2\tau) 
		\]
		
		In other words, that means if $\one_{\EE_{t-1}}$ holds, then $\one_{\EE_{t}}$ also holds with probability at least $1-1/ \tau$.

		Therefore, if we are running $\Tmp$ steps, we know that with probability at least $1-\frac{\Tmp}{\tau}$, $\one_{\EE_{\Tmp}}$ holds.
		Therefore, by Lemma \ref{lem:tmp},
		\begin{equation*}
		\forall t\leq \Tmp,  w_t> w_0 - t \eta \UR
		- \sqrt{2t} \eta \nu\log^{1/2}(2\tau)\geq 
		\Pm - t \eta \UR
		- \sqrt{2t} \eta \nu\log^{1/2}(2\tau)\geq 
		0
		\qedhere
		\end{equation*}
	\end{proof}

	Similarly, we define
	$
	\Tap \triangleq \left ( \frac{
		-\sqrt{2} \nu\log^{1/2}(2\tau)
		+\sqrt{
			2  \nu^2\log (2\tau)
			-4(\Ll-\nu)\LR 
		}
	}{2\LR }\right )^2
	$, which satisfies the following inequality.
	
	\begin{lem}
		\label{lem:tap}
		$\Pa - \Tap \eta \LR
		- \sqrt{2\Tap } \eta \nu\log^{1/2}(2\tau)<0$.
	\end{lem}
	\begin{proof}
		By the definition of $\Pa$, we want to show that
		\begin{align*}
		&  (\Ll-\nu) + \Tap   \LR
		+ \sqrt{2\Tap }  r\log^{1/2}(2\tau)\geq 0
		\end{align*}
		Which holds by the definition of $\Tap$.
	\end{proof}

	The Theorem below shows with decent probability, $\Tap-1$ is an upper bound on the total number of iterates on the flat side in the $i$-th round.
	
	\begin{thm}
		If $w_0\leq \Pa$, 
		with probability at least $1-\frac{\Tap}{\tau}$,
		$w_{\Tap}<0$. 
		\label{thm:tap}
	\end{thm}
	
	\begin{proof}
		Define event $\EE'_T=\{\forall t\leq T, w_t < w_0 - t \eta \LR + \sqrt{2t} \eta \nu \log^{1/2}(2\tau) \}$, and $G_t' = w_t+t \eta \LR>0$.
		
		We have
		\begin{align*}
		&
		\E[G'_{t+1}\one_{\EE'_t} | \mathcal{F}_{t} ] \\
		= & \E[( w_{t+1}+(t+1) \eta \LR)\one_{\EE'_t}|\mathcal{F}_{t}] \\
		\leq & \E[(w_t -\eta \omega_t+t \eta \LR )\one_{\EE'_t}|\mathcal{F}_{t}]\\
		= & ( w_t +t \eta \LR) \one_{\EE'_t}\\ 
		= & G'_t\one_{\EE'_t}
		\end{align*}
		
		Moreover, we know $\one_{\EE'_t}\leq \one_{\EE'_{t-1}}$, which means $G'_t\one_{\EE'_t} \leq  G'_t\one_{\EE'_{t-1}}$. So $G'_t\one_{\EE'_{t-1}}$ is a supermartingale. 
		
		We can also bound the absolute value of the difference in every iteration: 
		\begin{align*}
		&|G'_{t+1}\one_{\EE'_{t}}- \E[G'_{t+1}\one_{\EE'_{t}}|\F_t]|\\
		= & |(w_{t+1}+(t+1) \eta \LR)
		-(w_t-\eta \nabla \LL(w_t)+(t+1) \eta \LR)
		|\F_t]|\\
		\leq & \eta \nu
		\end{align*}
		
		Using Azuma inequality, we get 
		\begin{align*}
		\Pr\left (
		G'_{t}\one_{\EE'_{t-1}}- G'_0 \geq \sqrt{2t} \eta \nu\log^{1/2}(2\tau)
		\right ) 
		\leq 2e ^{-\frac{t \eta^2 \nu^2\log(2\tau)}{ t \eta^2 \nu^2 }}=\frac1 \tau 
		\end{align*}
		
		That means, if $\one_{\EE'_{t-1}}$ holds, with probability at least $1-1/ \tau$, 
		
		\[w_t< w_0-t \eta \LR +  \sqrt{2t} \eta \nu\log^{1/2}(2\tau)\]
		
		In other words, $\one_{\EE'_t}$  also holds. 
		Therefore, if we are running $\Tap$ steps, we know that with probability at least $1-\frac{\Tap}{\tau}$, $\one_{\EE'_{\Tap}}$ holds.
		Therefore, by Lemma \ref{lem:tap}, we know 
		\begin{equation*}
		w_{\Tap}< w_0 - \Tap \eta \LR
		- \sqrt{2\Tap } \eta \nu\log^{1/2}(2\tau)<0 \qedhere
		\end{equation*}
	\end{proof}
	
	\textbf{Remark.} To make sure Theorem \ref{thm:tmp} is not vacuous, we need to make sure that $\Tmp\geq 1$. If we want to make $\Tmp$, say, at least $2$, by Lemma \ref{lem:tmp}, we have:
	\[\Pm  - 2 \eta \UR
	- 2 \eta \nu \log^{1/2}(2\tau)\geq 0\]
	
	Notice that 
	$\Pm> (c-1)\eta \UR - 2\eta \nu$, so we could solve the above inequality and get
	\begin{align*}
	&(c-1)\eta \UR - 2\eta \nu  - 2 \eta \UR
	- 2 \eta \nu \log^{1/2}(2\tau)\geq 0\\
	\Rightarrow&
	\frac{(c-3) \UR - 2\nu }{2  \nu}
	\geq   \log^{1/2}(2\tau)\\
	\Rightarrow&
	\tau \leq \frac{e^{
			\left (\frac{(c-3) \UR }{2  \nu}-1\right )^2}}{2}\\
	\end{align*}
	Since we assume that $c$ is a large constant and $\UR\geq \nu$, so $\tau$ can be fairly large in order to make sure $\Tmp\geq 2$. 
	We also know that $\Tmp\leq \frac{-(\UL+\UR+2\nu)}{\UR}< c $.
	
	On the other hand, by simple calculation, we know $\Tap\leq \frac{-(\Ll-\nu)}{\LR}< c'< \frac{e^{c/3}}{6}$. 
	Therefore, we can always pick a $\tau$ such that  $\frac{\Tmp+\Tap }{\tau }\leq \frac12$. So finally, we are ready to prove Theorem \ref{thm:asym_avg}.

	\begin{proof}[Proof of Theorem \ref{thm:asym_avg}]
		By Lemma \ref{lem:w0bounds} and Theorem \ref{thm:tap}, 
		$\Tap$ is an upper bound on the length of the $i$-th round.
		By Theorem \ref{thm:tmp}, we know that SGD will stay at flat side for at least $\Tmp$ steps, and each step is lower bounded by 
		$ w_t>  w_0 - t \eta \UR
		- \sqrt{2t} \eta \nu\log^{1/2}(2\tau)$, therefore we know that with probability 
		$1-\frac{\Tmp+\Tap}{\tau}$:
		
		\begin{align*}
		\frac1{T_i} \sum_{j=0}^{T_i} w^i_j
		&\geq \frac{1}{\Tap}\left (\sum_{t=0}^{\Tmp}[w_0 - t \eta \UR
		- \sqrt{2t} \eta \nu\log^{1/2}(2\tau)] -\eta (\UR+\nu)\right )\\
		&\geq \frac{1}{\Tap}\left (
		\eta \UR \frac{(\Tmp+1)\Tmp}{2}
		+\sqrt{2\Tmp } \eta \nu\log^{1/2}(2\tau)
		-\eta (\UR+\nu)\right )\\
		&\geq \frac{\Tmp^2}{\Tap}
		\eta \UR
		\end{align*}
		
		The above inequality discussed the scenario when Theorem \ref{thm:tmp} and Theorem \ref{thm:tap} hold. If they do not hold, which happens with probability at most 
		$\frac{\Tmp+\Tap}{\tau}$, we need to get lower bound for $\frac1{T_i} \sum_{j=0}^{T_i} w^i_j$. Notice that by Lemma \ref{lem:w0bounds}, we know that SGD stays at the sharp side for at most $1$ iterate in each round, and also the iterates on the flat sides are always positive with $w_0\geq \Pm>\eta(\UR+\nu)$. Therefore, we have the following trivial bound:
		\[\frac1{T_i} \sum_{j=0}^{T_i} w^i_j\geq \frac{-\eta(\UR+\nu)+w_0}{2}>0
		\]
		
		Combining two cases together we get
		\begin{align*}
		\E\left [\frac1{T_i} \sum_{j=0}^{T_i} w^i_j\right ]&\geq 
		\left (1-\frac{\Tmp+\Tap}{\tau}\right )
		\frac{\Tmp^2}{\Tap}
		\eta \UR
		+0
		\end{align*}
		Since we can pick $\tau$ s.t. $\frac{\Tmp+\Tap }{\tau }\leq \frac12$, we have 
		\[
		\E\left [\frac1{T_i} \sum_{j=0}^{T_i} w^i_j\right ]\geq
		\frac{\Tmp^2}{2\Tap}
		\eta \UR\triangleq c_0>0 \qedhere 
		\]
	\end{proof}

	\section{Additional Figures in Section \ref{subsec:illusion_swa}: No Bumps Between SGD and SWA Solutions}
	\label{appendix:swa_asym_figures}
	
	Asymmetric valley of ResNet-110 on CIFAR-10, 
	$(r,p,c, \zeta)=(5,0.005,25,3)$. See Figure \ref{missing:inter_res110_C10}.
	\begin{figure}[ht]
		\centering
		\includegraphics[width=.45\textwidth]{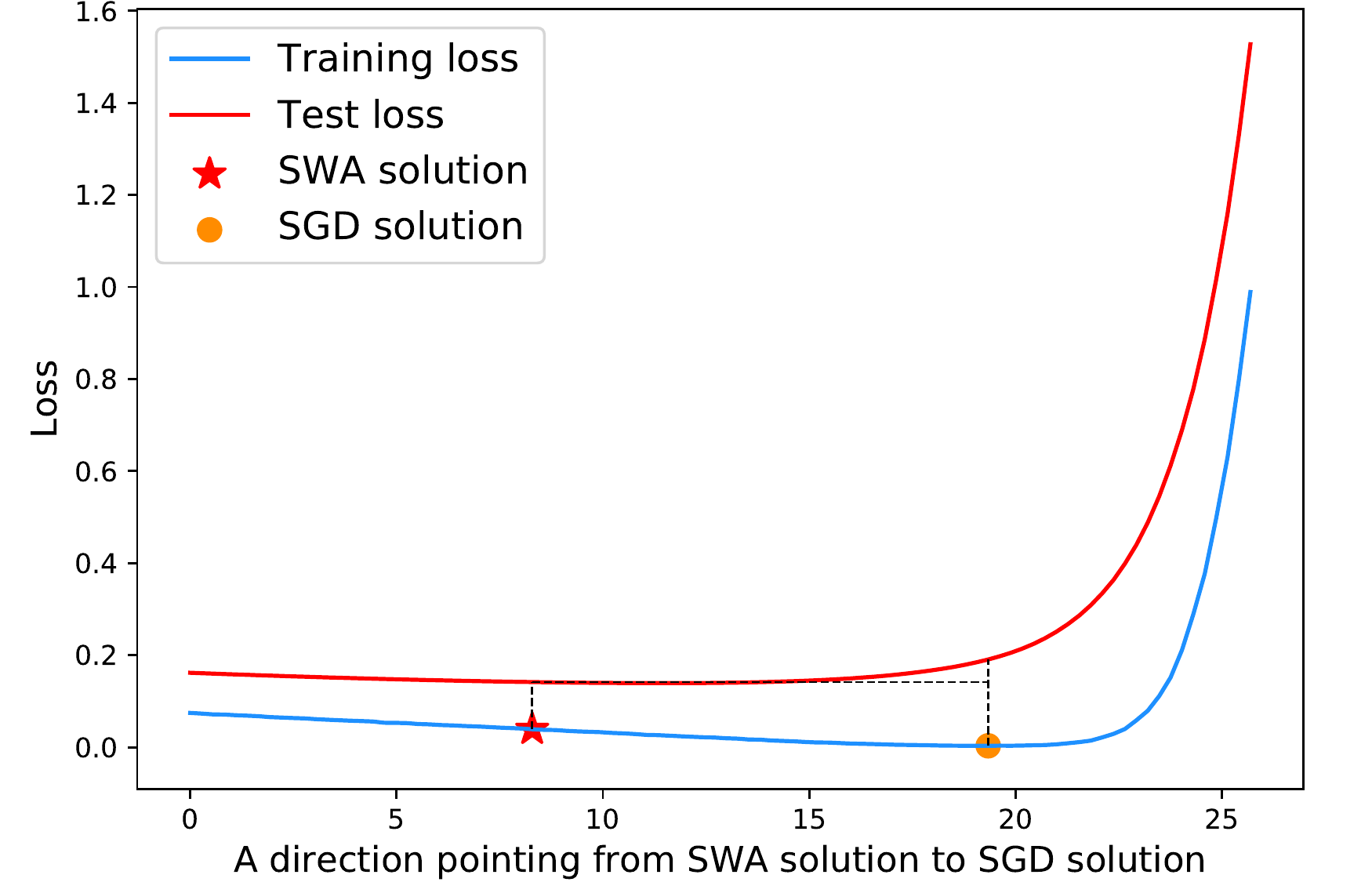}
		\caption{SWA and SGD interpolation (ResNet-110 on CIFAR-10)}
		\label{missing:inter_res110_C10} 
	\end{figure}
	
	Asymmetric valley of ResNet-164 on CIFAR-10, 
	$(r,p,c, \zeta)=(2,0.015,6,1)$. See Figure \ref{missing:inter_res164_C10}.
	\begin{figure}[ht]
		\centering
		\includegraphics[width=.45\textwidth]{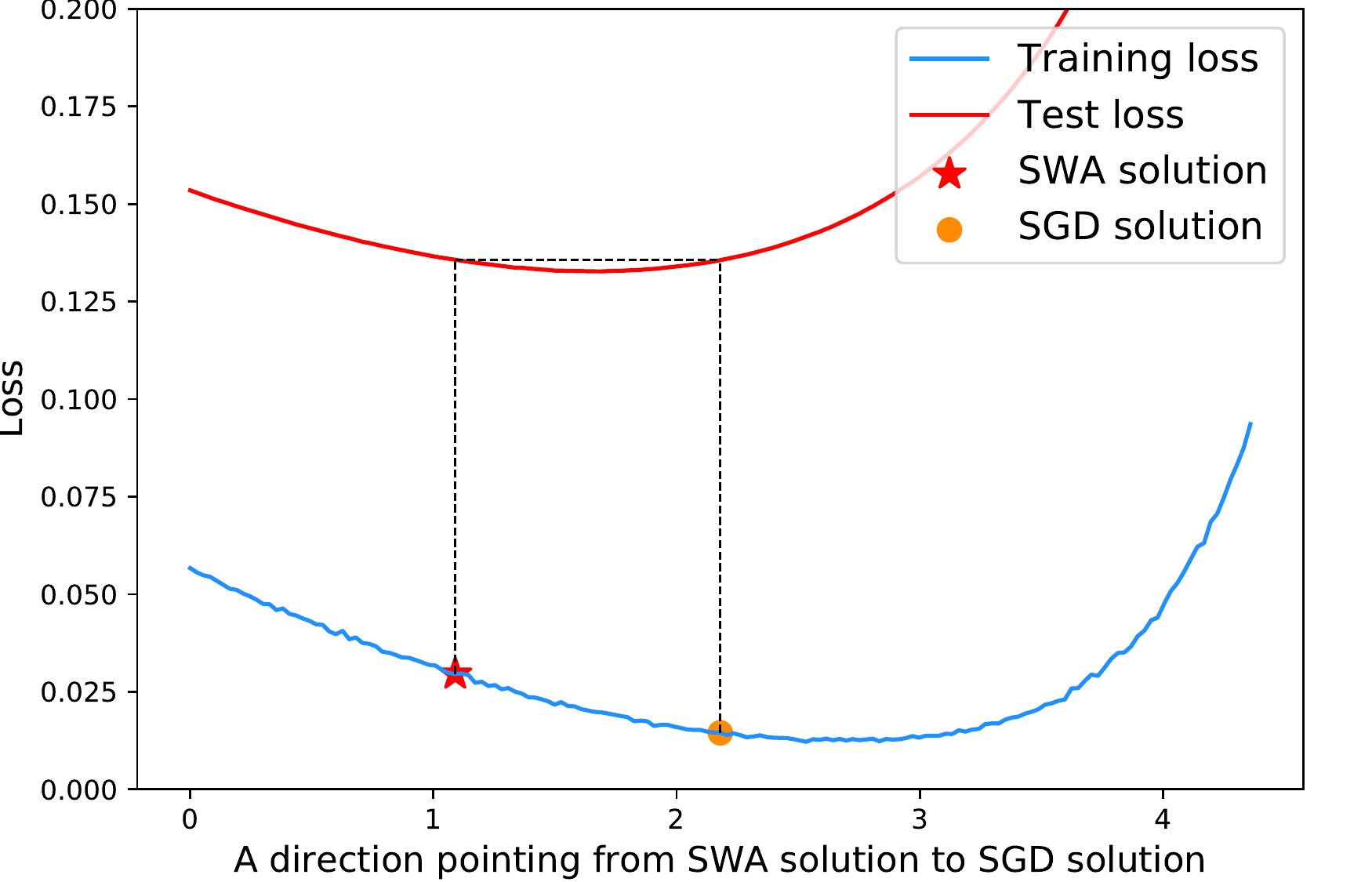}
		\caption{SWA and SGD interpolation (ResNet-164 on CIFAR-10)}
		\label{missing:inter_res164_C10} 
	\end{figure}
	
	Asymmetric valley of DenseNet-100 on CIFAR-10, 
	$(r,p,c, \zeta)=(6,7.35e-05,699,3)$. See Figure \ref{missing:inter_dense_C10}
	\begin{figure}[ht]
		\centering
		\includegraphics[width=.45\textwidth]{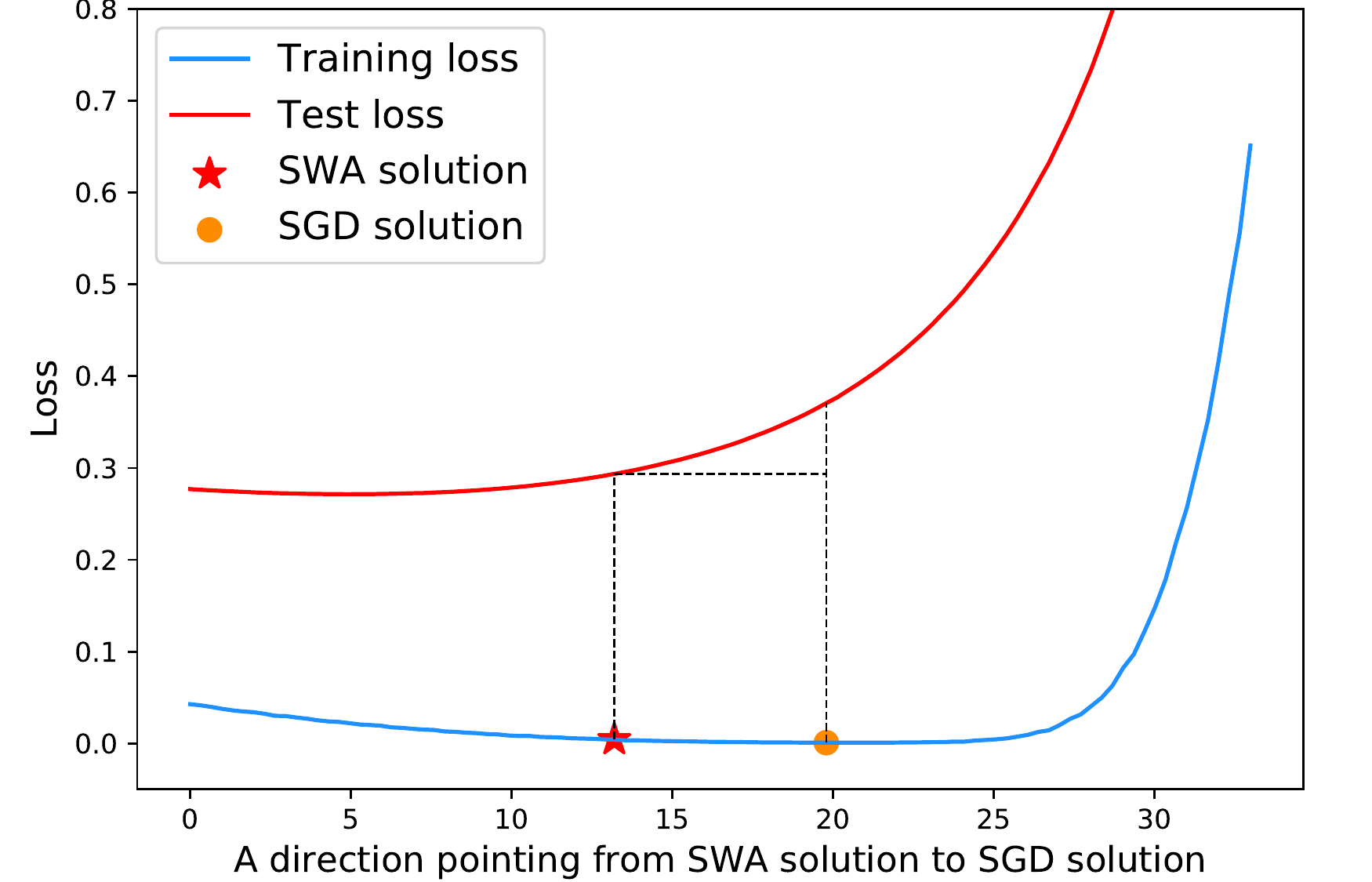}
		\caption{SWA and SGD interpolation (DenseNet-100 on CIFAR-10)}
		\label{missing:inter_dense_C10} 
	\end{figure}
	
	\section{Additional Figures in Section \ref{subsec:illusion_swa}: SGD Averaging Generates Good Bias}
	\label{appendix:verify_high_dimension}
	Examples for asymmetric directions of ResNet-110 on CIFAR-100 in Figure \ref{fig:high_dimension_find_110_1_C100_appendix}.
	\begin{figure}[htbp]
		\centering
		\begin{minipage}[t]{0.3\textwidth}
			\centering
			\includegraphics[width=5cm]{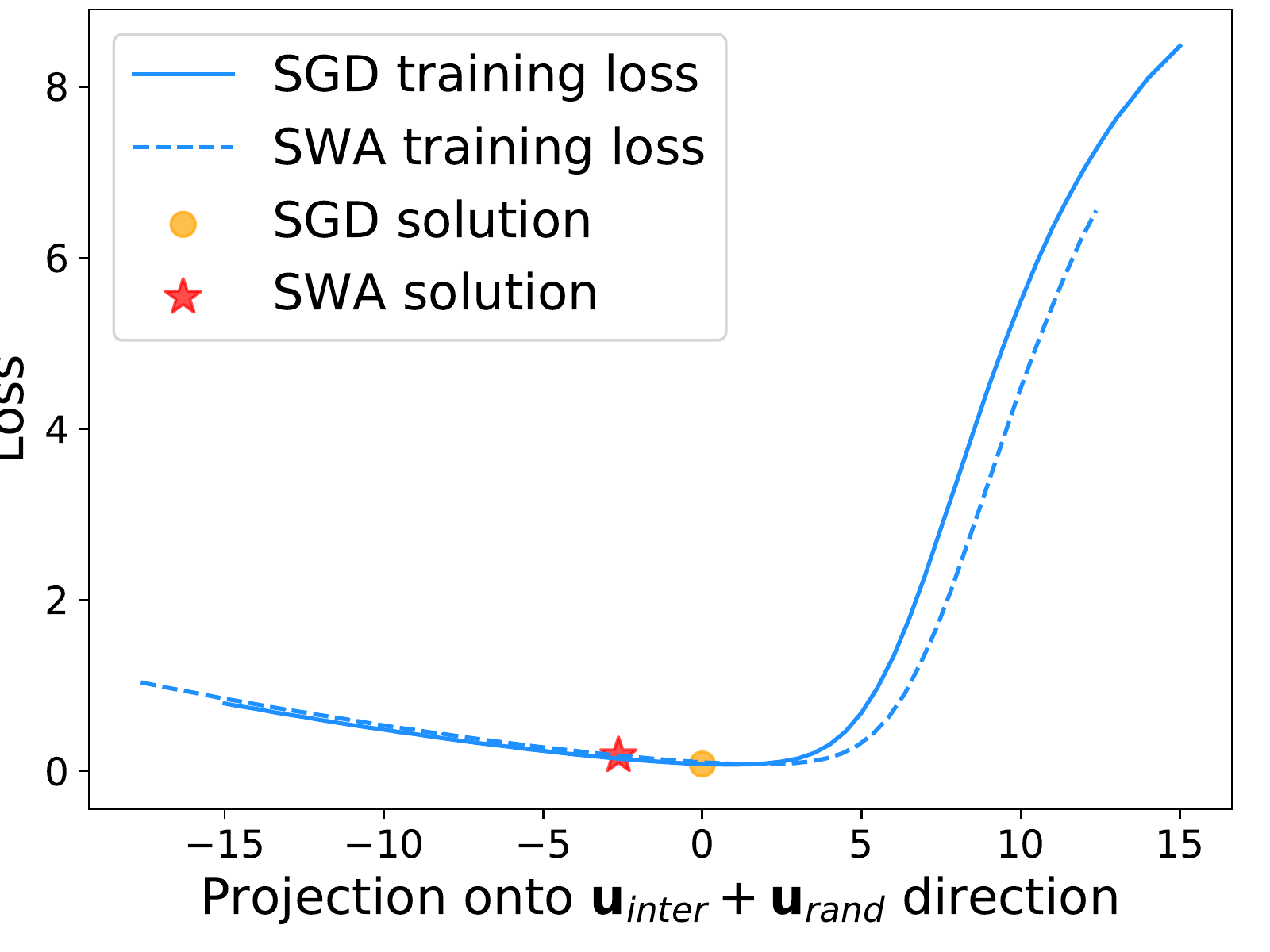}
			
		\end{minipage}\hspace{0.1in}
		\begin{minipage}[t]{0.3\textwidth}
			\centering
			\includegraphics[width=5cm]{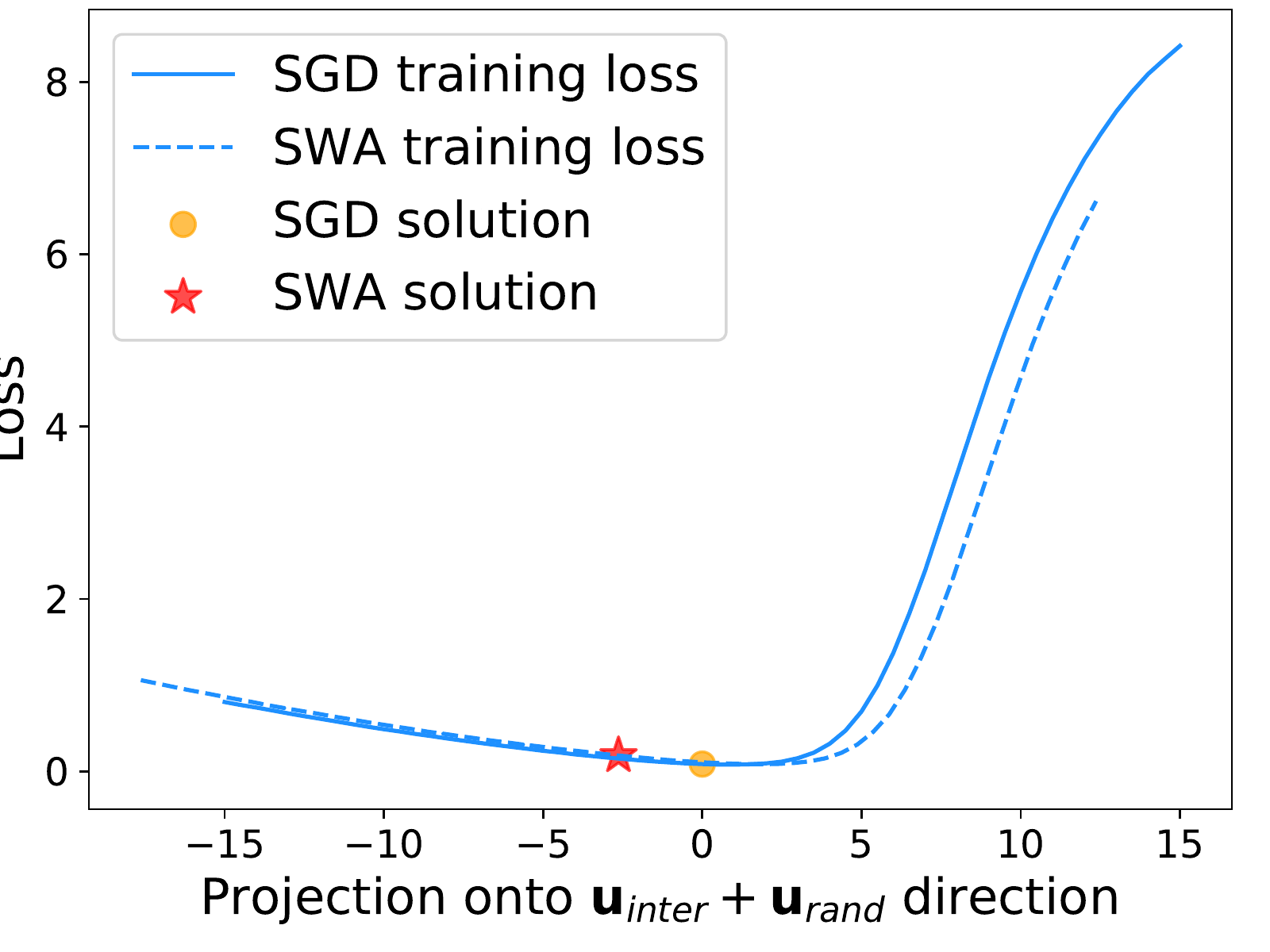}
			
		\end{minipage}\hspace{0.1in}
		\begin{minipage}[t]{0.3\textwidth}
			\centering
			\includegraphics[width=5cm]{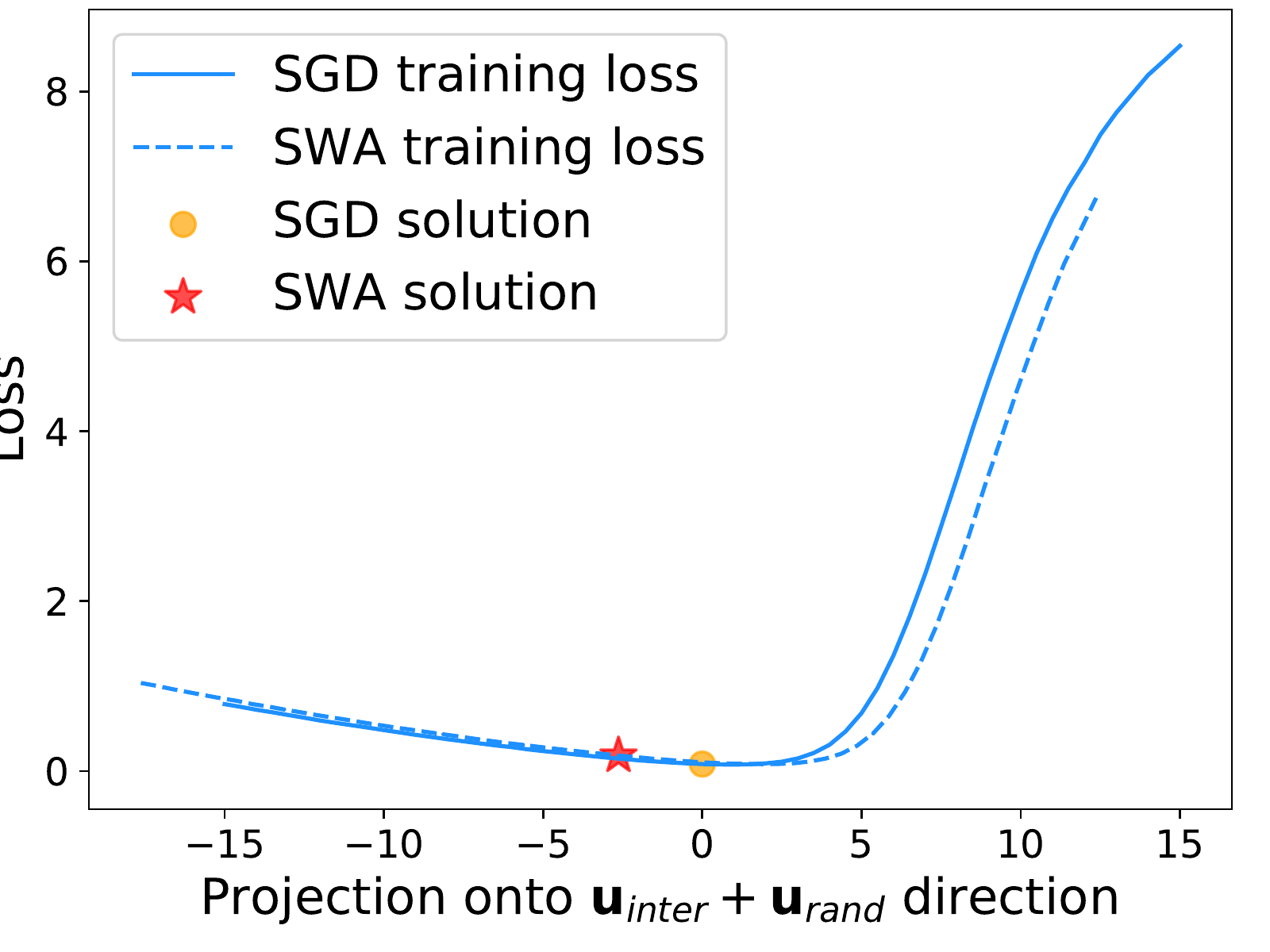}
			
		\end{minipage}
		\caption{The average of SGD has a bias on flat side (ResNet-110 on CIFAR-100).}
		\label{fig:high_dimension_find_110_1_C100_appendix}
		
	\end{figure}
	
	Examples for asymmetric directions of ResNet-164 on CIFAR-100 in Figure \ref{fig:high_dimension_find_164_1_C100},
	
	\begin{figure}[htbp]
		\centering
		\begin{minipage}[t]{0.3\textwidth}
			\centering
			\includegraphics[width=5cm]{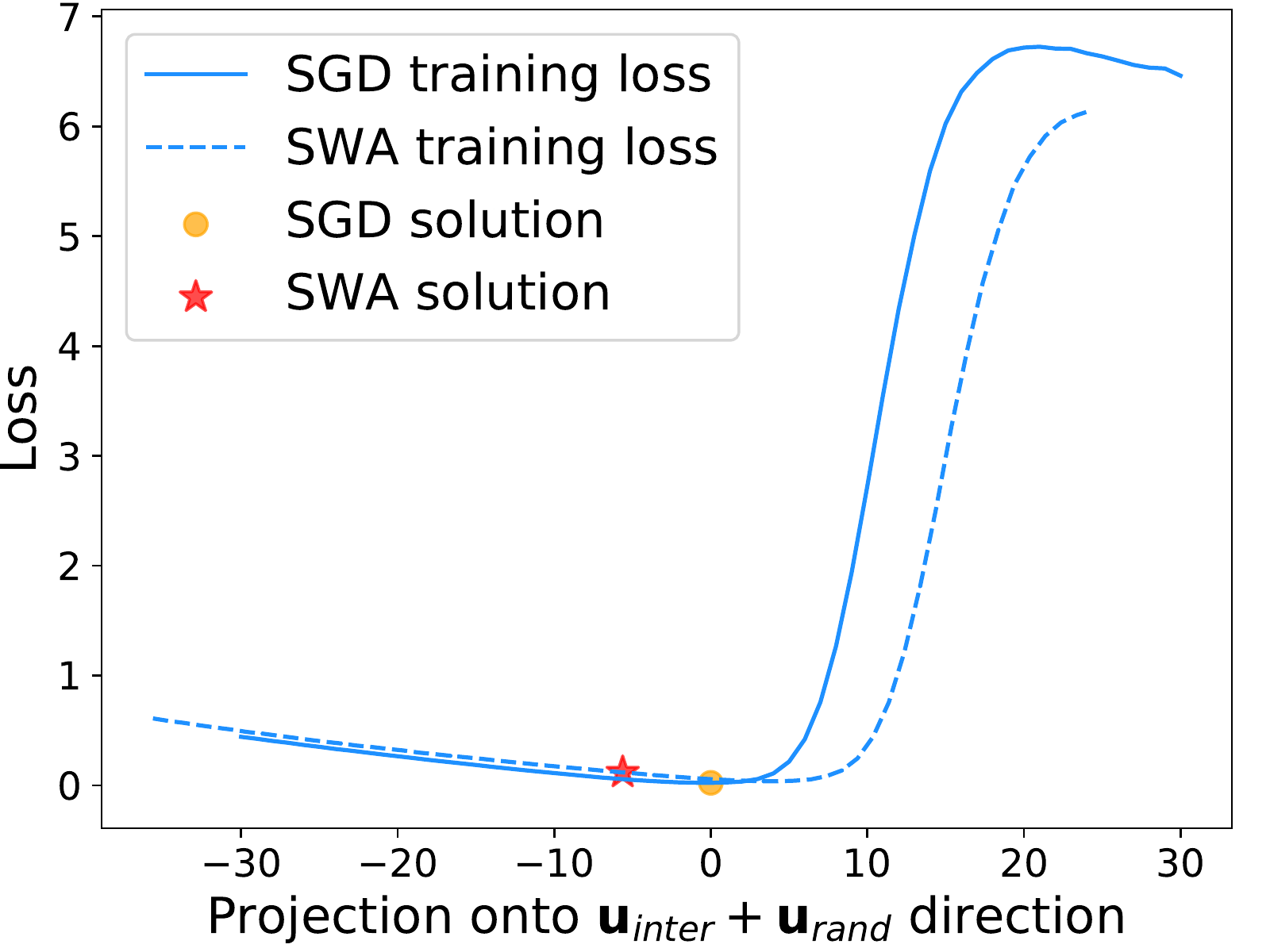}
		\end{minipage}\hspace{0.1in}
		\begin{minipage}[t]{0.3\textwidth}
			\centering
			\includegraphics[width=5cm]{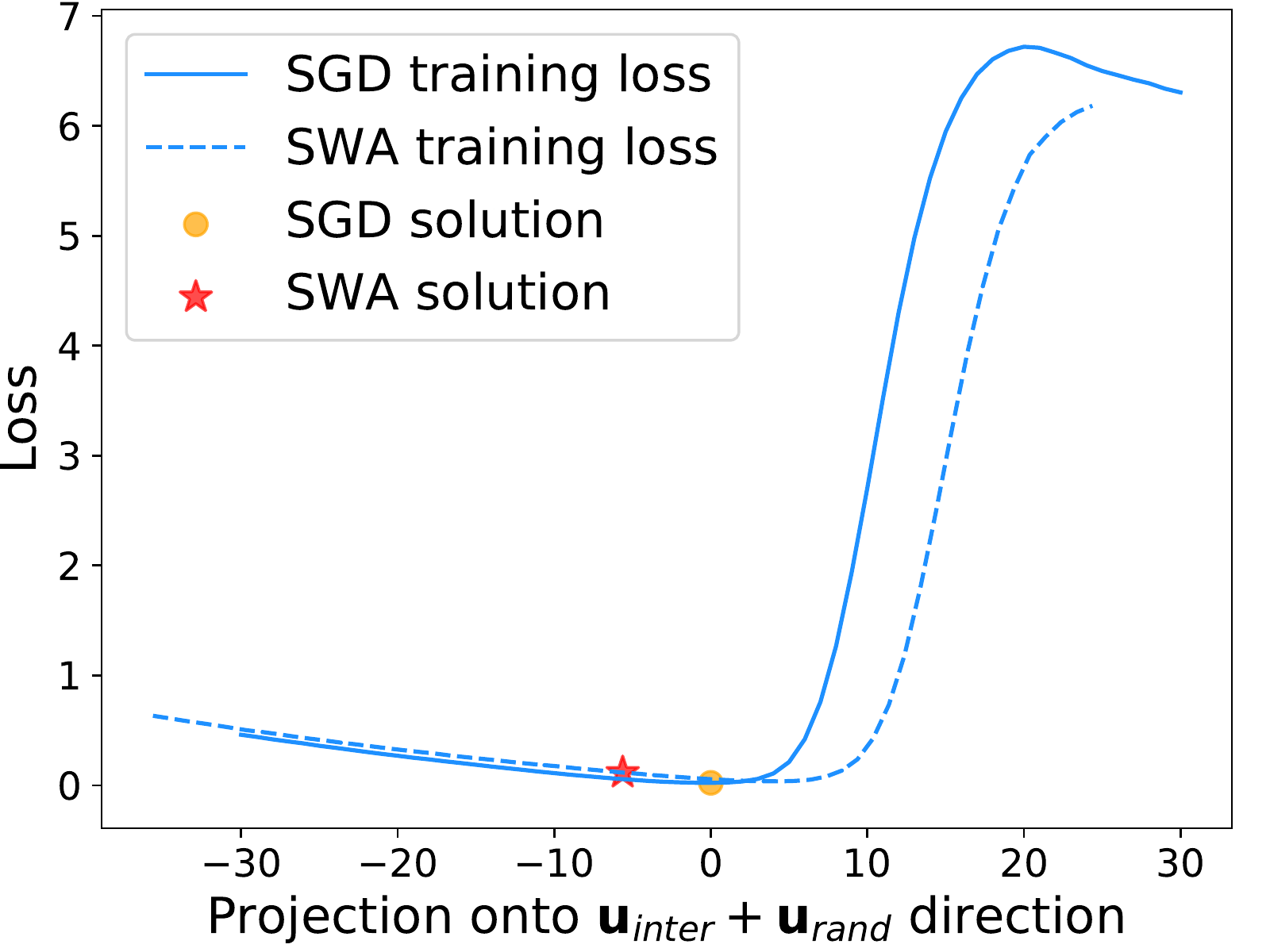}
		\end{minipage}\hspace{0.1in}
		\begin{minipage}[t]{0.3\textwidth}
			\centering
			\includegraphics[width=5cm]{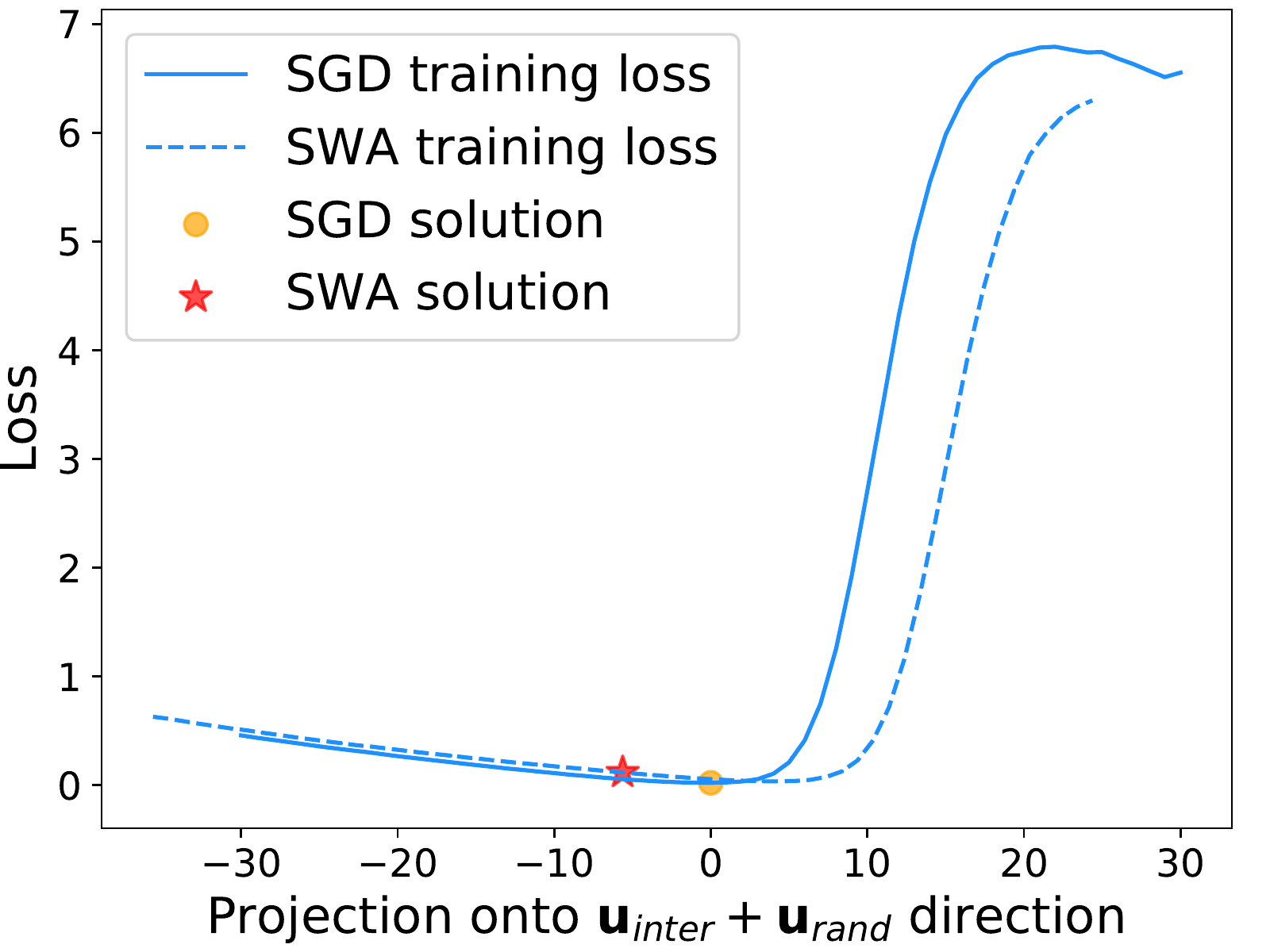}
		\end{minipage}
		\caption{The average of SGD has a bias on flat side (ResNet-164 on CIFAR-100).}
		\label{fig:high_dimension_find_164_1_C100}
	\end{figure}

	Examples for asymmetric directions of ResNet-110 on CIFAR-10 in Figure \ref{fig:high_dimension_find_110_1_C10}.
	
	\begin{figure}[htbp]
		\centering
		\begin{minipage}[t]{0.3\textwidth}
			\centering
			\includegraphics[width=5cm]{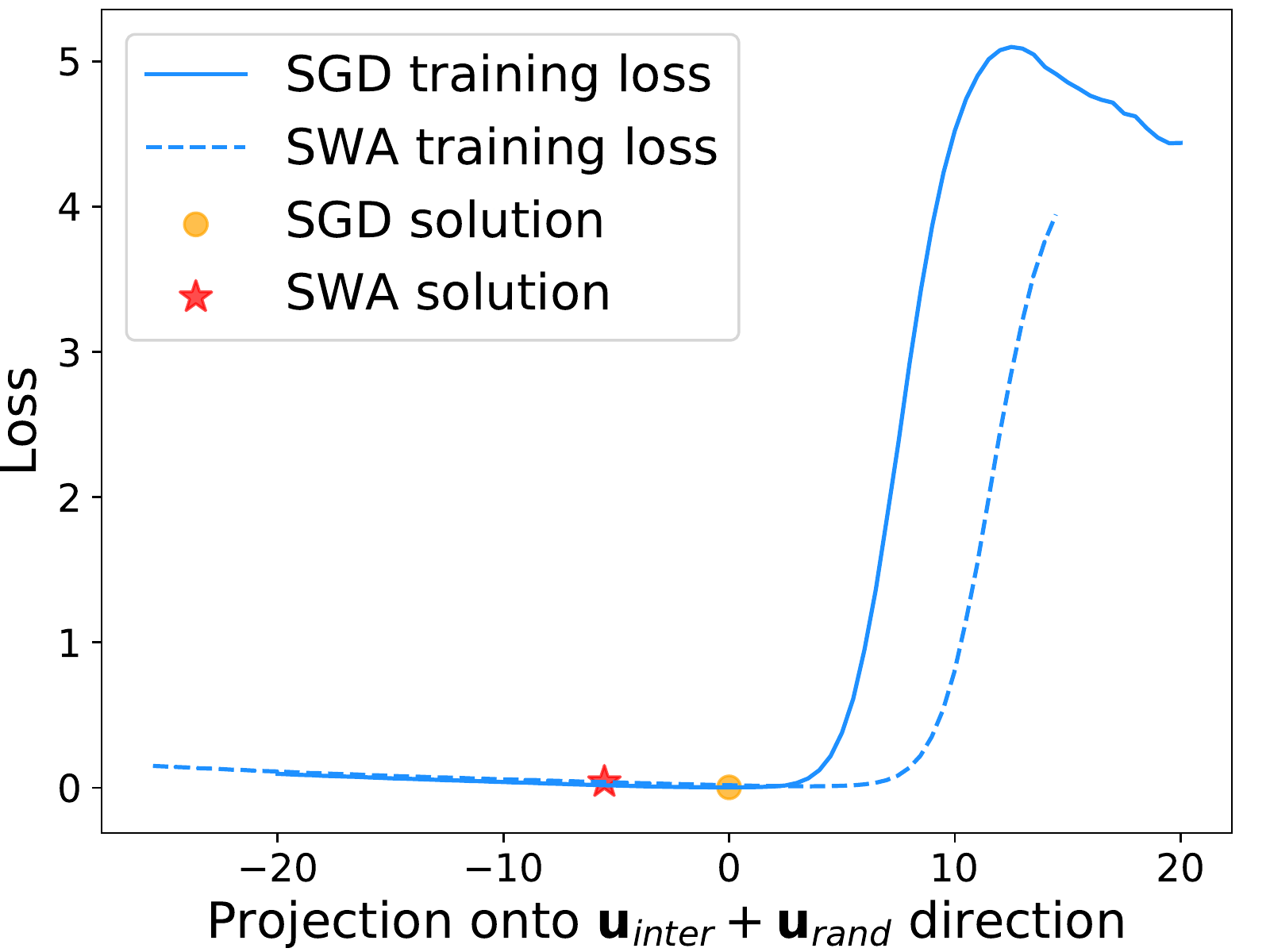}
		\end{minipage}\hspace{0.1in}
		\begin{minipage}[t]{0.3\textwidth}
			\centering
			\includegraphics[width=5cm]{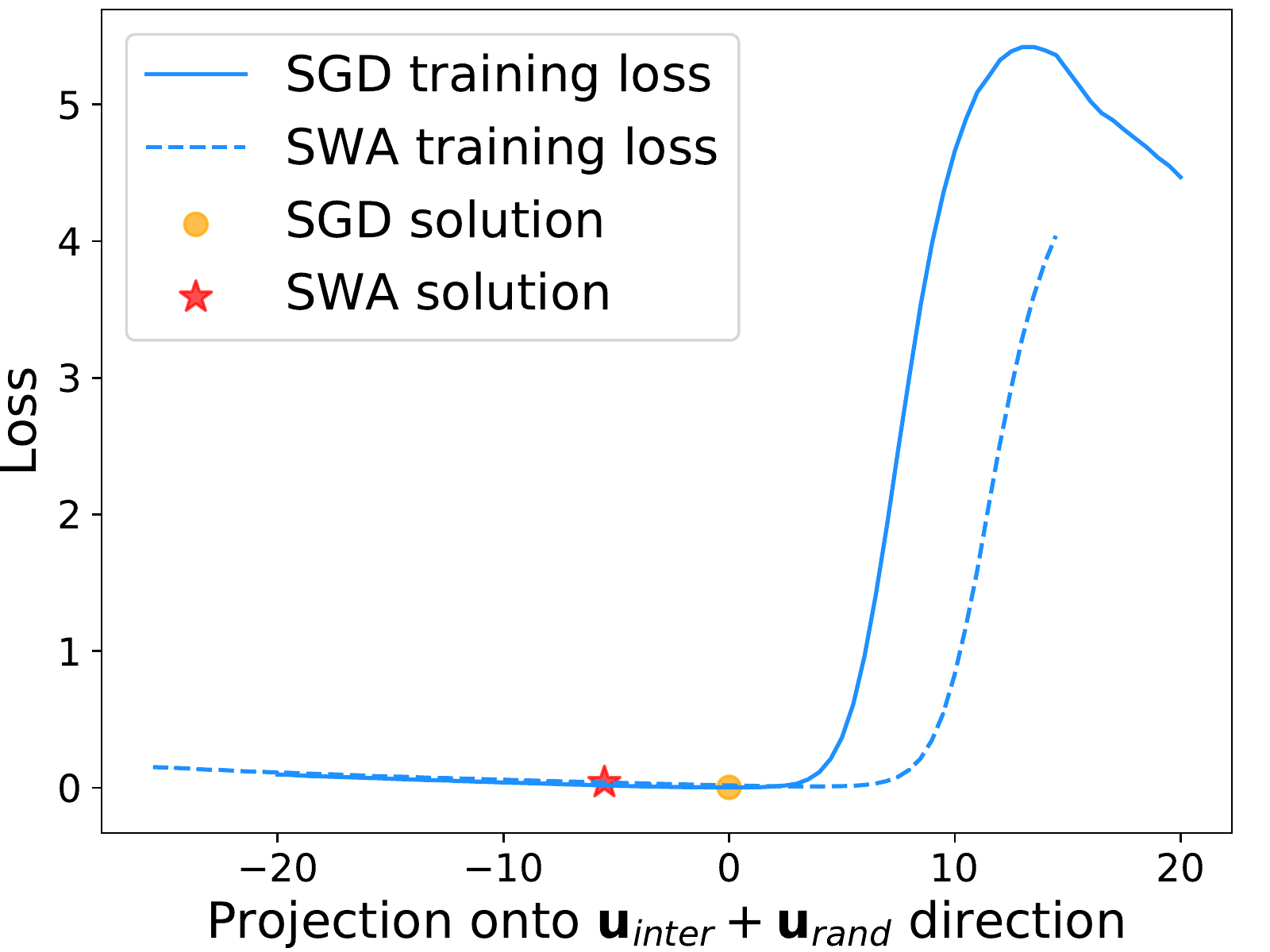}
		\end{minipage}\hspace{0.1in}
		\begin{minipage}[t]{0.3\textwidth}
			\centering
			\includegraphics[width=5cm]{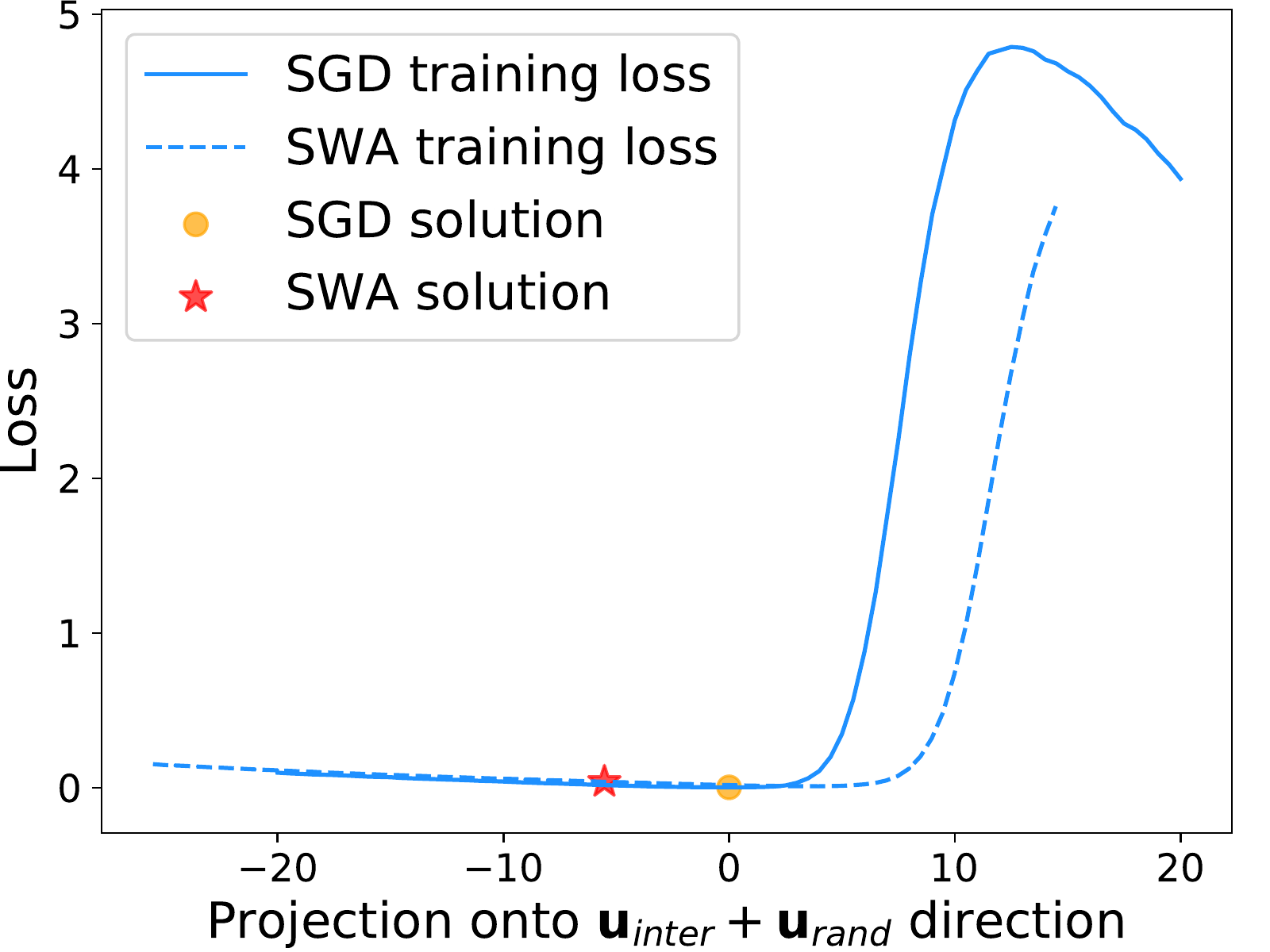}
		\end{minipage}
		\caption{The average of SGD has a bias on flat side (ResNet-110 on CIFAR-10).}
		\label{fig:high_dimension_find_110_1_C10}
	\end{figure}
	
	\section{Additional Figure for Section \ref{subsec:random_ray}: Width of Minima}
	\label{sec:random_ray}
	See Figure \ref{fig:random_ray_SGDandSWA}, similar results  were observed by \cite{swa} as well.
	\begin{figure}[htbp]
		\centering
		\includegraphics[width=.4\textwidth]{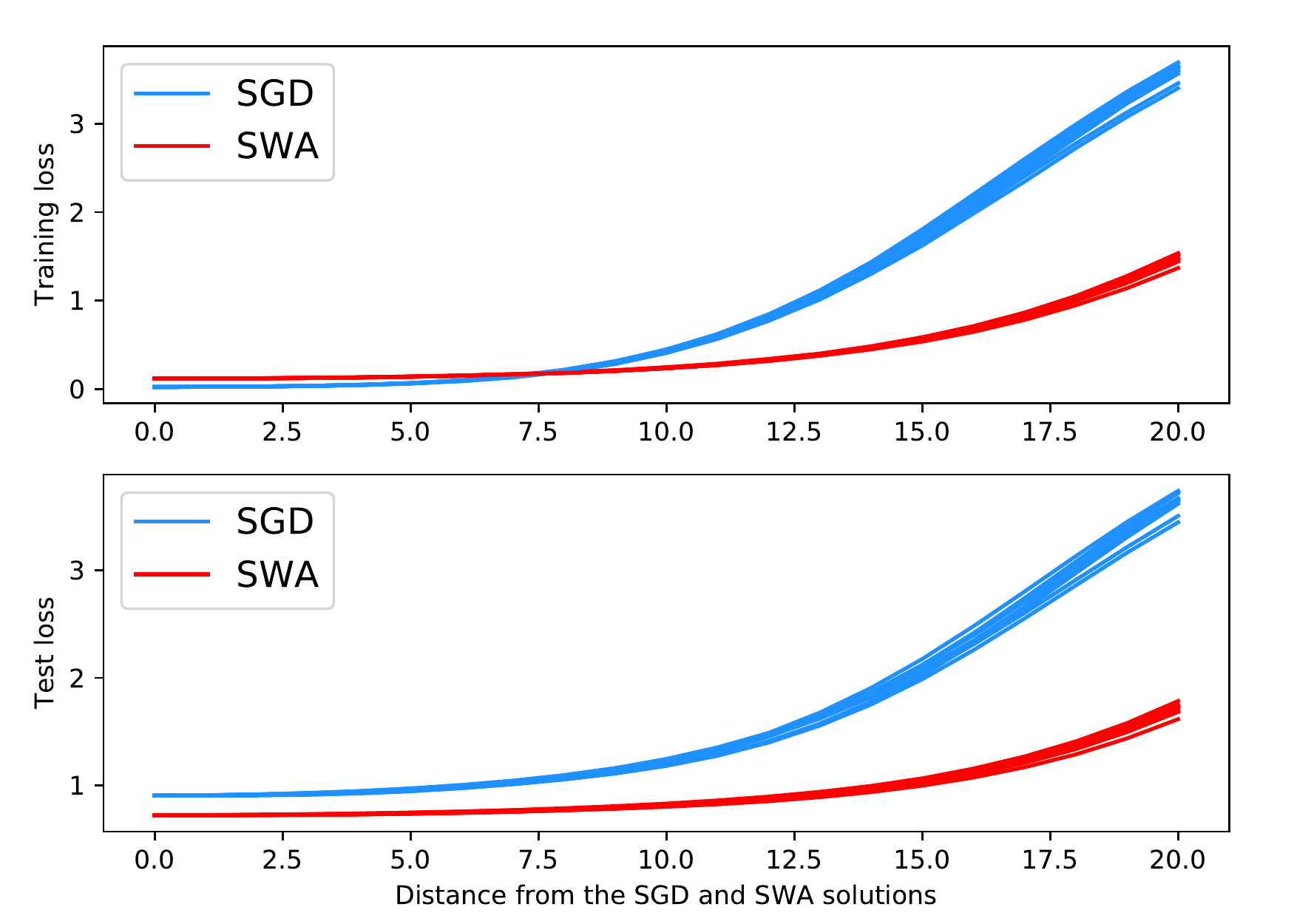}
		\caption{Random ray of SGD and SWA solution} 
		\label{fig:random_ray_SGDandSWA} 
	\end{figure}
	
	\section{Additional Figures for Section \ref{sec:bn}: BN Parameters Are More Asymmetric}
	\label{sec:exploration_on_bn_drection}
	
	See Figure \ref{fig:conv_bn_comp_res164_c10}, Figure \ref{fig:conv_bn_comp_res110_c100} and Figure \ref{fig:conv_bn_comp_dense_c100}.
	\begin{figure}[htbp]
		\centering
		\includegraphics[width=.4\textwidth]{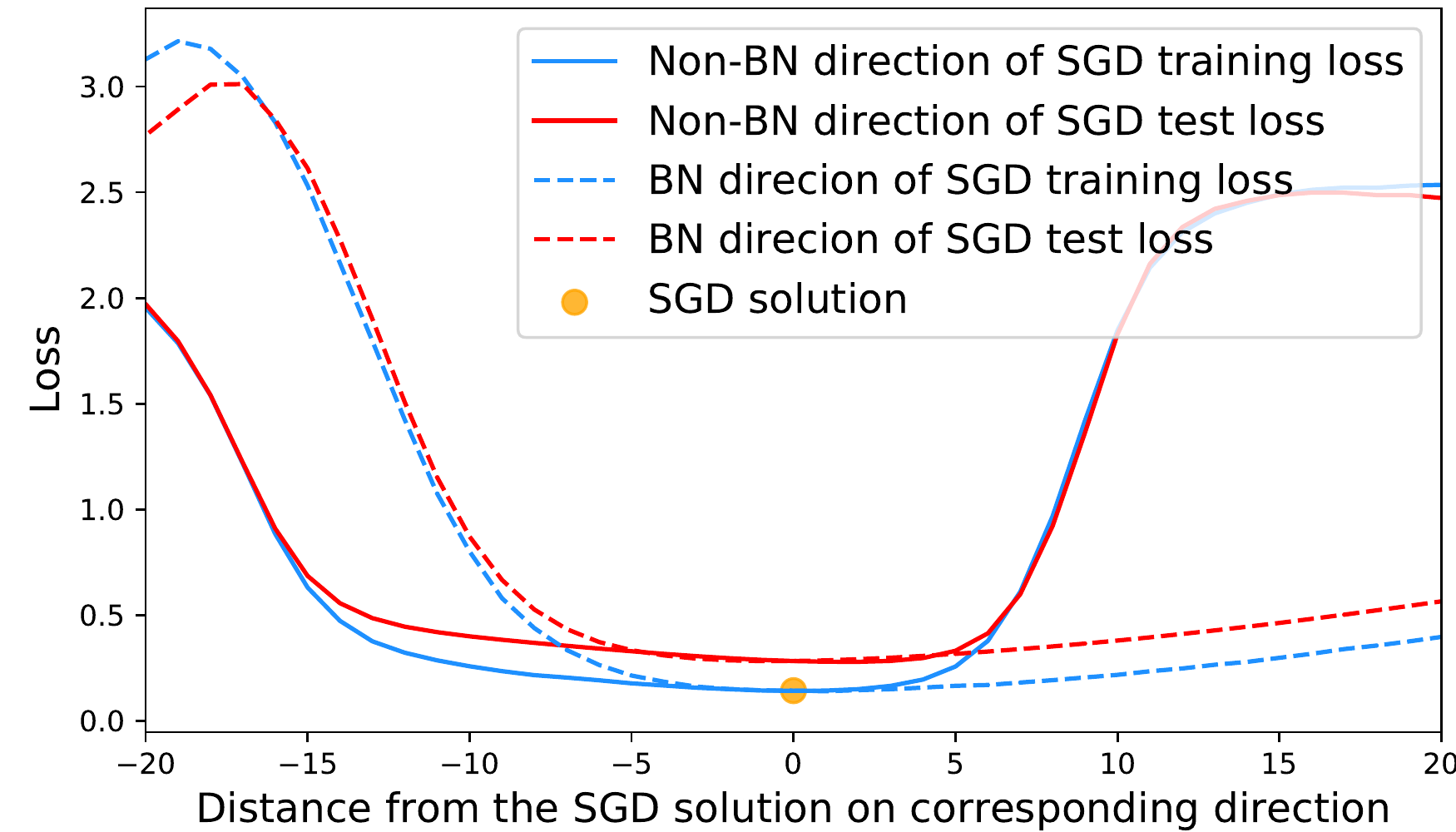}
		\caption{BN and Non-BN directions comparison of ResNet-164 on CIFAR-10}
		\label{fig:conv_bn_comp_res164_c10}
	\end{figure}
	
	\begin{figure}[htbp]
		\centering
		\includegraphics[width=.4\textwidth]{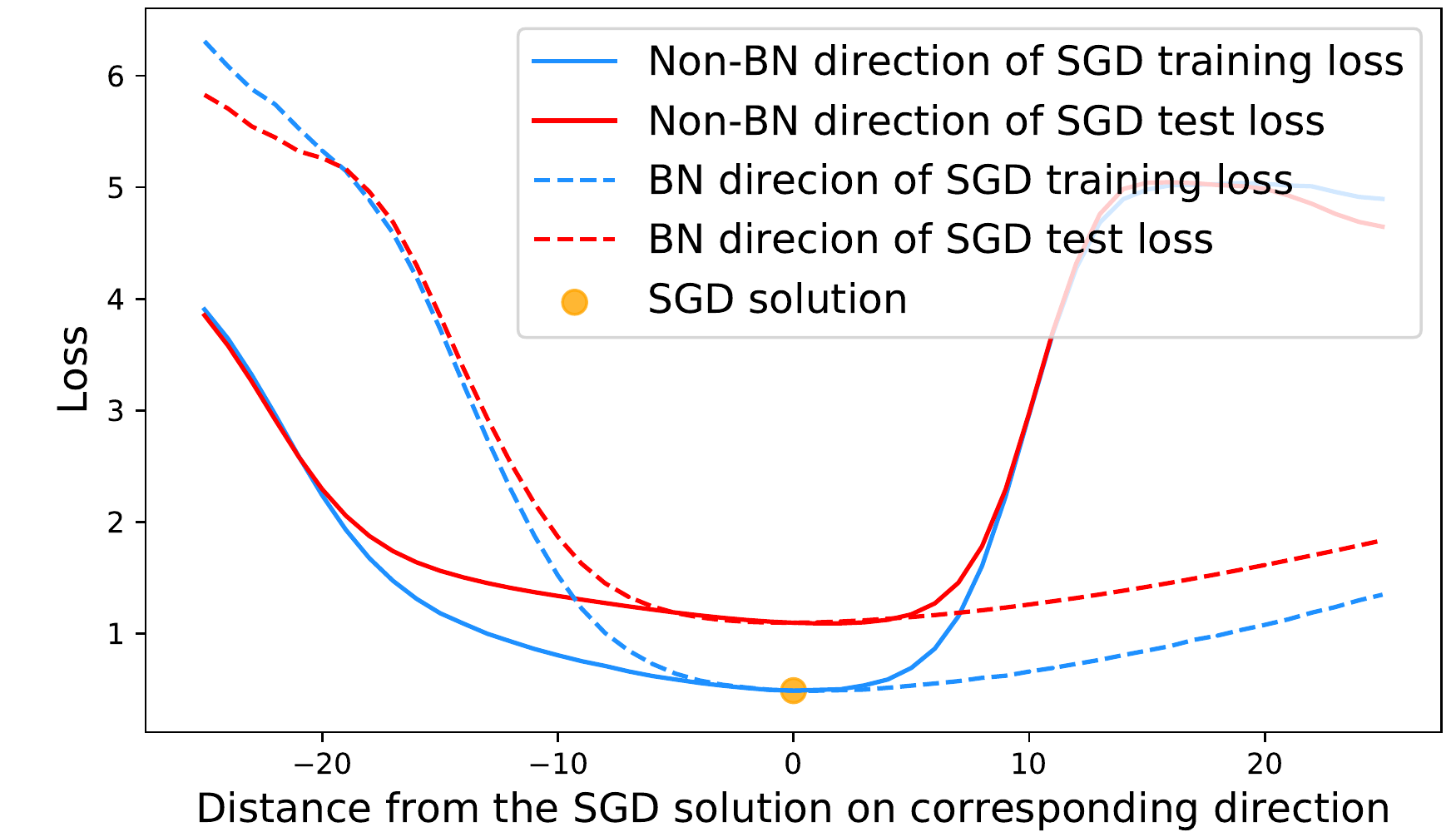}
		\caption{BN and Non-BN directions comparison of ResNet-110 on CIFAR-100}
		\label{fig:conv_bn_comp_res110_c100}
	\end{figure}
	
	\begin{figure}[H]
		\centering
		\includegraphics[width=.4\textwidth]{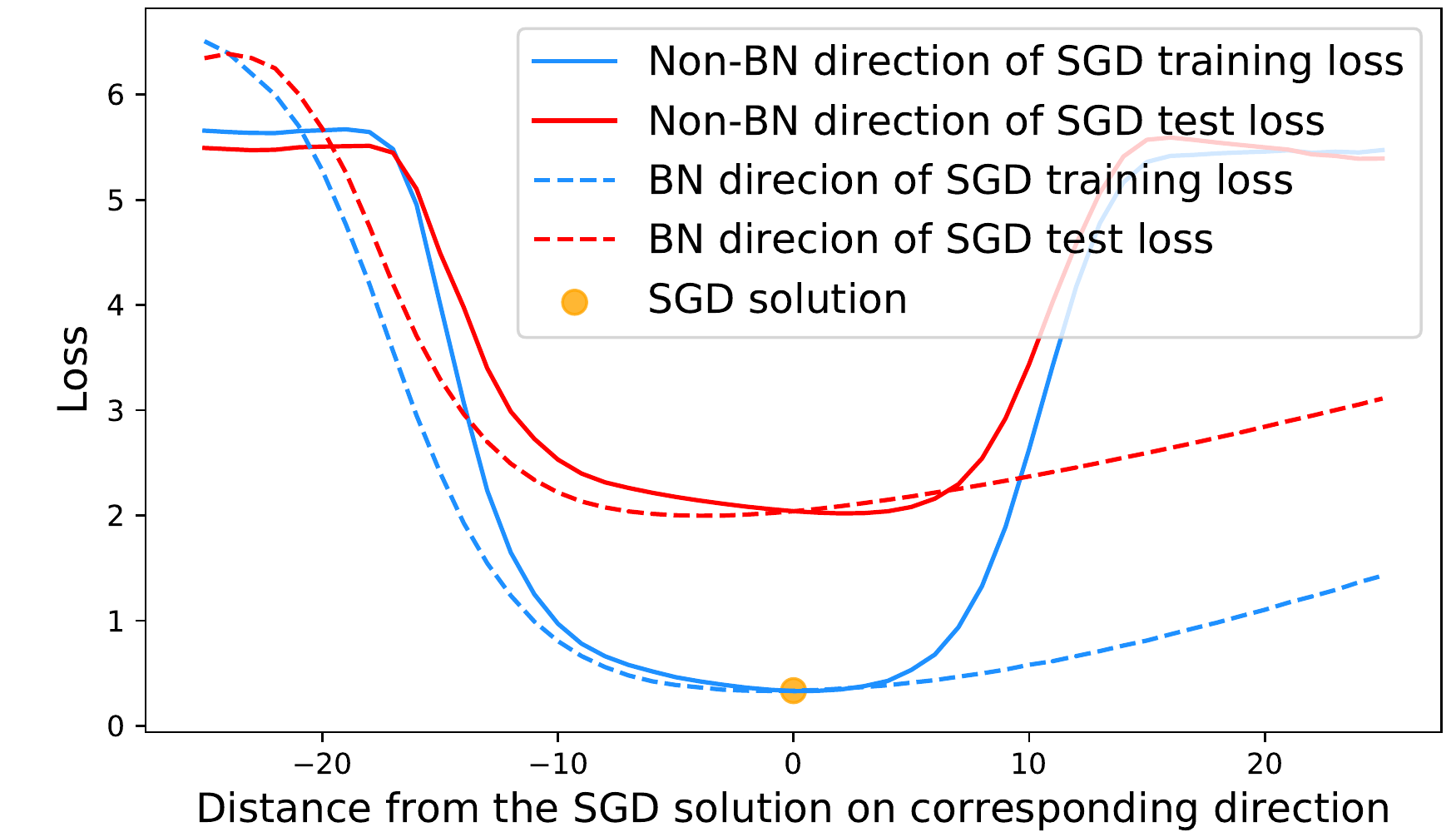}
		\caption{BN and Non-BN directions comparison of DenseNet-100 on CIFAR-100}
		\label{fig:conv_bn_comp_dense_c100}
	\end{figure}

\end{document}